\documentclass[10pt]{article}
\usepackage[preprint]{tmlr}   % preprint mode: authors shown, no "under review" header

\usepackage[american]{babel}
\usepackage{mathtools}
\usepackage{booktabs}
\usepackage{subcaption}
\usepackage{amssymb}
\usepackage{amsthm}
\usepackage{multirow}
\usepackage{graphicx}
\usepackage{hyperref}
\usepackage{url}

\newcommand{\E}{\mathbb{E}}
\newcommand{\LLR}{\mathrm{LLR}}
\newcommand{\Normal}{\mathcal{N}}
\newcommand{\Exp}{\mathrm{Exp}}
\newcommand{\Student}{t}
\newcommand{\NIG}{\mathrm{NIG}}

\newtheorem{proposition}{Proposition}
\newtheorem{corollary}{Corollary}

\makeatletter
\newcommand{\blfootnote}[1]{%
  \begingroup
    \stepcounter{footnote}%
    \renewcommand{\thefootnote}{}%
    \renewcommand{\@makefnmark}{}%
    \long\def\@makefntext##1{\parindent 1em\noindent ##1}%
    \footnotetext{#1}%
  \endgroup
}
\makeatother

\title{Exponential-Family Membership Inference:\\ From LiRA and RMIA to BaVarIA}

\author{\name Rickard Br\"annvall \email rickard.brannvall@ri.se \\
      \addr RISE Research Institutes of Sweden}

\begin{document}
\maketitle

\blfootnote{Supported by Vinnova, Sweden's innovation agency, through the project LEAKPRO II: Operational Privacy Risk Management for AI Systems (ref.\ 2025-03066). Code available at: \url{https://anonymous.4open.science/r/exponential-family-mia-C2B7}.}

\begin{abstract}
Membership inference attacks (MIAs) are becoming standard tools for auditing the privacy of machine learning models.
The leading attacks---LiRA \citep{carlini2022membership} and RMIA \citep{zarifzadeh2024lowcost}---appear to use distinct scoring strategies, while the recently proposed BASE \citep{lassila2025base} was shown to be equivalent to RMIA, making it difficult for practitioners to choose among them.
We show that all three are instances of a single exponential-family log-likelihood ratio framework,
differing only in their distributional assumptions and the number of parameters estimated per data point.
This unification reveals a hierarchy (BASE1--4) that connects RMIA and LiRA as endpoints of a spectrum of increasing model complexity, and yields a practical rule---match the attack's complexity to the available shadow-model budget.
Within this framework, we identify variance estimation as a primary bottleneck at small shadow-model budgets and propose BaVarIA, a Bayesian variance inference attack that replaces threshold-based parameter switching with conjugate normal-inverse-gamma priors.
BaVarIA yields a Student-$t$ predictive (BaVarIA-$t$) or a Gaussian with stabilized variance (BaVarIA-$n$), providing stable performance without per-dataset hyperparameter tuning.
Across 12 testbeds and 7 shadow-model budgets, BaVarIA is a drop-in replacement for LiRA that matches or, on average, improves upon it. The gains are largest in the practically important low-shadow-model and offline regimes: offline, the Bayesian prior replaces LiRA's heuristic and outperforms it on 10 of 12 testbeds.
\end{abstract}

\section{Introduction}\label{sec:intro}

Membership inference attacks (MIAs) determine whether a specific data point was used to train a machine learning model \citep{shokri2017membership}.
Beyond their role as privacy attacks, MIAs serve as empirical auditing tools: they provide lower bounds on the privacy leakage of trained models \citep{nasr2023tight, steinke2024privacy}, complementing theoretical guarantees from differential privacy \citep{dwork2006calibrating}.
As with any auditing technique, the same methods could in principle be misused; we focus on the defensive application throughout (see Appendix~\ref{app:broader_impact}).

The current landscape of score-based MIAs presents practitioners with several competing approaches.
LiRA \citep{carlini2022membership} fits per-point Gaussian models to shadow-model log-odds and computes a likelihood-ratio score.
RMIA \citep{zarifzadeh2024lowcost} uses a population-level reference to avoid per-point parameter estimation.
BASE \citep{lassila2025base} centers the target loss against a pooled shadow summary.
\citet{lassila2025base} proved that BASE and RMIA are equivalent, but the connection between these pooled, location-only methods and LiRA's per-point class-conditional Gaussian has not been established.

We clarify the relationship by showing that LiRA, RMIA, and BASE are all instances of a single framework.
The observation behind this is simple: each attack implicitly assumes a parametric distribution for a scalar summary statistic (loss, confidence, or log-odds) under the IN and OUT membership hypotheses, and then computes the corresponding log-likelihood ratio (LLR).
Different distributional assumptions and parameter-sharing constraints recover different attacks.

This unification yields three contributions:

\textbf{(1) A unifying framework (Section~\ref{sec:framework}).}
We formalize the exponential-family LLR and show that it specializes to known attacks under specific distributional assumptions (Exponential, Gaussian).
We define the BASE hierarchy (BASE1--4), a family of Gaussian LLR attacks with progressively relaxed parameter sharing, connecting RMIA/BASE at one end to LiRA at the other.
The hierarchy also yields two intermediate attacks between these endpoints and turns the choice of attack into a rule we test empirically: match the attack's complexity to the shadow-model budget.

\textbf{(2) BaVarIA: Bayesian Variance Inference Attack (Section~\ref{sec:bavaria}).}
As a secondary, applied contribution---a plug-in for LiRA rather than a replacement for RMIA---we address a concrete weakness: LiRA's performance degrades at small shadow-model budgets $K$ because per-point variance estimates become unreliable.
We replace maximum likelihood estimation with conjugate normal-inverse-gamma (NIG) Bayesian inference, yielding two variants: BaVarIA-$t$ (Student-$t$ predictive) and BaVarIA-$n$ (Gaussian with Bayesian variance).
Both provide Bayesian shrinkage that moves from global to per-point estimates as $K$ grows, with all posterior updates available in closed form.
In the large-$K$ limit both recover online LiRA exactly. The same prior also gives a natural attack for the offline setting, where it improves on LiRA.

\textbf{(3) Empirical evaluation (Section~\ref{sec:experiments}).}
We evaluate all methods across 12 testbeds (image and tabular), 7 shadow-model budgets ($K \in \{4, \ldots, 254\}$), and 32 experimental replicates.
On average, BaVarIA-$n$ matches or improves over LiRA at $K \geq 16$, while BaVarIA-$t$ gives the strongest AUC overall. The hierarchy's predicted complexity-versus-budget ordering is supported for TPR at low FPR on all 12 testbeds at both budget extremes. 
We also validate findings on an independent shadow-model collection.

% ==================================================================
\section{Background}\label{sec:background}

\paragraph{Setup.}
Let $D = \{(x_i, y_i)\}_{i=1}^N$ be a dataset.
A target model $f(\cdot; \theta_0)$ is trained on a subset $S_0 \subset D$, and $K$ shadow models $f(\cdot; \theta_k)$, $k = 1, \ldots, K$, are each trained on subsets $S_k \subset D$.
For a fixed point $i$, the membership indicator $m_{i,k} \in \{0,1\}$ records whether $(x_i, y_i) \in S_k$.
Membership is known for shadows ($k > 0$) and unknown for the target ($k = 0$).

We observe the per-point loss $\ell_{i,k} = \ell(f(x_i; \theta_k), y_i)$, the predicted confidence $p_{i,k}$ for the true label, and form scalar statistics such as the rescaled logit $\varphi_{i,k} = \log(p_{i,k} / (1 - p_{i,k}))$.
Under cross-entropy loss, $\ell = -\log p$, so all three quantities are monotone transforms of the same signal.

\paragraph{Optimal membership score.}
\citet{sablayrolles2019white} showed that the Bayes-optimal membership score is the log-likelihood ratio
\begin{equation}\label{eq:llr}
\LLR(z) = \log \frac{p(z \mid m=1)}{p(z \mid m=0)},
\end{equation}
where $z$ is any sufficient scalar statistic of the model output.
Their analysis further shows that the optimal attack depends on the model only through the loss, so black-box access to a scalar output statistic is sufficient---justifying the restriction to scalar statistics throughout this paper.
All score-based MIAs can be viewed as estimating this quantity under different modeling assumptions.

\paragraph{LiRA.}
For each point $i$, LiRA \citep{carlini2022membership} models the rescaled logit $\varphi_{i,k} \mid m = m' \sim \Normal(\mu_{i,m'}, \sigma_{i,m'}^2)$ with parameters estimated from shadow models.
Let $\mathcal{K}_{i,m} = \{k > 0 : m_{i,k} = m\}$.
The LiRA score is
\begin{equation}\label{eq:lira}
\text{LiRA}_i = \frac{(\varphi_{i,0} - \hat\mu_{i,0})^2}{2\hat\sigma_{i,0}^2} - \frac{(\varphi_{i,0} - \hat\mu_{i,1})^2}{2\hat\sigma_{i,1}^2} + \log \frac{\hat\sigma_{i,0}}{\hat\sigma_{i,1}},
\end{equation}
where $\hat\mu_{i,m}$ and $\hat\sigma_{i,m}^2$ are the MLE from $\{\varphi_{i,k}\}_{k \in \mathcal{K}_{i,m}}$.

\paragraph{RMIA and BASE.}
RMIA \citep{zarifzadeh2024lowcost} computes a likelihood ratio against a population reference, avoiding per-point parameter estimation.
BASE \citep{lassila2025base} scores each point by centering the target loss against a pooled shadow summary:
\begin{equation}\label{eq:base}
\text{BASE}_i = -\ell_{i,0} - \log\!\left(\tfrac{1}{K} \textstyle\sum_{k=1}^K e^{-\ell_{i,k}}\right).
\end{equation}
\citet{lassila2025base} proved that BASE and RMIA produce equivalent ROC curves at $\gamma = 1$, where $\gamma$ denotes RMIA's acceptance threshold on the ratio of the target's probability to the population reference.

% ==================================================================
\section{Exponential-Family Framework}\label{sec:framework}

We now show that the attacks of Section~\ref{sec:background} are instances of a single parametric framework.
For a fixed point $i$, assume the scalar statistic $z_{i,k} \mid m_{i,k} = m$ follows an exponential-family distribution
\begin{equation}\label{eq:expfam}
p(z \mid m) = h(z) \exp\!\big(\eta_m^\top T(z) - A(\eta_m)\big),
\end{equation}
with sufficient statistics $T(z)$, natural parameters $\eta_m$, base measure $h(z)$, and log-partition function $A$.
The LLR~\eqref{eq:llr} then takes the form
\begin{equation}\label{eq:expfam_llr}
\LLR(z) = (\eta_1 - \eta_0)^\top T(z) - \big(A(\eta_1) - A(\eta_0)\big),
\end{equation}
which is an affine function of $T(z)$.
Under any exponential-family model, the optimal membership score is thus a linear combination of the sufficient statistics, plus a constant offset.

\subsection{Distributional Specializations}

We consider two primary distributional families; others (Beta, Gamma) are developed in Appendix~\ref{app:derivations}.

\paragraph{Exponential model (loss statistic).}
Assume $\ell_{i,k} \mid m \sim \Exp(\lambda_m)$.
The LLR is
\begin{equation}\label{eq:exp_llr}
\LLR(\ell) = \log \frac{\lambda_1}{\lambda_0} - (\lambda_1 - \lambda_0)\, \ell,
\end{equation}
which is linear in the target loss without any structural constraint---the one-parameter exponential family yields an affine LLR automatically.

\paragraph{Gaussian model (log-odds statistic).}
Assume $\varphi_{i,k} \mid m \sim \Normal(\mu_m, \sigma_m^2)$.
The LLR is
\begin{equation}\label{eq:gauss_llr}
\LLR(\varphi) = \frac{(\varphi - \mu_0)^2}{2\sigma_0^2} - \frac{(\varphi - \mu_1)^2}{2\sigma_1^2} + \log \frac{\sigma_0}{\sigma_1}.
\end{equation}
Under equal variances $\sigma_0^2 = \sigma_1^2 = \sigma^2$, this reduces to
\begin{equation}\label{eq:gauss_equal_var}
\LLR(\varphi) = \frac{\mu_1 - \mu_0}{\sigma^2} \left(\varphi - \frac{\mu_1 + \mu_0}{2}\right),
\end{equation}
which is linear in $\varphi$.

Table~\ref{tab:expfam} summarizes the distributional families and their LLR forms.

\begin{table}%[t]
\centering
\caption{Exponential-family LLR-based attacks.}\label{tab:expfam}
\small
\begin{tabular}{lccc}
\toprule
Model & Statistic $z$ & Params & LLR form \\
\midrule
Exponential & loss $\ell$ & 2 & linear in $\ell$ \\
Gaussian & log-odds $\varphi$ & 4 & $\varphi, \varphi^2$ \\
\quad (equal var.) & log-odds $\varphi$ & 3 & linear in $\varphi$ \\
\midrule
\multicolumn{4}{l}{\emph{Additional families (Appendix~\ref{app:derivations}):}} \\
Gamma & loss $\ell$ & 4 &  $\ell, \log\ell$ \\
Beta & confidence $p$ & 4 & $\log p, \log(1{-}p)$ \\
\bottomrule
\end{tabular}
\end{table}

\subsection{The BASE Hierarchy}\label{sec:base_hierarchy}

We define a hierarchy of four attacks by progressively relaxing parameter-sharing constraints.
BASE1--4 follow from the Gaussian LLR~\eqref{eq:gauss_llr} under decreasing levels of parameter sharing; BASE1 can also be derived from any one-parameter exponential family, including the exponential~\eqref{eq:exp_llr}, and Gamma models (Appendix~\ref{app:base_formulas}).
All variants operate on a scalar statistic $z_{i,k}$ (loss, log-odds, or confidence) and differ only in how parameters are estimated from shadow models.
The derivation chain and two formal corollaries (LiRA $=$ BASE4, BASE $=$ BASE1) are given in Appendix~\ref{app:base_formulas}; here we summarize.

\paragraph{BASE1 (pooled centering).}
Under a constant LLR slope, only the per-point centering affects the ranking; pooling all shadows to estimate this center gives:
\begin{equation}\label{eq:base1}
\text{BASE1}_i = z_{i,0} - \hat{z}_{i,\text{pool}},
\end{equation}
where $\hat{z}_{i,\text{pool}}$ is estimated from all $K$ shadows.
Setting $z = -\ell$ with log-sum-exp centering recovers the BASE score~\eqref{eq:base}; the same result follows independently from the exponential and Gamma models (Appendix~\ref{app:base_formulas}).

\paragraph{BASE2 (separate means, constant variance).}
Estimating separate IN and OUT means but sharing one variance across all samples, the LLR reduces (up to the irrelevant constant variance) to
\begin{equation}\label{eq:base2}
\text{BASE2}_i \propto (\hat\mu_{i,1} - \hat\mu_{i,0})\left(z_{i,0} - \tfrac{1}{2}(\hat\mu_{i,1} + \hat\mu_{i,0})\right),
\end{equation}
with per-point IN/OUT means $\hat\mu_{i,1},\hat\mu_{i,0}$ and a shared variance whose value does not affect the ranking (ROC and AUC are invariant to it). This isolates the contribution of class-mean separation alone, with no per-sample variance estimation.

\paragraph{BASE3 (separate means, pooled variance).}
Relaxing the constant-gap assumption, we estimate means separately for IN and OUT shadows while pooling variance:
\begin{equation}\label{eq:base3}
\text{BASE3}_i = \frac{\hat\mu_{i,1} - \hat\mu_{i,0}}{\hat\sigma_i^2} \left(z_{i,0} - \frac{\hat\mu_{i,1} + \hat\mu_{i,0}}{2}\right).
\end{equation}

\paragraph{BASE4 (class-conditional parameters).}
Removing all parameter-sharing constraints and substituting MLE into the Gaussian LLR~\eqref{eq:gauss_llr} gives four per-point parameters $(\hat\mu_{i,m}, \hat\sigma_{i,m}^2)$ for $m \in \{0,1\}$:
\begin{equation}\label{eq:base4}
\text{BASE4}_i = \frac{(z_{i,0} - \hat\mu_{i,0})^2}{2\hat\sigma_{i,0}^2} - \frac{(z_{i,0} - \hat\mu_{i,1})^2}{2\hat\sigma_{i,1}^2} + \log \frac{\hat\sigma_{i,0}}{\hat\sigma_{i,1}}.
\end{equation}

\begin{proposition}[Equivalences]\label{prop:equiv}
Under the standard implementations:
(a) BASE1 on loss with log-sum-exp centering is ROC-equivalent to RMIA at $\gamma = 1$;
(b) BASE4 on rescaled logits coincides with LiRA~\eqref{eq:lira}.
\end{proposition}
Part~(a): \citet{lassila2025base} proved that the original BASE score~\eqref{eq:base} is ROC-equivalent to RMIA; Corollary~\ref{cor:base_equiv} in the appendix shows that BASE is an instance of BASE1.
Part~(b) follows from substituting MLE into the Gaussian LLR~\eqref{eq:gauss_llr} (Corollary~\ref{cor:lira_base4}). This identity holds for $K \ge 64$; at smaller $K$ LiRA substitutes a global variance (Section~\ref{sec:smallk}), so deployed LiRA and raw BASE4 diverge in this region.  

The hierarchy BASE1 $\to$ BASE2 $\to$ BASE3 $\to$ BASE4 thus connects RMIA and LiRA as endpoints of a spectrum of increasing model complexity.
Moving from BASE1 to BASE4 trades robustness for expressiveness: BASE1 estimates no per-point dispersion parameters (maximally pooled), while BASE4 estimates four per-point parameters (maximally flexible).
The intermediate variants BASE2 and BASE3 offer controlled points in between; empirically, BASE2 (class-mean separation only) tracks BASE3 on AUC and outperforms BASE4 at small $K$ in the low-FPR tail, while BASE3 competes with BASE4 at moderate $K$ (Section~\ref{sec:experiments}).

\subsection{When Does More Structure Help?}

The hierarchy reveals a bias-variance tradeoff in the estimation of the LLR.
RMIA/BASE does not require an explicit distributional assumption in its original formulation; we show that it is equivalent to an exponential model with a pooled mean estimate.
LiRA/BASE4 models the full Gaussian with separate parameters per class, exploiting variance differences that carry strong membership signal, but requiring enough data to estimate all four parameters reliably.

This tradeoff predicts that LiRA should outperform RMIA when $K$ is large enough for per-point Gaussian estimation, while RMIA should be preferable at very small $K$ where Gaussian parameters are poorly estimated.

% ==================================================================
\section{BaVarIA: Bayesian Variance Inference Attack}\label{sec:bavaria}

The framework of Section~\ref{sec:framework} shows that LiRA is a Gaussian plug-in LLR with per-point MLE parameters.
When $K$ is small, these estimates---especially the variance---become unreliable, motivating a Bayesian treatment.

\subsection{The Small-$K$ Problem}\label{sec:smallk}

With $K$ shadow models and balanced membership, each point has approximately $K/2$ IN and $K/2$ OUT observations.
At $K = 8$, per-class variance is estimated from $\sim\!4$ samples---insufficient for reliable estimation.
\citet{carlini2022membership} address this with a hard switch: when fewer than $\sim\!32$ observations are available per class (i.e., $K < 64$ total shadow models), per-point variances are replaced by a single global variance computed across all points---the default in the released implementation, so our comparisons are against LiRA as deployed.
This approach has two limitations.
First, it is discontinuous: the attack behavior changes abruptly at the threshold.
Second, it is all-or-nothing: it cannot combine partial per-point information with global information.
A natural alternative is to smoothly interpolate between global and per-point estimates based on the available evidence.

\subsection{NIG Prior and Posterior Predictive}

We place a normal-inverse-gamma (NIG) prior on the class-conditional parameters:
\begin{equation}\label{eq:nig_prior}
(\mu_m, \sigma^2) \sim \NIG(\mu_{\varnothing,m}, \kappa_\varnothing, \alpha_\varnothing, \beta_{\varnothing,m}),
\end{equation}
where subscript $\varnothing$ denotes the prior---the hyperparameters before observing any point-specific data---estimated via empirical Bayes from pooled shadow statistics (details in Appendix~\ref{app:bavaria}); the class-specific parameters $\mu_{\varnothing,m}$ and $\beta_{\varnothing,m}$ are estimated separately for each class $m$, while the strength parameters $\kappa_\varnothing$ and $\alpha_\varnothing$ are shared.
Given $n_m = |\mathcal{K}_{i,m}|$ shadow observations with sample mean $\bar{z}_m$ and sum of squares $S_m$, the posterior is $\NIG(\mu'_{i,m}, \kappa'_{i,m}, \alpha'_{i,m}, \beta'_{i,m})$ with standard conjugate updates 
\begin{align}
\mu'_{i,m} &= \frac{\kappa_\varnothing \mu_{\varnothing,m} + n_m \bar{z}_m}{\kappa_\varnothing + n_m} \notag
\\
\kappa'_{i,m} &= \kappa_\varnothing + n_m, \quad 
\alpha'_{i,m} = \alpha_\varnothing + \tfrac{n_m}{2}, \notag
\\
\beta'_{i,m} &= \beta_{\varnothing,m} + \tfrac{S_m}{2} + \frac{\kappa_\varnothing n_m (\bar{z}_m - \mu_{\varnothing,m})^2}{2(\kappa_\varnothing + n_m)} 
\label{eq:nig_update1}
\end{align}

The posterior predictive for a new observation $z$ under class $m$ is a Student-$t$ distribution, $z \mid m \sim \Student\!\left(\nu_{i,m}, \tilde\mu_{i,m}, \tilde\sigma_{i,m}^2\right)$
\begin{equation}\label{eq:student_t}
\nu_{i,m} = 2\alpha'_{i,m},\; \tilde\mu_{i,m} = \mu'_{i,m},\; \tilde\sigma_{i,m}^2 = \frac{\beta'_{i,m}(\kappa'_{i,m} + 1)}{\alpha'_{i,m} \kappa'_{i,m}}.
\end{equation}

\subsection{Two BaVarIA Variants}

\paragraph{BaVarIA-$t$ (Student-$t$ predictive).}
The LLR using the posterior-predictive Student-$t$ densities is
\begin{equation}\label{eq:bavaria_t}
\LLR(z) = \log \frac{t_{\nu_{i,1}}\!\left(\frac{z - \tilde\mu_{i,1}}{\tilde\sigma_{i,1}}\right)}{t_{\nu_{i,0}}\!\left(\frac{z - \tilde\mu_{i,0}}{\tilde\sigma_{i,0}}\right)} - \log \frac{\tilde\sigma_{i,1}}{\tilde\sigma_{i,0}},
\end{equation}
where $t_\nu$ is the standard Student-$t$ density with $\nu$ degrees of freedom.
The heavier tails of the $t$-distribution absorb parameter uncertainty, providing stability in the small-$K$ regime.

\paragraph{BaVarIA-$n$ (Gaussian with Bayesian variance).}
To isolate the effect of better variance estimation from the heavier-tailed predictive, we define a hybrid variant:
use MLE means $\hat\mu_{i,m}$ (as in LiRA) but replace the MLE variance with the NIG posterior mean $\E[\sigma^2 \mid \text{data}] = \beta'_{i,m} / (\alpha'_{i,m} - 1)$, then compute the standard Gaussian LLR~\eqref{eq:gauss_llr}.
This provides Bayesian shrinkage of the variance toward the global prior, serving as a continuous relaxation of the hard-switch strategy, while preserving the Gaussian LLR form.

% ==================================================================
\section{Experiments}\label{sec:experiments}

\subsection{Setup}

\paragraph{Datasets.}
We use 12 testbeds based on public datasets spanning image and tabular domains:
6 image (CIFAR-10/100 \citep{krizhevsky2009learning}, CINIC-10 \citep{darlow2018cinic}, each with ResNet-18 \citep{he2016deep} and WideResNet \citep{zagoruyko2016wide})
and 6 tabular (Location, Purchase100, Texas100, each with 3- and 4-layer MLPs).

\paragraph{Protocol.}
For each dataset, we train $K + 1$ models using the ``Design~B'' protocol of \citet{carlini2022membership}: one target model and $K$ shadow models, each trained on a random half of the dataset.
We evaluate at $K \in \{4, 8, 16, 32, 64, 128, 254\}$ and average over 32 replicates (rotated target).

\paragraph{Methods.}
We compare LiRA (= BASE4), RMIA (= BASE1), BASE3, BaVarIA-$n$, and BaVarIA-$t$.

\paragraph{Metrics.}
We report AUC and TPR at FPR $= 0.01$ (ROC curves over the full FPR range in Appendix~\ref{app:curves}).

\subsection{Main Results}

\begin{figure*}%[t]
\centering
\begin{subfigure}{\textwidth}\centering\includegraphics[width=0.82\linewidth]{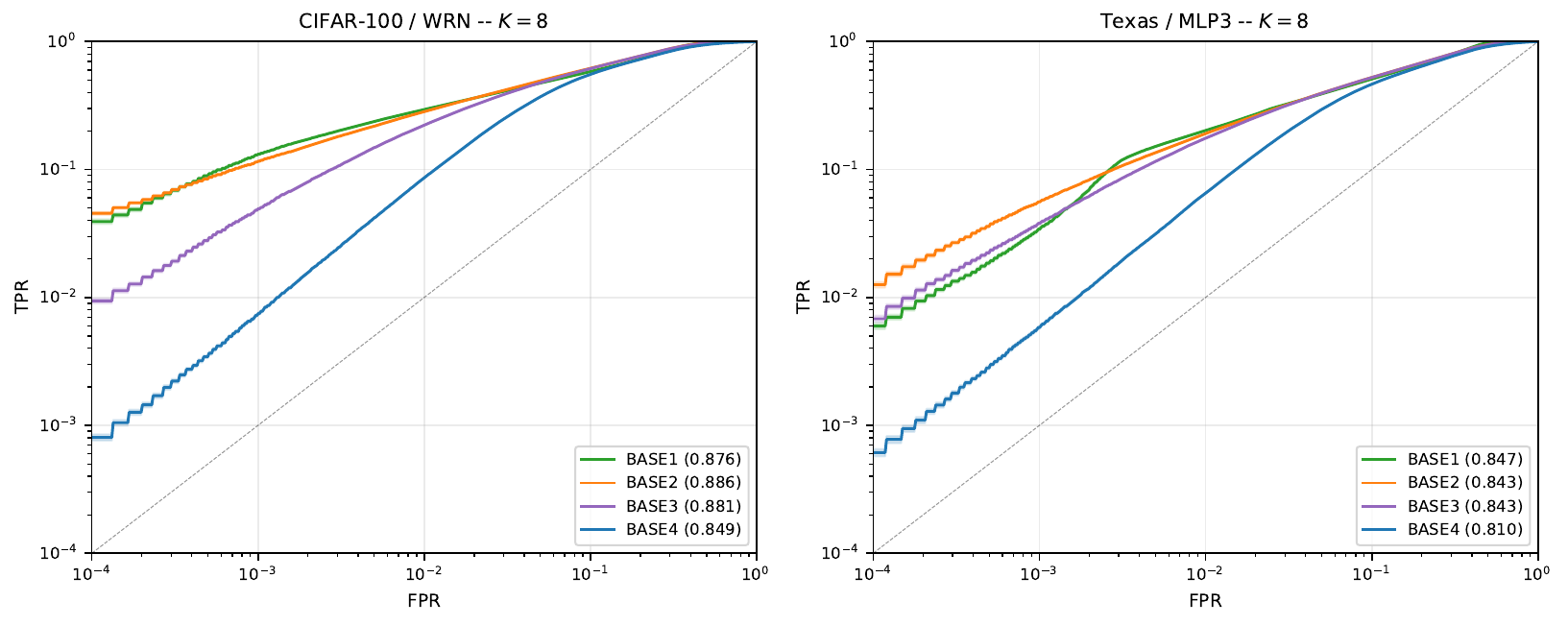}\caption{$K = 8$: simpler families lead in the low-FPR tail.}\end{subfigure}\\[3pt]
\begin{subfigure}{\textwidth}\centering\includegraphics[width=0.82\linewidth]{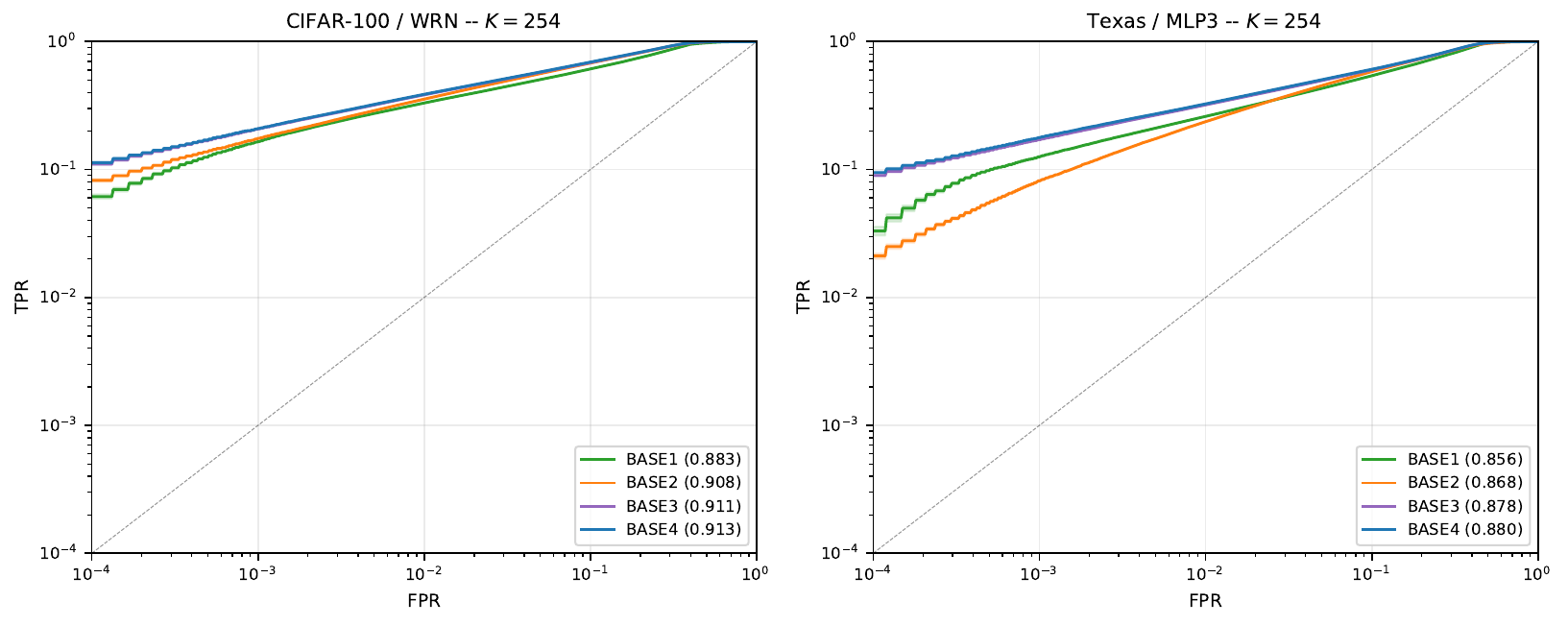}\caption{$K = 254$: richer families lead, reversing the small-$K$ ordering.}\end{subfigure}
\caption{Online ROC for the BASE hierarchy (BASE1 $=$ RMIA at $\gamma{=}1$; BASE4 $=$ LiRA for $K\ge64$) on CIFAR-100/WRN (left) and Texas/MLP3 (right) at the two budget extremes $K \in \{8, 254\}$, mean over 32 replicates. The method ordering reverses with $K$: simpler families win at small $K$, richer families at large $K$ (the intermediate $K=64$ and the continuous crossover are in Figure~\ref{fig:scaling}). 
}\label{fig:roc}
\end{figure*}

\begin{figure*}%[t]
\centering
\includegraphics[width=0.9\linewidth]{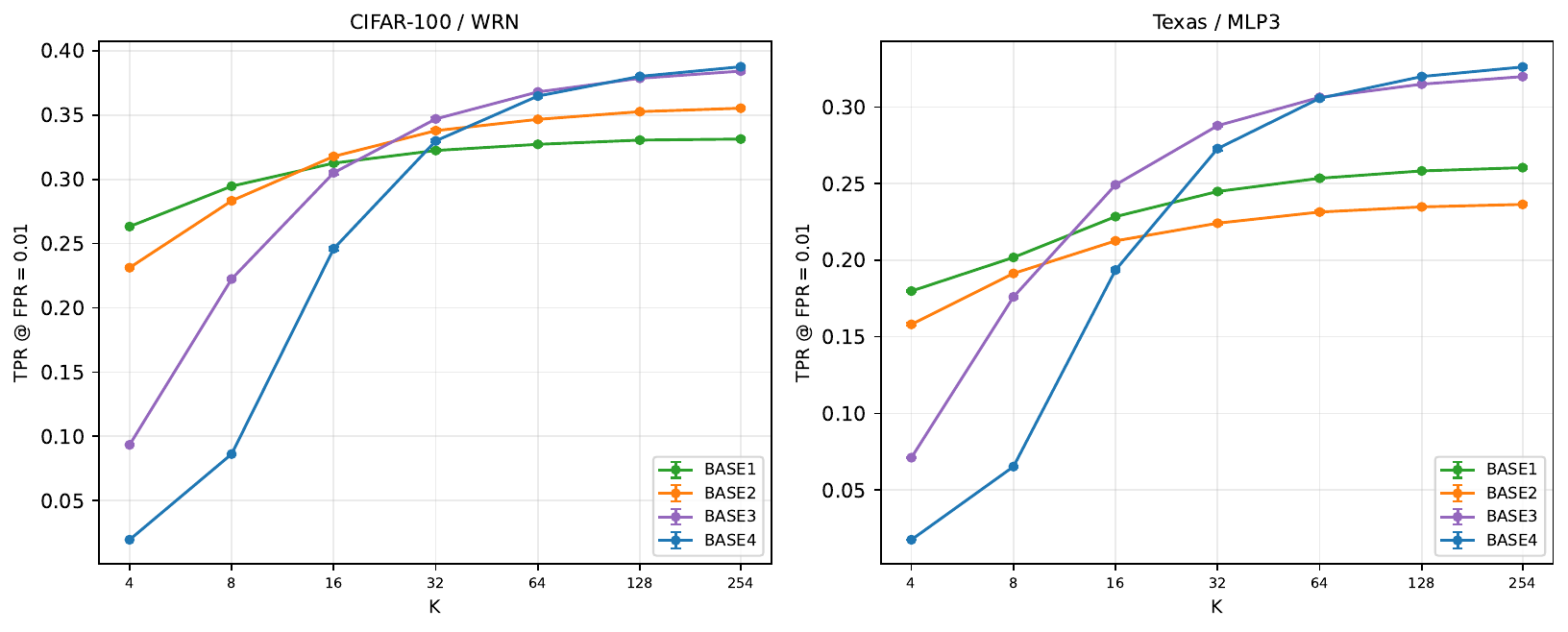}
\caption{TPR @ FPR$=0.01$ vs.\ shadow-model budget $K$ for the BASE hierarchy (CIFAR-100/WRN and Texas/MLP3), mean $\pm 1$ SE over 32 replicates. Simpler families lead at small $K$ and richer families at large $K$; BASE2 (separate IN/OUT means, constant variance) is the informative two-parameter intermediate.}\label{fig:scaling}
\end{figure*}

\begin{table*}%[h]
\centering
\footnotesize
\begin{tabular}{lcccccc}
\toprule
 & \multicolumn{3}{c}{$K=8$ \emph{(simple wins)}} & \multicolumn{3}{c}{$K=254$ \emph{(complex wins)}} \\
\cmidrule(lr){2-4} \cmidrule(lr){5-7}
Testbed & AUC & TPR@.01 & TPR@.001 & AUC & TPR@.01 & TPR@.001 \\
\midrule
CIFAR-10 / WRN & \textbf{+1.00}$^{\star}$ & \textbf{+1.00}$^{\star}$ & \textbf{+1.00}$^{\star}$ & \textbf{+1.00}$^{\star}$ & \textbf{+1.00}$^{\star}$ & \textbf{+1.00}$^{\star}$ \\
CIFAR-100 / WRN & +0.72$^{\star}$ & \textbf{+1.00}$^{\star}$ & \textbf{+1.00}$^{\star}$ & \textbf{+1.00}$^{\star}$ & \textbf{+1.00}$^{\star}$ & \textbf{+1.00}$^{\star}$ \\
CINIC-10 / WRN & +0.07$^{\star}$ & +0.74$^{\star}$ & +0.79$^{\star}$ & \textbf{+1.00}$^{\star}$ & \textbf{+1.00}$^{\star}$ & \textbf{+0.96}$^{\star}$ \\
Location / MLP3 & \textit{-0.53} & +0.75$^{\star}$ & \textbf{+0.99}$^{\star}$ & \textbf{+1.00}$^{\star}$ & \textbf{+1.00}$^{\star}$ & \textbf{+1.00}$^{\star}$ \\
Purchase / MLP3 & \textbf{+1.00}$^{\star}$ & \textbf{+1.00}$^{\star}$ & \textbf{+1.00}$^{\star}$ & +0.53$^{\star}$ & +0.80$^{\star}$ & +0.65$^{\star}$ \\
Texas / MLP3 & \textbf{+1.00}$^{\star}$ & \textbf{+1.00}$^{\star}$ & +0.65$^{\star}$ & \textbf{+1.00}$^{\star}$ & +0.85$^{\star}$ & +0.73$^{\star}$ \\
CIFAR-10 / RN18 & \textbf{+1.00}$^{\star}$ & \textbf{+1.00}$^{\star}$ & \textbf{+1.00}$^{\star}$ & \textit{-0.50} & +0.62$^{\star}$ & +0.87$^{\star}$ \\
CIFAR-100 / RN18 & -0.41 & +0.37$^{\star}$ & +0.82$^{\star}$ & \textbf{+1.00}$^{\star}$ & \textbf{+1.00}$^{\star}$ & \textbf{+1.00}$^{\star}$ \\
CINIC-10 / RN18 & -0.34 & +0.55$^{\star}$ & +0.90$^{\star}$ & \textbf{+1.00}$^{\star}$ & \textbf{+1.00}$^{\star}$ & \textbf{+1.00}$^{\star}$ \\
Location / MLP4 & -0.45 & +0.36$^{\star}$ & \textbf{+0.90}$^{\star}$ & \textbf{+1.00}$^{\star}$ & \textbf{+0.99}$^{\star}$ & \textbf{+0.93}$^{\star}$ \\
Purchase / MLP4 & \textbf{+1.00}$^{\star}$ & \textbf{+0.93}$^{\star}$ & +0.82$^{\star}$ & \textbf{+0.90}$^{\star}$ & \textbf{+0.92}$^{\star}$ & +0.87$^{\star}$ \\
Texas / MLP4 & +0.57$^{\star}$ & \textbf{+1.00}$^{\star}$ & \textbf{+1.00}$^{\star}$ & \textbf{+1.00}$^{\star}$ & \textbf{+0.97}$^{\star}$ & +0.76$^{\star}$ \\
\midrule
\textbf{Mean} & +0.39 & +0.81 & \textbf{+0.90} & +0.83 & \textbf{+0.93} & +0.90 \\
\textbf{Median} & +0.65 & \textbf{+0.97} & \textbf{+0.94} & \textbf{+1.00} & \textbf{+1.00} & \textbf{+0.95} \\
\bottomrule
\end{tabular}
\caption{Weighted concordance $S \in [-1, +1]$ of the observed method ordering with the BASE hierarchy's prediction (\emph{simple wins} at $K=8$: BASE1\,$>$\,BASE2\,$>$\,BASE3\,$>$\,BASE4; \emph{complex wins} at $K=254$: the reverse). $S=+1$: all pairs agree; $-1$: fully reversed. \textbf{Bold}: $S \ge 0.9$; \emph{italic}: $S \le -0.5$. $\star$: ordering prediction is significantly supported at 95\% confidence (bootstrap percentile over the 32 replicates lies above~0).}
\label{tab:summary32_concordance}
\end{table*}

Figure~\ref{fig:roc} shows online ROC curves for the BASE hierarchy at the two budget extremes: at $K = 8$ the simpler families (BASE1, BASE2) lead in the low tail (FPR $\le 0.01$)---pooled estimation avoids the IN/OUT split that starves per-class variance estimates---whereas at $K = 254$ the richer families (BASE3, BASE4 $=$ LiRA) lead once there is enough data to estimate per-point parameters. Within the simpler pair the ordering is less clear-cut at very low FPR (FPR $\approx 0.001$), where BASE2 can hold up better than the fully pooled BASE1 (e.g.\ Texas; Appendix~\ref{app:base_tables}).
Figure~\ref{fig:scaling} traces this crossover from simpler to richer across the full budget range: the reversal occurs around $K = 32$--$64$, the budget at which per-class variance estimation becomes reliable, and BASE2 is the informative two-parameter intermediate, tracking BASE3 throughout.
Table~\ref{tab:summary32_concordance} quantifies the reversal with a weighted concordance score $S \in [-1, +1]$ between the observed and predicted method orderings ($+1$: all pairs agree, $-1$: fully reversed; defined in Appendix~\ref{app:base_tables}); for TPR at low FPR the predicted ordering is supported on all $12/12$ testbeds at both budget extremes, with bootstrap 95\% confidence intervals excluding 0 (Appendix~\ref{app:base_tables}). 
This supports reading the hierarchy as a practitioner guide: calibrate the attack family to the available shadow-model budget---simpler at small $K$, richer once the budget supports per-point variance estimation.

\subsection{BaVarIA Ablation}
\label{sec:bavariaablation}

\begin{figure*}%[t]
\centering
\begin{subfigure}{\textwidth}\centering\includegraphics[width=0.8\linewidth]{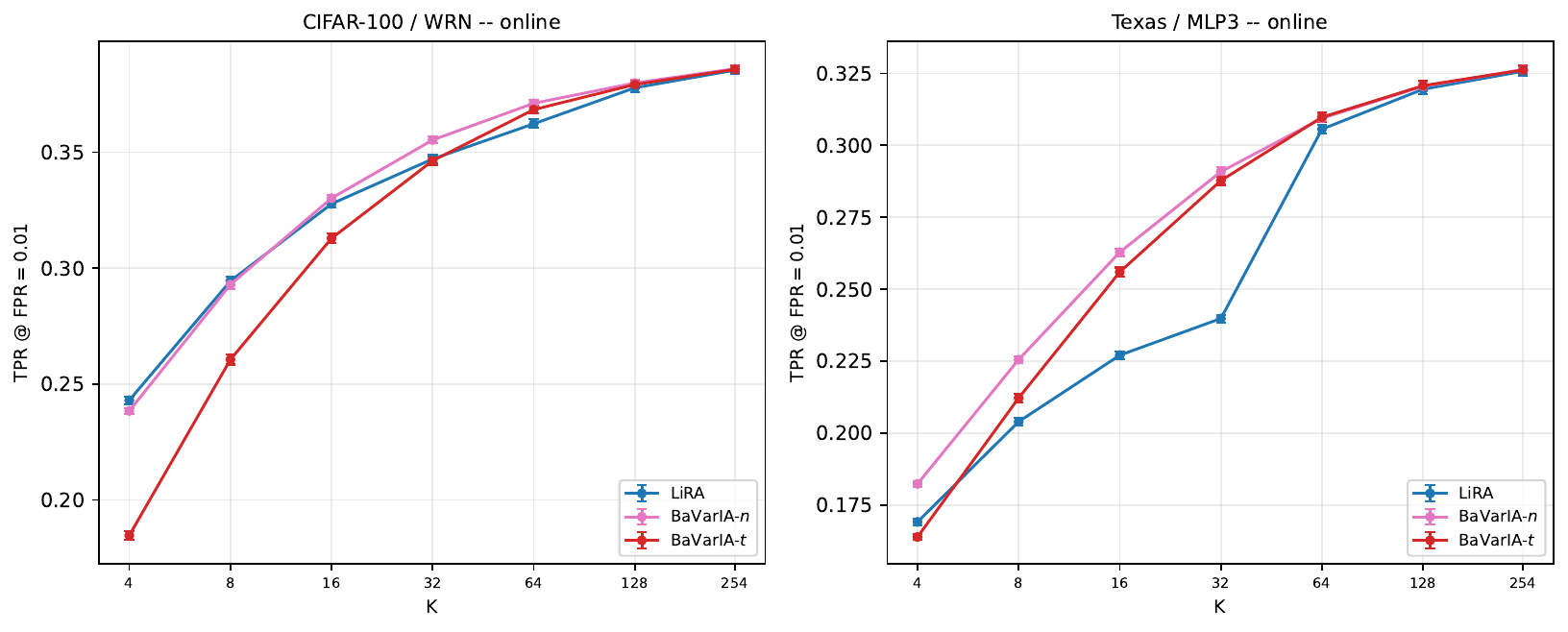}\caption{Online.}\end{subfigure}\\[3pt]
\begin{subfigure}{\textwidth}\centering\includegraphics[width=0.8\linewidth]{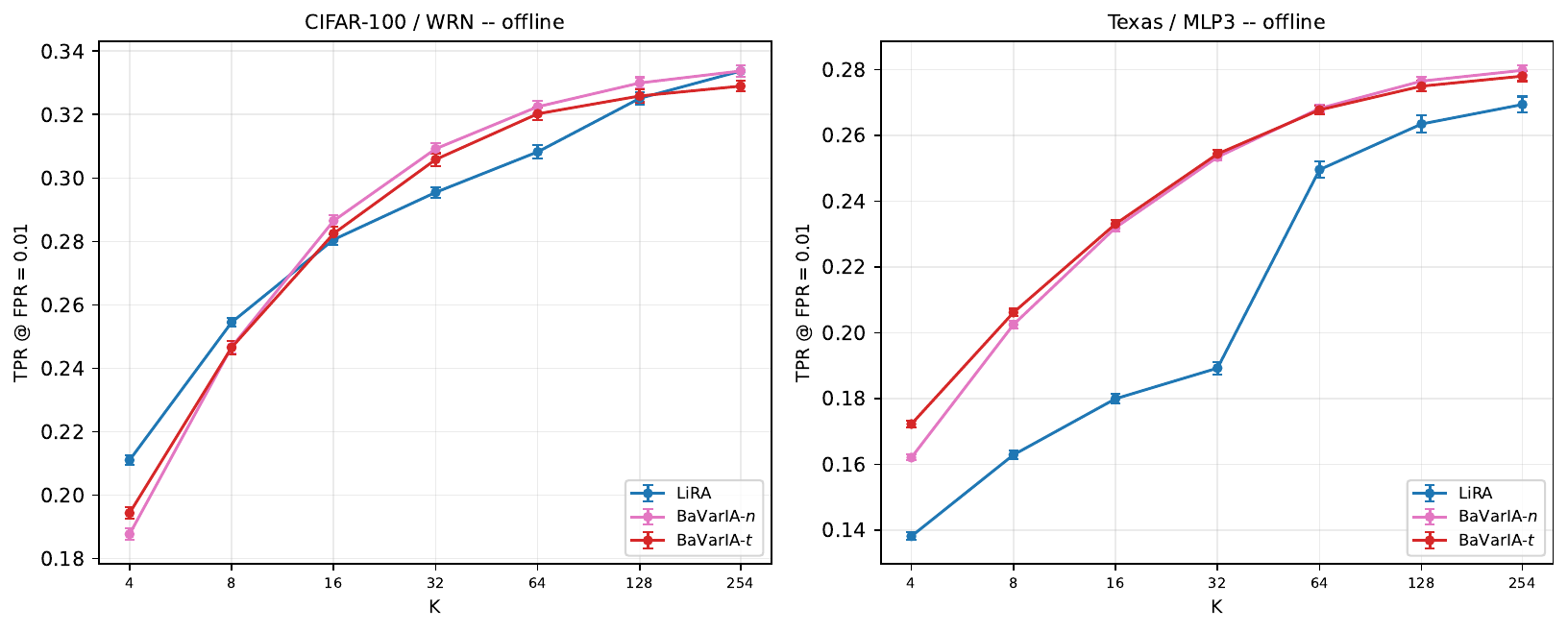}\caption{Offline.}\end{subfigure}
\caption{Online (top) and offline (bottom) TPR @ FPR $=0.01$ vs.\ $K$ for LiRA, BaVarIA-$n$, BaVarIA-$t$ on CIFAR-100/WRN and Texas/MLP3, mean $\pm 1$ SE over 32 replicates. Online, BaVarIA's gain peaks near the $K=64$ cutover and vanishes as $K \to \infty$; offline, the gain persists at every $K$ because the NIG prior substitutes for the absent IN class.}\label{fig:ablation}
\end{figure*}

\begin{table*}%[t]
\centering
\footnotesize
\begin{tabular}{l ccc ccc}
\toprule
 & \multicolumn{3}{c}{Online} & \multicolumn{3}{c}{Offline} \\
\cmidrule(lr){2-4} \cmidrule(lr){5-7}
Testbed & LiRA & BaVarIA-$n$ & BaVarIA-$t$ & LiRA & BaVarIA-$n$ & BaVarIA-$t$ \\
\midrule
CIFAR-10 / WRN   & 0.129 & \textbf{0.134} & 0.133          & \textbf{0.087} & 0.074 & 0.069 \\
CIFAR-100 / WRN  & 0.362 & \textbf{0.371} & 0.368          & 0.308 & \textbf{0.322} & 0.320 \\
CINIC-10 / WRN   & 0.180 & \textbf{0.187} & 0.187          & 0.123 & \textbf{0.181} & 0.178 \\
Location / MLP3  & 0.423 & \textbf{0.440} & 0.439          & 0.327 & \textbf{0.412} & 0.405 \\
Purchase / MLP3  & 0.031 & 0.035 & \textbf{0.036}          & 0.023 & \textbf{0.026} & 0.023 \\
Texas / MLP3     & 0.306 & 0.309 & \textbf{0.310}          & 0.250 & \textbf{0.268} & 0.268 \\
CIFAR-10 / RN18  & 0.039 & \textbf{0.042} & \textbf{0.042} & \textbf{0.038} & 0.032 & 0.031 \\
CIFAR-100 / RN18 & 0.663 & 0.669 & \textbf{0.669}          & 0.579 & 0.663 & \textbf{0.664} \\
CINIC-10 / RN18  & 0.288 & \textbf{0.293} & 0.291          & 0.223 & \textbf{0.277} & 0.273 \\
Location / MLP4  & 0.352 & \textbf{0.364} & 0.361          & 0.299 & \textbf{0.333} & 0.326 \\
Purchase / MLP4  & 0.053 & 0.059 & \textbf{0.059}          & 0.038 & \textbf{0.042} & 0.039 \\
Texas / MLP4     & 0.358 & \textbf{0.359} & 0.359          & 0.294 & \textbf{0.321} & 0.316 \\
\bottomrule
\end{tabular}
\caption{Online and offline TPR @ FPR$=0.01$ at $K=64$, per testbed, mean over 32 replicates; best per row \emph{within each setting} in bold. BaVarIA-$n$ beats LiRA on $12/12$ testbeds online and $10/12$ offline (offline, LiRA wins only CIFAR-10).
The BaVarIA-$n$ vs.\ LiRA differences are significant at $95\%$ confidence (Appendix~\ref{app:significance}).}
\label{tab:bavaria32_tpr001_K64}
\end{table*}

Figure~\ref{fig:ablation} shows BaVarIA's plug-in gain over LiRA online and offline, and---by contrasting the two variants---separates its two components, since BaVarIA-$n$ uses the same Gaussian LLR as LiRA (isolating the variance-estimation effect) while BaVarIA-$t$ adds the Student-$t$ predictive on top (the tail effect).

\emph{Better variance estimation} (BaVarIA-$n$): the NIG posterior mean shrinks noisy per-point variances toward a global prior, giving consistent improvement at $K \geq 16$ and converging to the MLE at large $K$---e.g.\ $+0.018$ in TPR@0.01 over LiRA at $K = 32$, averaged over 12 testbeds.

\emph{Heavier tails} (BaVarIA-$t$): the Student-$t$ predictive improves AUC at all $K$ (e.g.\ $\Delta\text{AUC} = +0.009$ at $K = 4$) but, when $K$ is small, its thick tails ($\nu \approx 6$ at $K = 4$) inflate scores for members and non-members alike, hurting TPR at very low FPR ($\Delta\text{TPR@0.01} = -0.016$ vs.\ $+0.003$ for BaVarIA-$n$).
This is a ranking-precision tradeoff: accounting for parameter uncertainty improves global ranking (AUC) at the cost of tail precision. Per-dataset results are in Appendix~\ref{app:bavaria_results}.

Table~\ref{tab:bavaria32_tpr001_K64} gives the per-testbed picture at $K = 64$, LiRA's variance-rule cutover.
BaVarIA-$n$ beats or ties LiRA on $12/12$ testbeds online and $10/12$ offline (column-mean averages, online/offline: LiRA $0.266$/$0.216$, BaVarIA-$n$ $0.272$/$0.246$, BaVarIA-$t$ $0.271$/$0.242$). The advantage is significant at $95\%$ confidence on every winning testbed (paired Wilcoxon signed-rank, Holm-corrected; Appendix~\ref{app:significance}, Table~\ref{tab:bavaria32_significance}).

The offline gain is the larger and, unlike the online gain, does not vanish as $K$ grows: LiRA's offline score discards the IN-class likelihood (heuristic rescaling of an OUT-only score), whereas BaVarIA substitutes the NIG conjugate prior, giving an offline attack at LiRA-equivalent per-point cost.
Appendix~\ref{app:offline} reports the full offline evaluation across all 12 testbeds and every method (Tables~\ref{tab:offline_auc}--\ref{tab:offline_tpr}): no single method dominates, but all remain viable offline.

\subsection{Distributional Diagnostics}

Per-datapoint Anderson-Darling tests (applied to each point's ${\sim}128$ IN and OUT shadow log-odds separately) reveal a range of departures from Gaussianity across datasets.
On CIFAR-10 WideResNet, only ${\sim}6\%$ of datapoints reject normality at the 5\% level---close to the nominal rate---indicating that per-datapoint Gaussianity is a good approximation.
On CIFAR-100 ResNet, $74\%$ of IN distributions reject normality, with heavy tails visible in the QQ plots (Appendix~\ref{app:qq}).
Despite these deviations, the Gaussian-family methods (BASE3, BASE4/LiRA, BaVarIA) lead the alternatives at moderate-to-large $K$.
This robustness is expected: the LLR score is evaluated at the \emph{target} model's output, which typically lies in the bulk of the distribution where the Gaussian approximation is accurate.
These results suggest that moderate departures from Gaussianity do not preclude effective use of LiRA or BaVarIA in practice.

\subsection{Cross-Validation}

To verify that our findings generalize beyond a single training pipeline, we cross-validated all methods on an independent shadow-model collection trained without data augmentation or dropout (224--1,480 shadows per dataset).
Models without augmentation exhibit substantially higher vulnerability on image datasets (e.g., CIFAR-10 ResNet AUC increases from 0.613 to 0.874), consistent with the known protective effect of data augmentation.
However, the qualitative ordering is preserved across both data sources: the Gaussian-family attacks lead---BaVarIA and LiRA (with BASE4 $\equiv$ LiRA at $K \geq 64$) at the top, then BASE3---while the pooled BASE1 ($=$ RMIA) trails. Within the BaVarIA pair, BaVarIA-$t$ is marginally stronger on AUC, while at low FPR (TPR@0.01) the two variants are within noise of each other on the independent collection.
At $K = 64$ on the independent data, BaVarIA-$n$ beats LiRA on TPR@0.01 in 10 of 12 testbeds, with the exceptions being the two CIFAR-100 testbeds where LiRA retains a marginal edge ($|\Delta\text{TPR}| < 0.003$).
This supports our findings across the training recipes and regularization settings tested (Appendix~\ref{app:crossval}).

\subsection{Defended Models (DP-SGD)}\label{sec:dpsgd}
To assess the methods against a privacy-preserving defense, we retrained the entire pipeline (target and shadow models) with DP-SGD \citep{abadi2016deep}---per-sample gradient clipping plus Gaussian noise---and ran LiRA, BASE/RMIA, and BaVarIA on the resulting models for four testbeds at $K=64$ across a noise-multiplier grid $\sigma \in \{0, 0.01, 0.1, 1.0\}$.
As the DP-compatible pipeline changes the underlying models (GroupNorm in place of BatchNorm, retuned batch size and learning rate; Appendix~\ref{app:dpsgd}), absolute risk numbers are not comparable to the main-paper testbeds; Table~\ref{tab:dpsgd_main} is read across methods within a row.

Table~\ref{tab:dpsgd_main} reports $\sigma = 1.0$, the only setting in our $\sigma$-grid with a non-vacuous privacy guarantee.
The low-FPR attack channel is all but closed---every method lands within $0.005$ of the $0.01$ random-guessing rate, at a heavy cost in accuracy---and in that compressed regime BaVarIA-$n$ and LiRA are indistinguishable while BASE/RMIA scores best on all four testbeds.
This is consistent with what the framework already predicts for small $K$: pooling wins when there is too little signal for per-point variance estimation.

Across the rest of the grid two further patterns emerge (Appendix~\ref{app:dpsgd}).
First, BaVarIA-$n$'s advantage over LiRA persists under noise---it matches or beats LiRA throughout the grid in both metrics (Tables~\ref{tab:auc} and~\ref{tab:tpr01})---with margins that are small throughout and have vanished by $\sigma = 1.0$.
Second, BASE/RMIA progressively overtakes the Gaussian methods as the injected noise increases, so the $\sigma = 1.0$ ordering above is the endpoint of a trend rather than an isolated result.

\begin{table}
\centering
\caption{Online TPR @ FPR$=0.01$ under DP-SGD at noise multiplier $\sigma=1.0$ ($K=64$, mean over 32 replicates). Best per row in bold; TPR $=0.01$ is the random-guessing level at this FPR. The full grid $\sigma \in \{0, 0.01, 0.1, 1.0\}$, privacy budgets, utility, and ROC curves in Appendix~\ref{app:dpsgd}.}\label{tab:dpsgd_main}
\small
\begin{tabular}{lr ccc}
\toprule
Testbed & $\varepsilon$ & LiRA & BaVarIA-$n$ & BASE/RMIA \\
\midrule
CIFAR-10 / WRN  & 6.7 & .010 & .011 & \textbf{.013} \\
CIFAR-100 / WRN & 6.1 & .009 & .009 & \textbf{.012} \\
Location / MLP3 & 8.9 & .012 & .012 & \textbf{.015} \\
Purchase / MLP3 & 5.7 & .010 & .011 & \textbf{.012} \\
\bottomrule
\end{tabular}
\end{table}

% ==================================================================
\section{Related Work}\label{sec:related}

\paragraph{Membership inference attacks.}
MIA was introduced by \citet{shokri2017membership}, who trained binary classifiers on shadow-model outputs to distinguish members from non-members.
Subsequent work simplified the attack surface: \citet{yeom2018privacy} showed that thresholding the loss suffices, and \citet{sablayrolles2019white} proved that the Bayes-optimal score is the log-likelihood ratio under IN vs.\ OUT, deriving a global-threshold (MALT) to per-sample (MAST) approximation hierarchy.
BaVarIA's NIG prior moves along this axis: a global prior (MALT-like) at small $K$, per-point MLE (MAST-like) at large $K$.
\citet{carlini2022membership} operationalized this insight with LiRA, fitting per-point Gaussian models to shadow-model log-odds and achieving state-of-the-art results at the time, albeit at the cost of many shadow models.
\citet{ye2022enhanced} and \citet{watson2022importance} developed calibrated variants that account for difficulty heterogeneity across data points.
\citet{salem2019ml} demonstrated that relaxed threat models (fewer shadow models, different architectures) still yield effective attacks.
More recently, \citet{zarifzadeh2024lowcost} introduced RMIA, which compares the target output to a population reference rather than estimating per-point parameters, achieving competitive performance with lower computational cost.
\citet{lassila2025base} introduced BASE in the context of MIA on graph neural networks, which simplifies for i.i.d.\ data to the conventional MIA setup, and proved that BASE and RMIA are ROC-equivalent, simplifying the picture.
Our framework unifies these approaches: LiRA, RMIA, and BASE are instances of exponential-family LLR testing with different distributional assumptions and parameter-sharing constraints.

\paragraph{Recent attacks and their pipeline stage.}
Several recent works improve distinct stages of the score-based MIA pipeline and are largely orthogonal to the likelihood-ratio computation that our BASE hierarchy and BaVarIA govern.
GLiRA \citep{galichin2024glira} and IMIA \citep{imia2025} target \emph{shadow-model acquisition}: GLiRA distills shadow models in the architecture-unknown setting, while IMIA trains target-informed imitative shadows. Both use a standard LiRA-style score---so any BASE or BaVarIA score can be substituted for their LLR without changing their training procedure. IMIA's main scoring variant is, on re-derivation, a quadratic in the IN/OUT means with constant variance; it coincides with our BASE2 point (separate means, constant variance), consistent with the IMIA paper attributing its gains to the shadow-training procedure rather than the scoring rule.
RAPID \citep{he2024rapid} is a standalone learned attack over two features: a difficulty-calibrated score---the loss minus a reference-model mean, a location correction closely related to BASE---together with the raw loss, which it reuses to offset calibration errors on high-loss non-members. Its calibrated component could in principle be replaced by any score in our hierarchy, but substituting a variance-aware score (LiRA or BaVarIA-$n$) yields diminishing returns in the low-FPR tail: adding RAPID's raw-loss feature lifts a plain mean-only calibration substantially yet barely moves---or slightly reduces---the tail metrics of an already variance-aware score in our experiments (Appendix~\ref{app:rapid}, Table~\ref{tab:rapid_varianceaware}), consistent with the raw-loss correction addressing a weakness specific to mean-only calibration.
Summarizing in our terms, GLiRA and IMIA change how shadow outputs are \emph{acquired}, while RAPID adds a learned correction to a location-calibrated score; the BASE hierarchy and BaVarIA govern the \emph{likelihood-ratio computation} that all of them ultimately rely on.

\paragraph{Privacy auditing.}
MIAs serve as empirical lower bounds on privacy leakage, complementing theoretical guarantees from differential privacy \citep{dwork2006calibrating}.
\citet{nasr2023tight} demonstrated that well-calibrated MIAs can produce tight auditing bounds that approach the theoretical $(\varepsilon, \delta)$-DP guarantee.
\citet{steinke2024privacy} showed that meaningful auditing is possible with only a single training run, reducing the computational burden.
\citet{song2021systematic} provided a systematic evaluation framework for comparing MIA methods across diverse settings.
Our work contributes the new BaVarIA attack and a method-selection criterion that calibrates the choice of attack to the available shadow-model budget $K$.

\paragraph{Bayesian methods and robust estimation.}
Conjugate Bayesian analysis of Gaussian parameters via normal-inverse-gamma priors is classical \citep{murphy2007conjugate, gelman2013bayesian}, and the resulting Student-$t$ predictive is well understood.
Shrinkage estimation \citep{james1961estimation} and empirical Bayes methods \citep{morris1983parametric, efron2010large} provide established frameworks for borrowing strength across related estimation problems.
Our contribution applies this machinery to variance stabilization in membership inference: the NIG prior shrinks per-point variance estimates toward a global value, with the effect diminishing as the per-point sample grows, replacing the hard switch of \citet{carlini2022membership} with a continuous Bayesian alternative.

The exponential-family framework also accommodates Gamma and Beta distributional assumptions, and a discriminative extension (ELSA, the Exponential-family Log-likelihood-ratio Score Attack) that learns the LLR feature map; these are developed in Appendices~\ref{app:distributions} and~\ref{app:elsa}.
ELSA recovers the Gaussian LLR discriminatively but does not improve on the generative estimators; we therefore foreground the variance-stabilizing BaVarIA in the main text.

% ==================================================================
\section{Conclusion}\label{sec:conclusion}

We presented an exponential-family framework that unifies the leading membership inference attacks---LiRA, RMIA, and the recently proposed BASE---as instances of log-likelihood ratio testing under different distributional assumptions.
The resulting BASE hierarchy (BASE1--4) reveals these attacks as points on a spectrum from maximal pooling (RMIA/BASE) to full per-point estimation (LiRA/BASE4).
Within this framework, we identified variance estimation as the primary source of degradation at small $K$ and proposed BaVarIA, which replaces LiRA's threshold-based variance switching with conjugate NIG Bayesian inference.
The ablation of BaVarIA-$n$ (Gaussian with Bayesian variance) and BaVarIA-$t$ (Student-$t$ predictive) separates the effects of better variance estimation from heavier-tailed predictives, giving practitioners guidance: BaVarIA-$n$ for low-FPR auditing, BaVarIA-$t$ for general AUC.

\paragraph{Limitations.}
Our evaluation uses Design~B resampling~\citep{carlini2022membership} (each shadow model trains on a random half of the dataset; Section~\ref{sec:experiments}) with 32 replicates, which provides stable estimates for AUC and TPR@0.01 but limits precision for TPR@0.001---we note that the replicates are not fully independent as all draw their target and shadow models from a single shared pool.
All methods assume access to shadow models trained on data from the same distribution; out-of-distribution transfer is not evaluated.
The NIG hyperparameters are set via simple empirical Bayes defaults; more sophisticated hierarchical estimation could further improve small-$K$ performance.
Our evaluation covers CIFAR-scale image models and mid-sized tabular classifiers; behavior at larger scales (e.g., ImageNet-scale vision models or language models) is left to future work.

\paragraph{Practical recommendation.}
Our main contribution is the BASE hierarchy as a guide for calibrating a privacy-risk audit to the available shadow-model (computational) budget: simpler attacks (RMIA/BASE) at small budgets and in the low-FPR tail, richer Gaussian attacks (LiRA, BaVarIA) once the budget allows per-class variance estimation.
Use BaVarIA-$n$ as a drop-in replacement for LiRA: 
it is rarely worse, and often better (especially at small and intermediate $K$, where BaVarIA's Bayesian shrinkage is most effective),
requires no per-dataset hyperparameter tuning, and runs in comparable time.
It additionally provides an offline attack that outperforms LiRA on 10 of 12 testbeds across the shadow-model budgets tested (exceptions: the two CIFAR-10 testbeds; Appendix~\ref{app:significance}), as the NIG prior at $n_\mathrm{in}=0$ supplies the IN-class likelihood that LiRA's offline heuristic discards.
For applications where AUC is the primary metric, BaVarIA-$t$ provides a small additional improvement at all shadow budgets.

\paragraph{Future work.}
The exponential-family framework naturally extends to multi-output statistics (e.g., full softmax vectors rather than scalar confidence).
Discriminative feature-map learning (ELSA, Appendix~\ref{app:elsa})---for instance with automatic relevance determination for feature selection---is a natural extension; whether it can surpass the generative estimators given sufficient shadow models remains open.

% ==================================================================
\bibliographystyle{tmlr}
\bibliography{main}

\newpage
\appendix
\section*{Supplementary Material}

% ------------------------------------------------------------------
\section{Broader Impact Statement}\label{app:broader_impact}
Membership inference is dual-use: the techniques an adversary uses to test whether a record was in a training set are also the standard tool for auditing privacy leakage and validating defenses---the defensive use we target throughout. Our contribution is unifying and diagnostic rather than a new attack capability: we show that LiRA, RMIA, and BASE are instances of a single likelihood-ratio framework and add a variance-stabilized estimator (BaVarIA) whose advantage is confined to the low-shadow-model regime, building only on threat models and primitives already public. The benefit is cheaper, better-calibrated audits that help practitioners detect and remediate leakage before deployment, and our DP-SGD evaluation (Appendix~\ref{app:dpsgd}) shows the same tools quantifying the protection a defense confers. The risk is that stronger or cheaper attacks could be misused against deployed models; standard mitigations apply---differential privacy, limiting model and confidence-score exposure, and responsible disclosure when auditing third-party systems. On balance, we believe clearer, better-calibrated auditing methods strengthen privacy practice more than they enable misuse.

% ------------------------------------------------------------------
\section{Exponential-Family Derivations}\label{app:derivations}

This appendix states the log-likelihood ratio implied by each distributional family and discusses parameter estimation when the families are used directly for MIA.  Derivations of the BASE hierarchy are in Appendix~\ref{app:base_formulas}.

\subsection{Generic Form}

Recall from Section~\ref{sec:framework} that each distributional family models a scalar statistic $z$ under the two membership classes $m \in \{0,1\}$, where $m = 1$ denotes IN and $m = 0$ denotes OUT.
We assume $z \mid m$ follows an exponential-family distribution
\[
p(z \mid m) = h(z)\,\exp\!\bigl(\eta_m^\top T(z) - A(\eta_m)\bigr),
\]
with sufficient statistics $T(z)$, natural parameters $\eta_m$, base measure $h(z)$, and log-partition function~$A$~\eqref{eq:expfam}.
The log-likelihood ratio for membership is
\[
\LLR(z) = (\eta_1 - \eta_0)^\top T(z) - \bigl(A(\eta_1) - A(\eta_0)\bigr),
\]
restating~\eqref{eq:expfam_llr}.
The LLR is affine in $T(z)$: the optimal membership score is a linear combination of the sufficient statistics plus a constant offset.
If some components of $T(z)$ are nonlinear in $z$ (e.g., $z^2$ in the Gaussian), the LLR becomes nonlinear in $z$ unless the corresponding natural parameters are shared across classes.
The following subsections specialize this result to four distributional families, identifying the sufficient statistics, natural parameters, and LLR for each.

\subsection{Exponential Specialization}

The exponential model operates on the loss statistic $z = \ell$, where $\ell = -\log p$ is the cross-entropy loss and $p$ is the model's predicted confidence for the ground truth label (see Table~\ref{tab:expfam}).

For $\ell \mid m \sim \Exp(\lambda_m)$ with rate $\lambda_m > 0$, the sufficient statistic is $T(\ell) = \ell$, the natural parameter is $\eta_m = -\lambda_m$, and the log-partition function is $A(\eta_m) = -\log\lambda_m$.
Substituting into~\eqref{eq:expfam_llr}:
\[
\LLR(\ell) = -(\lambda_1 - \lambda_0)\ell - (-\log\lambda_1 + \log\lambda_0) = \log\frac{\lambda_1}{\lambda_0} - (\lambda_1 - \lambda_0)\ell.
\]
The LLR is affine in~$\ell$ with slope $\lambda_0 - \lambda_1$.
Since the exponential family has a single sufficient statistic ($\ell$ itself), the LLR is automatically affine---no parameter constraints are needed to obtain a linear score.

\subsection{Gamma Specialization}

Like the exponential, the Gamma model operates on the loss statistic $z = \ell$.

For $\ell \mid m \sim \text{Gamma}(\kappa_m, \vartheta_m)$ with shape $\kappa_m > 0$ and scale $\vartheta_m > 0$, the sufficient statistics are $T(\ell) = (\ell,\, \log\ell)$, the natural parameters are $\eta_m = (-1/\vartheta_m,\; \kappa_m - 1)$, and the log-partition function is $A(\eta_m) = \log\Gamma(\kappa_m) + \kappa_m\log\vartheta_m$.
Substituting into~\eqref{eq:expfam_llr} gives the general Gamma LLR:
\[
\LLR(\ell) = \left(\frac{1}{\vartheta_0} - \frac{1}{\vartheta_1}\right)\ell + (\kappa_1 - \kappa_0)\log\ell - \log\frac{\Gamma(\kappa_1)\,\vartheta_1^{\kappa_1}}{\Gamma(\kappa_0)\,\vartheta_0^{\kappa_0}}.
\]
The LLR has two variable terms: one linear in $\ell$ (driven by the scale difference $1/\vartheta_0 - 1/\vartheta_1$) and one linear in $\log\ell$ (driven by the shape difference $\kappa_1 - \kappa_0$).
The constant offset depends on the normalizing constants of both class-conditional distributions.

Under shared scale $\vartheta_0 = \vartheta_1 = \vartheta$, the linear-in-$\ell$ term vanishes and the LLR simplifies to
\[
\LLR(\ell) = (\kappa_1 - \kappa_0)\log\ell - \log\frac{\Gamma(\kappa_1)\,\vartheta^{\kappa_1}}{\Gamma(\kappa_0)\,\vartheta^{\kappa_0}},
\]
which is affine in $\log\ell$ with slope $\kappa_1 - \kappa_0$.
Since $\ell = -\log p$, we have $\log\ell = \log(-\log p)$, so the shared-scale Gamma LLR operates on the log-loss rather than on the loss or confidence directly.

Conversely, under shared shape $\kappa_0 = \kappa_1 = \kappa$, the $\log\ell$ term vanishes and the LLR simplifies to
\[
\LLR(\ell) = \left(\frac{1}{\vartheta_0} - \frac{1}{\vartheta_1}\right)\ell - \kappa\log\frac{\vartheta_1}{\vartheta_0},
\]
which is affine in $\ell$ with slope $1/\vartheta_0 - 1/\vartheta_1$.
This shared-shape case is structurally identical to the exponential LLR: both are linear in the loss itself.

These two specializations illustrate a general pattern.
The shared-scale case (affine in $\log\ell$) retains only the shape difference between classes, while the shared-shape case (affine in $\ell$) retains only the scale difference and mirrors the exponential LLR.

\subsection{Beta Specialization}

The Beta model operates on the confidence statistic $z = p$, where $p$ is the model's predicted probability for the ground truth label.
Under cross-entropy loss, $p = e^{-\ell}$, so confidence and loss are related by a monotone transform (see Table~\ref{tab:expfam}).

For $p \mid m \sim \text{Beta}(\alpha_m, \beta_m)$ with shape parameters $\alpha_m, \beta_m > 0$, the sufficient statistics are $T(p) = (\log p,\, \log(1-p))$, the natural parameters are $\eta_m = (\alpha_m - 1,\; \beta_m - 1)$, and the log-partition function is $A(\eta_m) = \log B(\alpha_m, \beta_m)$, where $B$ is the Beta function.
Substituting into~\eqref{eq:expfam_llr}:
\[
\LLR(p) = (\alpha_1 - \alpha_0)\log p + (\beta_1 - \beta_0)\log(1 - p) - \log\frac{B(\alpha_1, \beta_1)}{B(\alpha_0, \beta_0)}.
\]
The LLR is linear in $(\log p, \log(1-p))$, providing a natural two-parameter extension of the exponential model.
With two shape parameters per class (four total), the Beta family can capture asymmetric behavior near $p = 0$ and $p = 1$ independently.

\subsection{Gaussian Specialization}

The Gaussian model operates on the rescaled logit statistic $z = \varphi$, where $\varphi = \log(p/(1-p))$ is the log-odds of the ground truth label confidence~$p$.
The rescaled logit maps $p \in (0,1)$ to $\varphi \in (-\infty, +\infty)$, making the Gaussian a natural model for this statistic (see Table~\ref{tab:expfam}).
This is the distributional model underlying LiRA~\citep{carlini2022membership}.

For $\varphi \mid m \sim \Normal(\mu_m, \sigma_m^2)$, the sufficient statistics are $T(\varphi) = (\varphi,\, \varphi^2)$, the natural parameters are $\eta_m = \bigl(\mu_m/\sigma_m^2,\; -1/(2\sigma_m^2)\bigr)$, and the log-partition function is $A(\eta_m) = \mu_m^2/(2\sigma_m^2) + \log\sigma_m$.
Substituting into~\eqref{eq:expfam_llr} and simplifying gives the full Gaussian LLR:
\[
\LLR(\varphi) = \frac{(\varphi - \mu_0)^2}{2\sigma_0^2} - \frac{(\varphi - \mu_1)^2}{2\sigma_1^2} + \log\frac{\sigma_0}{\sigma_1}.
\]
This is the full Gaussian LLR~\eqref{eq:gauss_llr}, written out here for reference.
The LLR is quadratic in $\varphi$ whenever $\sigma_0 \neq \sigma_1$, because the $\varphi^2$ coefficient $\frac{1}{2\sigma_0^2} - \frac{1}{2\sigma_1^2}$ is nonzero.
This nonlinearity allows the Gaussian model to capture differences in the spread of shadow statistics between IN and OUT, not just shifts in location.

Under equal variances $\sigma_0^2 = \sigma_1^2 = \sigma^2$, the quadratic terms cancel and the LLR simplifies to
\[
\LLR(\varphi) = \frac{\mu_1 - \mu_0}{\sigma^2}\left(\varphi - \frac{\mu_1 + \mu_0}{2}\right),
\]
restating~\eqref{eq:gauss_equal_var}.
This is linear in $\varphi$: the slope is proportional to the mean gap $\mu_1 - \mu_0$ and inversely proportional to the variance, while the intercept is set by the midpoint of the two class means.

\subsection{Parameter Estimation}\label{app:param_est}

When using these distributional families for membership inference, we estimate class-conditional parameters from shadow-model statistics.
For a given audit point~$i$, the $K$ shadow models produce statistics $z_{i,1}, \ldots, z_{i,K}$, each labeled by membership $m_{i,k} \in \{0, 1\}$, where $m = 1$ denotes \emph{IN} (point~$i$ was in the training set of shadow model~$k$) and $m = 0$ denotes \emph{OUT}.
Let $\mathcal{K}_{i,m} = \{k : m_{i,k} = m\}$ be the index set for each class.
Parameters are estimated separately per class and per audit point.

\paragraph{Exponential.}
The exponential model has one parameter per class: the rate $\lambda_{i,m}$.
The MLE is the reciprocal of the sample mean loss:
\[
\hat\lambda_{i,m} = \frac{1}{\bar\ell_{i,m}}, \qquad \text{where} \quad \bar\ell_{i,m} = \frac{1}{|\mathcal{K}_{i,m}|} \sum_{k \in \mathcal{K}_{i,m}} \ell_{i,k}.
\]

\paragraph{Gamma.}
The Gamma model has two parameters per class: shape $\kappa_{i,m}$ and scale $\vartheta_{i,m}$.
No closed-form MLE exists.
The profile log-likelihood yields the implicit equation
\[
\log\hat\kappa_{i,m} - \psi(\hat\kappa_{i,m}) = \log\bar\ell_{i,m} - \overline{\log\ell}_{i,m},
\]
where $\psi$ is the digamma function and
\[
\overline{\log\ell}_{i,m} = \frac{1}{|\mathcal{K}_{i,m}|}\sum_{k\in\mathcal{K}_{i,m}}\log\ell_{i,k}
\]
is the sample mean of the log-losses.
This equation is solved numerically (e.g.\ Newton--Raphson), after which the scale estimate follows as $\hat\vartheta_{i,m} = \bar\ell_{i,m}/\hat\kappa_{i,m}$.

For initialization, method-of-moments estimates are convenient:
\[
\hat\kappa_{i,m}^{\text{MoM}} = \frac{\bar\ell_{i,m}^2}{s_{i,m}^2}, \qquad
\hat\vartheta_{i,m}^{\text{MoM}} = \frac{s_{i,m}^2}{\bar\ell_{i,m}},
\]
where $s_{i,m}^2 = |\mathcal{K}_{i,m}|^{-1}\sum_{k\in\mathcal{K}_{i,m}}(\ell_{i,k} - \bar\ell_{i,m})^2$ is the sample variance.

\paragraph{Beta.}
The Beta model has two parameters per class: $\alpha_{i,m}$ and $\beta_{i,m}$.
The MLE requires solving the coupled digamma equations
\begin{align*}
\psi(\hat\alpha_{i,m}) - \psi(\hat\alpha_{i,m} + \hat\beta_{i,m}) &= \overline{\log p}_{i,m}, \\
\psi(\hat\beta_{i,m}) - \psi(\hat\alpha_{i,m} + \hat\beta_{i,m}) &= \overline{\log(1-p)}_{i,m},
\end{align*}
where
\[
\overline{\log p}_{i,m} = \frac{1}{|\mathcal{K}_{i,m}|}\sum_{k\in\mathcal{K}_{i,m}} \log p_{i,k}, \qquad
\overline{\log(1{-}p)}_{i,m} = \frac{1}{|\mathcal{K}_{i,m}|}\sum_{k\in\mathcal{K}_{i,m}} \log(1-p_{i,k}).
\]
These are solved numerically, using probability-weighted moment (PWM) estimators for initialization.

\paragraph{Gaussian.}
The Gaussian model has two parameters per class: mean $\mu_{i,m}$ and variance $\sigma_{i,m}^2$.
The MLEs are the sample mean and variance:
\[
\hat\mu_{i,m} = \frac{1}{|\mathcal{K}_{i,m}|} \sum_{k \in \mathcal{K}_{i,m}} z_{i,k}, \qquad
\hat\sigma_{i,m}^2 = \frac{1}{|\mathcal{K}_{i,m}|} \sum_{k \in \mathcal{K}_{i,m}} (z_{i,k} - \hat\mu_{i,m})^2.
\]
The BASE hierarchy (Appendix~\ref{app:base_formulas}) introduces parameter-sharing constraints that reduce the number of free parameters.
BASE3 pools the variance across the two classes while estimating means separately.
BASE2 keeps the separately estimated means but replaces the per-point variance with a single global constant, isolating class-mean separation.
BASE1 collapses to a single pooled centering estimate, discarding all per-class variance information.

\paragraph{Robust estimation.}
Standard MLE estimates can be sensitive to outlier shadow statistics---for instance, a single anomalously low loss can heavily influence the exponential rate estimate.
Several robust alternatives exist:
\begin{itemize}
\item \emph{Trimmed means:} discard the most extreme fraction of shadow statistics before computing the sample mean.
\item \emph{Winsorized variance:} clamp extreme values to a quantile threshold before computing variance, reducing the influence of heavy tails.
\item \emph{Median-based estimators:} replace the sample mean with the median, which has a bounded influence function.
\item \emph{Interquartile range (IQR):} use the IQR as a robust spread estimate in place of the standard deviation, avoiding sensitivity to extreme tails.
\item \emph{Log-sum-exp averaging:} estimate the pooled center via $\hat{z}_{i,\text{pool}} = -\log\bigl(\frac{1}{K}\sum_{k=1}^K e^{-\ell_{i,k}}\bigr)$ instead of the arithmetic mean.  This is equivalent to the softmin and provides a natural robust centering in the confidence domain.
\end{itemize}
We explored trimmed means, winsorized variance, and median-based alternatives in our experiments, but did not find that they improved MIA performance over standard MLE.
The exception is log-sum-exp averaging for BASE1, which provides a modest improvement and is used in the standard BASE implementation~\citep{lassila2025base}.

% ------------------------------------------------------------------
\section{BASE Hierarchy: Derivations and Formulas}\label{app:base_formulas}

We now fix a target point $i$ and write $\LLR_i$ for its per-datapoint membership log-likelihood ratio.
The $K$ shadow statistics $z_{i,1},\ldots,z_{i,K}$ are split by membership label $m_{i,k} \in \{0,1\}$.
Each BASE variant follows from the exponential-family LLR (Appendix~\ref{app:derivations}) under specific parameter-sharing constraints.
We present the hierarchy top-down, from the most flexible (BASE4) to the most constrained (BASE1).

\emph{Sign convention.}
Since the IN model has been trained on the target point while the OUT model has not, we generally expect higher confidence for IN:
$p_{\text{IN}} > p_{\text{OUT}}$, equivalently $-\ell_{\text{IN}} > -\ell_{\text{OUT}}$ and $\varphi_{\text{IN}} > \varphi_{\text{OUT}}$.
Throughout the derivations below we therefore choose the scalar statistic $z \in \{p,\, -\ell,\, \varphi\}$ so that the IN--OUT mean difference $\Delta\mu \coloneqq \mu_{i,1} - \mu_{i,0}$ is positive.
This convention ensures that the LLR slope is positive and membership scores increase with the target statistic.

% ---- Gaussian derivation chain (paragraphs, not corollaries) ----

\paragraph{BASE4 (full Gaussian, 4 per-point parameters).\label{cor:base4}}
Assume $z_{i,k} \mid m \sim \Normal(\mu_{i,m},\, \sigma_{i,m}^2)$ with all four parameters $(\mu_{i,0}, \sigma_{i,0}^2, \mu_{i,1}, \sigma_{i,1}^2)$ free.
Substituting maximum-likelihood estimates into the Gaussian LLR~\eqref{eq:gauss_llr} gives the BASE4 score~\eqref{eq:base4}:
\[
\text{BASE4}_i = \frac{(z_{i,0} - \hat\mu_{i,0})^2}{2\hat\sigma_{i,0}^2} - \frac{(z_{i,0} - \hat\mu_{i,1})^2}{2\hat\sigma_{i,1}^2} + \log \frac{\hat\sigma_{i,0}}{\hat\sigma_{i,1}}.
\]

\paragraph{BASE3 (equal variance, 3 per-point parameters).}\label{cor:base3}
Constrain $\sigma_{i,0}^2 = \sigma_{i,1}^2 = \sigma_i^2$.
The equal-variance Gaussian LLR~\eqref{eq:gauss_equal_var} is linear in~$z$:
\[
\LLR_i(z) = \frac{\mu_{i,1} - \mu_{i,0}}{\sigma_i^2}\!\left(z_{i,0} - \frac{\mu_{i,1} + \mu_{i,0}}{2}\right).
\]
Substituting MLE estimates yields the BASE3 score~\eqref{eq:base3}:
\[
\text{BASE3}_i = \frac{\hat\mu_{i,1} - \hat\mu_{i,0}}{\hat\sigma_i^2}\!\left(z_{i,0} - \frac{\hat\mu_{i,1} + \hat\mu_{i,0}}{2}\right),
\]
with the estimators
\[
\hat\mu_{i,m} = \frac{1}{|\mathcal{K}_{i,m}|}\sum_{k \in \mathcal{K}_{i,m}} z_{i,k}, \qquad
\hat\sigma_i^2 = \frac{1}{|\mathcal{K}_{i,0}| + |\mathcal{K}_{i,1}|}\sum_{m \in \{0,1\}}\;\sum_{k \in \mathcal{K}_{i,m}}(z_{i,k} - \hat\mu_{i,m})^2.
\]

\paragraph{BASE2 (separate means, constant variance; 2 per-point parameters).}\label{cor:base2}
We obtain BASE2 by replacing BASE3's per-point pooled variance $\hat\sigma_i^2$ with a single global constant $\sigma^2$ shared across all data points, while retaining the separately estimated IN/OUT means.
The equal-variance Gaussian LLR~\eqref{eq:gauss_equal_var} then has a slope denominator that is constant across points; since the global $\sigma^2 > 0$ does not affect the ranking, BASE2 reduces to
\[
\text{BASE2}_i \propto (\hat\mu_{i,1} - \hat\mu_{i,0})\!\left(z_{i,0} - \frac{\hat\mu_{i,1} + \hat\mu_{i,0}}{2}\right),
\]
with separately estimated means
\[
\hat\mu_{i,m} = \frac{1}{|\mathcal{K}_{i,m}|}\sum_{k \in \mathcal{K}_{i,m}} z_{i,k}, \qquad m \in \{0,1\}.
\]
This is BASE3 with the per-point variance removed; it isolates the effect of class-mean separation and coincides with IMIA's quadratic scoring variant \citep{imia2025}, whose quadratic terms cancel under the shared IN/OUT variance to leave this linear score.

\paragraph{BASE1 (mean--variance proportionality, 1 per-point parameter).}\label{cor:base1_gauss}
Finally, suppose the means satisfy $\mu_{i,m} = a_m\,\sigma_i^2$ with constants $a_0, a_1$ independent of~$i$.
Then $\mu_{i,1} - \mu_{i,0} = (a_1 - a_0)\,\sigma_i^2$, so the equal-variance LLR slope simplifies to the constant $a_1 - a_0$:
\[
\LLR_i(z) = (a_1 - a_0)\!\left(z_{i,0} - \bar\mu_i\right),
\]
where $\bar\mu_i = (\mu_{i,1} + \mu_{i,0})/2$.
Since $a_1 - a_0 > 0$ is constant across points, the ranking depends only on $z_{i,0} - \bar\mu_i$, yielding the BASE1 score~\eqref{eq:base1}:
\[
\text{BASE1}_i = z_{i,0} - \hat{z}_{i,\text{pool}},
\]
where $\hat{z}_{i,\text{pool}}$ estimates the centering $\bar\mu_i$ from all $K$ shadow statistics for point~$i$.

% ---- Corollary 1: LiRA = BASE4 ----

\begin{corollary}[LiRA $=$ BASE4]\label{cor:lira_base4}
Setting $z = \varphi$ (rescaled logits) in the BASE4 formula above and substituting MLE parameters $(\hat\mu_{i,m}, \hat\sigma_{i,m}^2)$ into~\eqref{eq:gauss_llr} recovers the LiRA score~\eqref{eq:lira}.
\end{corollary}
\begin{proof}
Direct substitution of MLE $(\hat\mu_{i,m}, \hat\sigma_{i,m}^2)$ into~\eqref{eq:gauss_llr} gives~\eqref{eq:base4}, which equals~\eqref{eq:lira}.
\end{proof}

% ---- Corollary 2: BASE [Lassila] = BASE1 ----

\begin{corollary}[BASE $=$ BASE1]\label{cor:base_equiv}
Setting $z = -\ell$ in the BASE1 score and estimating the centering via the log-sum-exp average gives the BASE score of \citet{lassila2025base}:
\[
\text{BASE}_i = -\ell_{i,0} - \log\!\left(\frac{1}{K}\sum_{k=1}^K e^{-\ell_{i,k}}\right).
\]
\end{corollary}
\begin{proof}
Substitute $z = -\ell$ into the BASE1 score $z_{i,0} - \hat{z}_{i,\text{pool}}$.
For the centering estimate, replace the arithmetic loss-domain average with its Jensen-stable confidence-domain counterpart $\hat{z}_{i,\text{pool}} = \log(K^{-1}\sum_k e^{-\ell_{i,k}})$, giving the standard BASE formula.
\end{proof}

% ---- Independent derivation: BASE1 from the Exponential model ----

\paragraph{Independent derivation: BASE from the exponential model.}\label{cor:base1_exp}
Assume $\ell_{i,k} \mid m \sim \Exp(\lambda_{i,m})$ with a constant rate gap $\lambda_{i,1} - \lambda_{i,0} = \Delta\lambda > 0$ across data points.
The exponential LLR~\eqref{eq:exp_llr} for point~$i$ is
\[
\LLR_i(\ell) = \log\frac{\lambda_{i,1}}{\lambda_{i,0}} - (\lambda_{i,1} - \lambda_{i,0})\,\ell_{i,0},
\]
which is exactly affine in the target loss.
Under the constant-gap assumption, the slope $-\Delta\lambda$ is common to all points.
Introduce the pooled rate $\lambda_{i,\text{pool}} = (\lambda_{i,1} + \lambda_{i,0})/2$ and rewrite the intercept:
\[
\log\frac{\lambda_{i,1}}{\lambda_{i,0}} = \log\frac{\lambda_{i,\text{pool}} + \Delta\lambda/2}{\lambda_{i,\text{pool}} - \Delta\lambda/2} \;\approx\; \frac{\Delta\lambda}{\lambda_{i,\text{pool}}},
\]
where the approximation holds for $\Delta\lambda / \lambda_{i,\text{pool}} \ll 1$.
Substituting yields the approximate LLR
\[
\LLR_i(\ell) \;\approx\; \Delta\lambda\!\left(\frac{1}{\lambda_{i,\text{pool}}} - \ell_{i,0}\right).
\]
Since $\Delta\lambda > 0$ is constant, the ranking depends only on $\lambda_{i,\text{pool}}^{-1} - \ell_{i,0}$.
The pooled rate inverse $\lambda_{i,\text{pool}}^{-1}$ is the expected loss under the pooled model, estimated by the sample mean $\hat\ell_{i,\text{pool}} = K^{-1}\sum_k \ell_{i,k}$ or the Jensen-stable log-sum-exp average.
The latter choice yields
\[
\text{BASE}_i = -\ell_{i,0} - \log\!\left(\frac{1}{K}\sum_{k=1}^K e^{-\ell_{i,k}}\right).
\]

% ---- Independent derivation: BASE1 from the Gamma model ----

\paragraph{Independent derivation: BASE from the Gamma model.}\label{cor:base1_gamma}
Consider a Gamma model on the \emph{confidence}: $p_{i,k} \mid m \sim \text{Gamma}(\kappa_{i,m},\, \vartheta_i)$ with a per-point scale $\vartheta_i > 0$ shared across membership classes.
Assume a constant shape gap $\kappa_{i,1} - \kappa_{i,0} = \Delta\kappa > 0$ across data points.
Under the shared-scale Gamma model, the density is $f(p \mid \kappa, \vartheta) = [\Gamma(\kappa)\,\vartheta^\kappa]^{-1}\,p^{\kappa-1}\,e^{-p/\vartheta}$, and the LLR for the target confidence $p_{i,0}$ is
\[
\LLR_i(p_{i,0}) = (\kappa_{i,1} - \kappa_{i,0})\log p_{i,0} - (\kappa_{i,1} - \kappa_{i,0})\log\vartheta_i - \log\frac{\Gamma(\kappa_{i,1})}{\Gamma(\kappa_{i,0})}.
\]
Since $\log p_{i,0} = -\ell_{i,0}$, this becomes
\[
\LLR_i = -\Delta\kappa\,\ell_{i,0} - \Delta\kappa\,\log\vartheta_i - \log\frac{\Gamma(\kappa_{i,1})}{\Gamma(\kappa_{i,0})}.
\]
The slope $-\Delta\kappa$ is constant across points, so the ranking depends on $-\ell_{i,0} - \log\vartheta_i$ (the log-Gamma ratio varies slowly for moderate shape variation).
The shared scale $\vartheta_i$ is estimated by the pooled mean of the shadow confidences:
\[
\hat\vartheta_i = \frac{1}{K}\sum_{k=1}^K p_{i,k} = \frac{1}{K}\sum_{k=1}^K e^{-\ell_{i,k}}.
\]
Substituting $\hat\vartheta_i$ and dropping the constant $\Delta\kappa$ gives
\[
\text{LLR}_i \;\propto\; -\ell_{i,0} - \log\!\left(\frac{1}{K}\sum_{k=1}^K e^{-\ell_{i,k}}\right),
\]
recovering the BASE score with log-sum-exp centering.

% ---- Unifying remark ----

\emph{Remark (unifying principle).}
All three BASE derivations share a common structure: the distributional LLR is affine in a scalar statistic with a slope that is constant across data points, so membership inference reduces to pooled centering.
More generally, in any one-parameter exponential family with sufficient statistic $T(z)$, the LLR is
\[
\LLR_i(z) = (\eta_{i,1} - \eta_{i,0})\,T(z_{i,0}) - \bigl(A(\eta_{i,1}) - A(\eta_{i,0})\bigr).
\]
If the parameter gap $\eta_{i,1} - \eta_{i,0} = \Delta\eta$ is constant across data points, the slope is constant and the score ranking reduces to $T(z_{i,0})$ minus a per-point centering term---exactly the BASE1 structure.
The three derivations above are instances with $T(\ell) = \ell$ (exponential), $T(z) = z$ (Gaussian with mean--variance proportionality), and $T(p) = \log p = -\ell$ (shared-scale Gamma on confidence).
In each case, centering in the confidence domain via $K^{-1}\sum e^{-\ell_{i,k}}$ yields the log-sum-exp average of the standard BASE implementation~\citep{lassila2025base}.

% ------------------------------------------------------------------
\section{BaVarIA Details}\label{app:bavaria}

\subsection{NIG Posterior Updates}

The NIG prior $(\mu, \sigma^2) \sim \NIG(\mu_{\varnothing,m}, \kappa_\varnothing, \alpha_\varnothing, \beta_{\varnothing,m})$ has density
\[
p(\mu, \sigma^2) = \frac{\sqrt{\kappa_\varnothing}}{\sigma\sqrt{2\pi}} \frac{\beta_{\varnothing,m}^{\alpha_\varnothing}}{\Gamma(\alpha_\varnothing)} \left(\frac{1}{\sigma^2}\right)^{\alpha_\varnothing + 1} \exp\!\left(-\frac{2\beta_{\varnothing,m} + \kappa_\varnothing(\mu - \mu_{\varnothing,m})^2}{2\sigma^2}\right).
\]

For a fixed data point $i$ and class $m$, after observing $n_m = |\mathcal{K}_{i,m}|$ shadow statistics with sample mean $\bar{z}_m$ and sum of squared deviations $S_m = \sum_{k \in \mathcal{K}_{i,m}}(z_{i,k} - \bar{z}_m)^2$, the posterior hyperparameters $(\mu'_{i,m}, \kappa'_{i,m}, \alpha'_{i,m}, \beta'_{i,m})$ are given by~\eqref{eq:nig_update1}.

\subsection{Empirical Bayes Hyperparameters}

The prior is estimated separately for each class $m \in \{0, 1\}$ via empirical Bayes, pooling shadow statistics across all points:
\begin{itemize}
\item $\mu_{\varnothing,m}$: global mean of the statistic across all class-$m$ shadow observations.
\item $\kappa_\varnothing$: controls the strength of the mean prior; set to 1 (weakly informative).
\item $\alpha_\varnothing$: controls the strength of the variance prior; set to 2 (weakly informative, ensuring $\E[\sigma^2]$ is finite).
\item $\beta_{\varnothing,m}$: set to $\sigma^{2}_{m,\text{global}} \cdot (\alpha_\varnothing - 1)$, so that $\E[\sigma^2] = \beta_{\varnothing,m} / (\alpha_\varnothing - 1)$ matches the observed class-$m$ global variance. With $\alpha_\varnothing = 2$, this gives $\beta_{\varnothing,m} = \sigma^{2}_{m,\text{global}}$.
\end{itemize}
The strength parameters $\kappa_\varnothing$ and $\alpha_\varnothing$ are shared across classes.
These defaults are used throughout all experiments. No per-dataset tuning is performed.

\subsection{Convergence to LiRA}

As the number of shadow observations $n_m = |\mathcal{K}_{i,m}|$ grows, the NIG posterior concentrates around the maximum likelihood estimates.
Specifically, from the update equations~\eqref{eq:nig_update1}
\begin{itemize}
\item The posterior mean converges to the sample mean: $\mu'_{i,m} = \frac{\kappa_\varnothing \mu_{\varnothing,m} + n_m \bar{z}_m}{\kappa_\varnothing + n_m} \to \bar{z}_m$ as $n_m \to \infty$, since the prior contribution $\kappa_\varnothing \mu_{\varnothing,m}$ becomes negligible.
\item The posterior variance converges to the MLE variance: $\beta'_{i,m}/\alpha'_{i,m} \to S_m/n_m = \hat\sigma_{m}^2$, so the Bayesian shrinkage vanishes.
\item The degrees of freedom $\nu_{i,m} = 2\alpha'_{i,m} = 2\alpha_\varnothing + n_m \to \infty$, so the Student-$t$ predictive~\eqref{eq:student_t} converges to a Gaussian.
\end{itemize}
Consequently, BaVarIA-$n$ reduces to LiRA (same Gaussian LLR with MLE parameters), and BaVarIA-$t$ also reduces to LiRA since the Student-$t$ converges to the Gaussian.
The prior becomes non-informative, and all three methods produce identical scores.

% ------------------------------------------------------------------
\section{Full Experimental Results}\label{app:full_results}

All experiments use Design~B resampling \citep{carlini2022membership} with 32 replicates across 12 testbeds and 7 shadow-model budgets ($K \in \{4, 8, 16, 32, 64, 128, 254\}$).
The full per-testbed scaling and ROC curves are collected in Appendix~\ref{app:curves}: BASE-hierarchy scaling (Figures~\ref{fig:online_base_auc}--\ref{fig:online_base_tpr_rn}), online-vs-offline BaVarIA curves (Figures~\ref{fig:offline_auc}--\ref{fig:offline_tpr_rn}), and ROC curves  (Figures~\ref{fig:roc_base_k8}--\ref{fig:roc_bav_offline_k64}).
Tables below show one representative architecture per data source (WideResNet for image, MLP4 for tabular); the second variant follows the same trends.

\subsection{Verification of Equivalences}

\paragraph{BASE1 = RMIA.}
BASE1 on loss with log-sum-exp averaging reproduces RMIA at $\gamma = 1$ exactly ($\max|\Delta\text{AUC}| = 0$ across all $K$ and datasets), confirming the equivalence of \citet{lassila2025base}.

\paragraph{BASE4 = LiRA.}
BASE4 on rescaled logits with MLE parameters (denominator $n$, i.e., the biased variance estimator) reproduces LiRA exactly at $K \geq 64$ ($\max|\Delta\text{AUC}| < 10^{-6}$). At $K < 64$, LiRA's hard switch substitutes a global variance for the per-point MLE, breaking the equivalence ($\max|\Delta\text{AUC}| = 0.039$ at $K = 8$).

\subsection{BASE Hierarchy at Multiple Shadow Budgets}\label{app:base_tables}

Tables~\ref{tab:base32_tpr001_K8}, \ref{tab:base32_tpr001_K64}, and~\ref{tab:base32_tpr001_K254} report TPR@0.01 for the BASE hierarchy (BASE1--BASE4) across all 12 testbeds at $K \in \{8, 64, 254\}$; LiRA $=$ BASE4 and RMIA $=$ BASE1 by construction (LiRA additionally applies a global-variance safeguard at small $K$, discussed below). The tables show the ordering reversal used throughout: at $K = 8$ the simpler families (BASE1, BASE2) lead in the low-FPR tail while raw per-class MLE (BASE4) collapses (e.g.\ $0.027$ on CIFAR-10/WRN, estimated from ${\sim}4$ observations per class), and by $K = 254$ the richer families lead as per-class variances converge.
\paragraph{Weighted concordance score.} For a given testbed and metric, let $m_a$ be the metric value of BASE method $a \in \{1,2,3,4\}$. Over the $\binom{4}{2}=6$ method pairs we set
\begin{equation}
S = \frac{\sum_{a<b} s_{ab}\,|m_a - m_b|}{\sum_{a<b} |m_a - m_b|} \in [-1,+1],
\end{equation}
where $s_{ab}=+1$ if the observed order of the pair matches the hierarchy's prediction---%
at small $K$ (\emph{simple wins}), the reverse at large $K$ (\emph{complex wins})---and $s_{ab}=-1$ otherwise. Weighting each pair by its gap $|m_a-m_b|$ keeps near-ties from dominating; $S=+1$ when every pair agrees with the prediction and $S=-1$ under a full reversal.
Table~\ref{tab:summary32_concordance} reports $S$ for each testbed and budget. We attach a percentile bootstrap 95\% confidence interval (CI) to each $S$ by resampling the 32 replicates and recomputing $S$ on the resampled means; a $\star$ marks cells whose CI excludes~0. The tail metrics (TPR@0.01, TPR@0.001) are significantly supported on all 12 testbeds at both budgets; only the AUC columns are mixed---four unsupported cells at $K=8$ (the noisy small-budget regime) and one at $K=254$ (CIFAR-10/RN18).

At very low FPR (FPR $\approx 0.001$) the ordering within the simpler pair is less stable at $K=8$: BASE2 holds up better than the fully pooled BASE1 on about half of the testbeds shown in Figure~\ref{fig:roc_base_k8} (most clearly Texas/MLP3, whose TPR@0.001 concordance $S=+0.65$ is its weakest cell), while BASE1 remains ahead on the rest. A plausible reading is that the class-mean separation BASE2 retains, and BASE1 discards, matters more at very low FPR. Differences are small and do not affect our conclusions about the hierarchy.

\paragraph{LiRA's hard-switch discontinuity.}
Conventional LiRA implementations switch from a single global variance to per-point variance estimation at $K = 64$ (when at least 32 observations are available per class). This binary switch produces a visible kink in LiRA's scaling curve---the step gain from $K = 32$ to $K = 64$ is anomalously large relative to the otherwise monotonically decreasing marginal returns at adjacent $K$ (Figures~\ref{fig:online_base_tpr}--\ref{fig:online_base_tpr_rn}), most pronounced for TPR@0.01 where the variance term in the Gaussian LLR most influences the tail. BaVarIA-$n$'s advantage over LiRA correspondingly peaks at $K = 32$ (where LiRA is still constrained to global variance) and drops sharply at $K = 64$; BaVarIA avoids the discontinuity entirely, as the NIG posterior smoothly shifts weight from the global prior to per-point data as $K$ grows (Figure~\ref{fig:ablation}).

\begin{table}%[h]
\centering
\footnotesize
\begin{tabular}{lcccc}
\toprule
Testbed & BASE1 & BASE2 & BASE3 & BASE4 \\
\midrule
CIFAR-10 / WRN & \textbf{0.098} & 0.086 & 0.072 & 0.027 \\
CIFAR-100 / WRN & \textbf{0.294} & 0.283 & 0.223 & 0.086 \\
CINIC-10 / WRN & 0.111 & \textbf{0.156} & 0.109 & 0.041 \\
Location / MLP3 & 0.225 & \textbf{0.297} & 0.215 & 0.110 \\
Purchase / MLP3 & \textbf{0.028} & 0.024 & 0.019 & 0.011 \\
Texas / MLP3 & \textbf{0.202} & 0.191 & 0.177 & 0.065 \\
CIFAR-10 / RN18 & \textbf{0.039} & 0.026 & 0.022 & 0.014 \\
CIFAR-100 / RN18 & 0.351 & \textbf{0.566} & 0.463 & 0.254 \\
CINIC-10 / RN18 & 0.133 & \textbf{0.215} & 0.161 & 0.062 \\
Location / MLP4 & 0.139 & \textbf{0.217} & 0.211 & 0.083 \\
Purchase / MLP4 & 0.036 & \textbf{0.039} & 0.028 & 0.013 \\
Texas / MLP4 & \textbf{0.233} & 0.231 & 0.216 & 0.084 \\
\bottomrule
\end{tabular}
\caption{TPR @ FPR$=0.01$, $K=8$, redefined BASE hierarchy. Mean over 32 replicates. Best per row in bold. LiRA $=$ BASE4, RMIA $=$ BASE1.}
\label{tab:base32_tpr001_K8}
\end{table}

\begin{table}%[h]
\centering
\footnotesize
\begin{tabular}{lcccc}
\toprule
Testbed & BASE1 & BASE2 & BASE3 & BASE4 \\
\midrule
CIFAR-10 / WRN & 0.111 & 0.126 & \textbf{0.136} & 0.128 \\
CIFAR-100 / WRN & 0.326 & 0.345 & \textbf{0.366} & 0.362 \\
CINIC-10 / WRN & 0.129 & 0.180 & \textbf{0.182} & 0.179 \\
Location / MLP3 & 0.274 & 0.378 & 0.398 & \textbf{0.434} \\
Purchase / MLP3 & 0.037 & 0.037 & \textbf{0.040} & 0.031 \\
Texas / MLP3 & 0.253 & 0.232 & \textbf{0.306} & 0.305 \\
CIFAR-10 / RN18 & \textbf{0.047} & 0.043 & 0.046 & 0.039 \\
CIFAR-100 / RN18 & 0.425 & 0.636 & 0.645 & \textbf{0.664} \\
CINIC-10 / RN18 & 0.175 & 0.272 & 0.277 & \textbf{0.287} \\
Location / MLP4 & 0.195 & 0.280 & \textbf{0.364} & 0.351 \\
Purchase / MLP4 & 0.048 & 0.059 & \textbf{0.063} & 0.053 \\
Texas / MLP4 & 0.284 & 0.279 & 0.355 & \textbf{0.357} \\
\bottomrule
\end{tabular}
\caption{TPR @ FPR$=0.01$, $K=64$, redefined BASE hierarchy. Mean over 32 replicates. Best per row in bold. LiRA $=$ BASE4, RMIA $=$ BASE1.}
\label{tab:base32_tpr001_K64}
\end{table}

\begin{table}%[h]
\centering
\footnotesize
\begin{tabular}{lcccc}
\toprule
Testbed & BASE1 & BASE2 & BASE3 & BASE4 \\
\midrule
CIFAR-10 / WRN & 0.113 & 0.132 & 0.142 & \textbf{0.143} \\
CIFAR-100 / WRN & 0.330 & 0.354 & 0.383 & \textbf{0.385} \\
CINIC-10 / WRN & 0.131 & 0.183 & 0.189 & \textbf{0.192} \\
Location / MLP3 & 0.286 & 0.387 & 0.421 & \textbf{0.452} \\
Purchase / MLP3 & 0.038 & 0.040 & \textbf{0.045} & 0.043 \\
Texas / MLP3 & 0.261 & 0.237 & 0.320 & \textbf{0.326} \\
CIFAR-10 / RN18 & 0.048 & 0.047 & \textbf{0.051} & 0.050 \\
CIFAR-100 / RN18 & 0.434 & 0.645 & 0.660 & \textbf{0.682} \\
CINIC-10 / RN18 & 0.181 & 0.280 & 0.290 & \textbf{0.304} \\
Location / MLP4 & 0.206 & 0.286 & \textbf{0.383} & 0.381 \\
Purchase / MLP4 & 0.050 & 0.062 & \textbf{0.069} & 0.066 \\
Texas / MLP4 & 0.290 & 0.284 & 0.371 & \textbf{0.376} \\
\bottomrule
\end{tabular}
\caption{TPR @ FPR$=0.01$, $K=254$, redefined BASE hierarchy. Mean over 32 replicates. Best per row in bold. LiRA $=$ BASE4, RMIA $=$ BASE1.}
\label{tab:base32_tpr001_K254}
\end{table}

\subsection{Distributional Diagnostics}\label{app:qq}

Figure~\ref{fig:qq_diagnostic} shows QQ plots of per-datapoint standardized shadow log-odds against the $\Normal(0,1)$ reference, and Table~\ref{tab:ad_normality} reports Anderson-Darling (AD) rejection rates at the $5\%$ level for all 12 testbeds.
Rejection rates range from near-nominal (${\sim}6\%$ on CIFAR-10 WRN) to over $80\%$ (Location MLP3 IN class), with the largest departures concentrated in the IN class---consistent with heavier tails visible in the QQ plots for CIFAR-100 ResNet ($74\%$ IN rejection).

Despite these departures, the Gaussian-family methods lead at moderate-to-large $K$ (Section~\ref{sec:experiments}).
The LLR score is evaluated at the target model's output, which typically lies in the bulk of the distribution where the Gaussian approximation remains accurate; tail departures affect the density estimate but not the ranking of most data points.

\begin{figure}
\centering
\includegraphics[width=\linewidth]{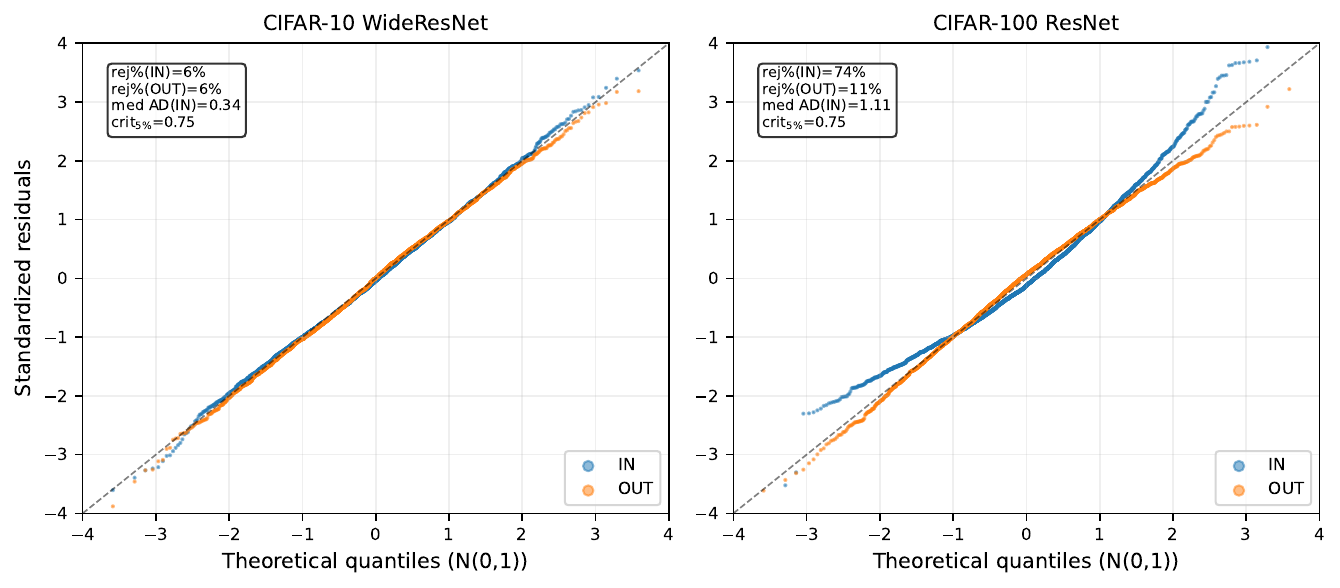}
\caption{Per-datapoint QQ diagnostic: standardized residuals of shadow log-odds (per data point, per class) pooled across datapoints and plotted against $\Normal(0,1)$ quantiles. Annotations show per-datapoint Anderson-Darling (AD) rejection rates. CIFAR-10 WRN is well-approximated by a Gaussian (${\sim}6\%$ AD rejection); CIFAR-100 ResNet shows heavier tails in the IN class ($74\%$ AD rejection).}\label{fig:qq_diagnostic}
\end{figure}

\begin{table}
\centering
\caption{Per-datapoint Anderson-Darling normality test on shadow log-odds ($K = 254$, $n = 1{,}000$ datapoints per class). Rejection rate at the $5\%$ significance level (critical value $0.748$). Gaussianity holds well for WideResNet on CIFAR-10/100 but is strongly violated for Location MLP3 and CINIC-10 WideResNet.}\label{tab:ad_normality}
\small
\begin{tabular}{l rr rr}
\toprule
 & \multicolumn{2}{c}{IN class} & \multicolumn{2}{c}{OUT class} \\
\cmidrule(lr){2-3} \cmidrule(lr){4-5}
Dataset & med AD & rej\% & med AD & rej\% \\
\midrule
C10-RN   & 0.475 & 22.4 & 0.438 & 19.4 \\
C10-WRN  & 0.342 &  5.9 & 0.355 &  6.5 \\
C100-RN  & 1.106 & 74.3 & 0.388 & 11.1 \\
C100-WRN & 0.357 &  6.8 & 0.369 &  8.2 \\
CN10-RN  & 0.607 & 37.6 & 0.420 & 18.2 \\
CN10-WRN & 0.934 & 57.3 & 1.463 & 67.2 \\
\midrule
Loc-3    & 1.966 & 85.0 & 0.402 & 12.9 \\
Loc-4    & 1.055 & 70.5 & 0.618 & 41.1 \\
Pur-3    & 0.403 & 15.1 & 0.456 & 20.4 \\
Pur-4    & 0.400 & 15.9 & 0.437 & 18.6 \\
Tex-3    & 0.406 & 17.2 & 0.396 & 14.4 \\
Tex-4    & 0.437 & 17.3 & 0.482 & 26.7 \\
\bottomrule
\end{tabular}
\end{table}

% ------------------------------------------------------------------
\section{Statistical Significance}\label{app:significance}

This appendix reports significance tests for the two comparative claims in the main body: the TPR @ FPR$=0.01$ advantage of BaVarIA-$n$ over LiRA at $K=64$ (Table~\ref{tab:bavaria32_tpr001_K64}) and the BASE-hierarchy concordance (Table~\ref{tab:summary32_concordance}; bootstrap CIs described in Appendix~\ref{app:base_tables}). All tests resample over the 32 replicates. For the BaVarIA-$n$ vs.\ LiRA comparison we form, per testbed and setting, the 32 same-replicate paired deltas $\Delta=$ BaVarIA-$n$ $-$ LiRA and apply a paired Wilcoxon signed-rank test together with a percentile bootstrap 95\% CI of the mean $\Delta$ ($10{,}000$ resamples). Online and offline are treated as separate families, with Holm--Bonferroni correction across the 12 testbeds within each. To ensure the pairing reflects the estimator and not the shadow draw, both methods are scored on one common shadow subsample per replicate (the ``shared-shadows'' design).

\begin{table*}[t]
\centering
\footnotesize
\begin{tabular}{l cc c cc}
\toprule
 & \multicolumn{2}{c}{Online} & & \multicolumn{2}{c}{Offline} \\
\cmidrule(lr){2-3} \cmidrule(lr){5-6}
Testbed & $\Delta$ (\%pt) & 95\% CI & & $\Delta$ (\%pt) & 95\% CI \\
\midrule
CIFAR-10 / WRN & +0.63$^{\star}$ & [+0.52, +0.73] & & -1.36$^{\dagger}$ & [-1.74, -0.97] \\
CIFAR-100 / WRN & +0.85$^{\star}$ & [+0.71, +0.99] & & +1.32$^{\star}$ & [+1.04, +1.60] \\
CINIC-10 / WRN & +0.71$^{\star}$ & [+0.65, +0.78] & & +5.81$^{\star}$ & [+5.54, +6.08] \\
Location / MLP3 & +0.72$^{\star}$ & [+0.40, +1.06] & & +7.68$^{\star}$ & [+6.56, +8.88] \\
Purchase / MLP3 & +0.44$^{\star}$ & [+0.42, +0.46] & & +0.25$^{\star}$ & [+0.17, +0.33] \\
Texas / MLP3 & +0.42$^{\star}$ & [+0.29, +0.55] & & +1.82$^{\star}$ & [+1.40, +2.26] \\
CIFAR-10 / RN18 & +0.25$^{\star}$ & [+0.19, +0.30] & & -0.65$^{\dagger}$ & [-0.82, -0.48] \\
CIFAR-100 / RN18 & +0.56$^{\star}$ & [+0.44, +0.67] & & +8.30$^{\star}$ & [+8.07, +8.53] \\
CINIC-10 / RN18 & +0.53$^{\star}$ & [+0.43, +0.66] & & +5.38$^{\star}$ & [+4.98, +5.95] \\
Location / MLP4 & +0.91$^{\star}$ & [+0.38, +1.46] & & +3.47$^{\star}$ & [+2.48, +4.48] \\
Purchase / MLP4 & +0.61$^{\star}$ & [+0.59, +0.65] & & +0.40$^{\star}$ & [+0.28, +0.53] \\
Texas / MLP4 & +0.28$^{\star}$ & [+0.16, +0.40] & & +2.47$^{\star}$ & [+1.91, +3.13] \\
\bottomrule
\end{tabular}
\caption{Paired advantage $\Delta=$ BaVarIA-$n$ $-$ LiRA in TPR @ FPR$=0.01$ at $K=64$ (percentage points), over the 32 replicates (shared-shadows: both methods on one common shadow subsample per replicate). 95\% CI is a percentile bootstrap of the mean $\Delta$. $^{\star}$/$^{\dagger}$: significantly better/worse at $95\%$ confidence (paired Wilcoxon signed-rank, Holm--Bonferroni within each family). BaVarIA-$n$ significantly beats LiRA on 12/12 online and 10/12 offline, and is significantly worse on 2/12 offline (both CIFAR-10 testbeds).}
\label{tab:bavaria32_significance}
\end{table*}

BaVarIA-$n$ significantly beats LiRA on all $12/12$ testbeds online and $10/12$ offline, and is significantly worse only on the two CIFAR-10 testbeds offline---where LiRA's offline score is already strong and per-class variance is hardest to estimate.
The margins are small online---between $+0.25$ and $+0.91$ percentage points---and substantially larger offline (up to $+8.3$ pp), where LiRA's score lacks the IN-class likelihood that the NIG prior supplies.

% ------------------------------------------------------------------
\section{Distribution Family Comparison}\label{app:distributions}

Table~\ref{tab:dist_auc} (AUC) and Table~\ref{tab:dist_comparison} (TPR@0.01) compare the four distributional families of the framework at $K = 64$ (online, mean over 32 replicates): Exponential and Gamma on the loss, Beta on the confidence, and Gaussian on the rescaled log-odds, with LiRA and RMIA as references.
The Gaussian on log-odds reproduces LiRA exactly---as they are identical by definition---and is the strongest family on every testbed.

On the loss statistic, the Gamma model's extra shape parameter gives no consistent benefit over the one-parameter Exponential: its AUC is lower on five of six testbeds, and its TPR@0.01 gains, where they occur, are marginal (Table~\ref{tab:dist_comparison}). The Beta model on confidence stays below the Gaussian everywhere; against the loss-based families it is mixed in AUC but ahead at TPR@0.01 on most testbeds. Differences between families are larger at TPR@0.01 than in AUC; at TPR@0.01 the Exponential is near the $0.01$ chance level on two testbeds ($0.018$ on CIFAR-10, $0.008$ on Purchase), while the pooled RMIA ($=$BASE1), though below the Gaussian throughout, is the second-strongest score on four of six. Because the families differ in both the sufficient statistic and the functional form of the likelihood ratio, these comparisons do not separate the two contributions. They do, however, support the paper's focus on the Gaussian branch: no alternative family improves on the Gaussian on log-odds on any testbed at either metric, and the cost of a misspecified family appears far larger in the low-FPR tail than in AUC.

\begin{table}
\centering
\caption{AUC by distributional family at $K = 64$, mean over 32 replicates; best per row in bold.}\label{tab:dist_auc}
\small
\begin{tabular}{lrrrrrr}
\toprule
Dataset & EXP & GAUSS & GAMMA & BETA & LiRA & RMIA \\
\midrule
C10-WRN  & 0.695 & \textbf{0.720} & 0.651 & 0.691 & \textbf{0.720} & 0.691 \\
C100-WRN & 0.883 & \textbf{0.908} & 0.897 & 0.903 & \textbf{0.908} & 0.882 \\
CN10-WRN & 0.748 & \textbf{0.756} & 0.738 & 0.746 & \textbf{0.756} & 0.708 \\
Loc-4    & 0.885 & \textbf{0.929} & 0.852 & 0.887 & \textbf{0.929} & 0.843 \\
Pur-4    & 0.617 & \textbf{0.649} & 0.566 & 0.597 & \textbf{0.649} & 0.642 \\
Tex-4    & 0.857 & \textbf{0.895} & 0.785 & 0.865 & \textbf{0.895} & 0.862 \\
\bottomrule
\end{tabular}
\end{table}

\begin{table}
\centering
\caption{TPR@0.01 by distributional family at $K = 64$, mean over 32 replicates; best per row in bold.}\label{tab:dist_comparison}
\small
\begin{tabular}{lrrrrrr}
\toprule
Dataset & EXP & GAUSS & GAMMA & BETA & LiRA & RMIA \\
\midrule
C10-WRN  & 0.018 & \textbf{0.129} & 0.020 & 0.018 & \textbf{0.129} & 0.112 \\
C100-WRN & 0.128 & \textbf{0.361} & 0.060 & 0.330 & \textbf{0.361} & 0.326 \\
CN10-WRN & 0.168 & \textbf{0.178} & 0.048 & 0.146 & \textbf{0.178} & 0.129 \\
Loc-4    & 0.041 & \textbf{0.355} & 0.054 & 0.112 & \textbf{0.355} & 0.189 \\
Pur-4    & 0.008 & \textbf{0.053} & 0.010 & 0.012 & \textbf{0.053} & 0.049 \\
Tex-4    & 0.060 & \textbf{0.357} & 0.075 & 0.151 & \textbf{0.357} & 0.281 \\
\bottomrule
\end{tabular}
\end{table}

% ------------------------------------------------------------------
\section{BaVarIA Per-Dataset Results}\label{app:bavaria_results}

Tables~\ref{tab:bavaria32_online_auc_K8} and~\ref{tab:bavaria32_online_tpr001_K8} report per-dataset AUC and TPR@0.01 at $K = 8$, the smallest budget we consider and the regime in which per-class variance estimation is least reliable; Table~\ref{tab:bavaria32_tpr001_K64} gives the corresponding tail results at $K = 64$.
Figures~\ref{fig:offline_auc}--\ref{fig:offline_tpr_rn} show the full scaling curves across all datasets (solid = online, dashed = offline).
We discuss the online results in this section and refer to Appendix~\ref{app:offline} for an analysis of the offline results.

At $K = 8$ (Table~\ref{tab:bavaria32_online_auc_K8}), BaVarIA-$t$ achieves the best AUC on 11 of 12 testbeds, losing only on CIFAR-100/WRN.
The margins are narrow---between $+0.003$ and $+0.010$ over LiRA, largest on Texas/MLP3 ($0.854$ vs.\ $0.844$)---so the advantage at this budget is one of consistency rather than magnitude.
BaVarIA-$n$ is essentially indistinguishable from LiRA on AUC at $K = 8$ (4 wins, 2 ties and 6 losses; mean $-0.001$, worst $-0.006$)---the AUC gain at this budget comes from the Student-$t$ predictive, whose heavier tails BaVarIA-$n$'s Gaussian LLR cannot exploit.

On TPR@0.01 at the same budget (Table~\ref{tab:bavaria32_online_tpr001_K8}) the ordering differs: LiRA is best on 7 of 12 testbeds (tied on one) and BaVarIA-$t$ on none, while BaVarIA-$n$ leads on the four Location and Texas testbeds and on CIFAR-100/RN18 (for example $0.371$ vs.\ $0.344$ on Location/MLP3).
The margins are small and asymmetric---LiRA's leads are all within $0.005$, BaVarIA-$n$'s reach $0.027$---so the count alone does not settle the comparison at this budget.
The tail gains reported in Section \ref{sec:bavariaablation} arrive at moderate budgets---at $K = 64$ BaVarIA-$n$ beats LiRA on all 12 testbeds (Table~\ref{tab:bavaria32_tpr001_K64})---consistent with the peak near the $K = 64$ cutover in Figure~\ref{fig:ablation}.

\begin{table}[h]
\centering
\footnotesize
\begin{tabular}{lccc}
\toprule
Testbed & LiRA & BaVarIA-$n$ & BaVarIA-$t$ \\
\midrule
CIFAR-10 / WRN & 0.686 & 0.680 & \textbf{0.689} \\
CIFAR-100 / WRN & \textbf{0.888} & 0.886 & 0.887 \\
CINIC-10 / WRN & 0.729 & 0.729 & \textbf{0.732} \\
Location / MLP3 & 0.933 & 0.936 & \textbf{0.941} \\
Purchase / MLP3 & 0.581 & 0.576 & \textbf{0.585} \\
Texas / MLP3 & 0.844 & 0.844 & \textbf{0.854} \\
CIFAR-10 / RN18 & 0.577 & 0.574 & \textbf{0.582} \\
CIFAR-100 / RN18 & 0.968 & 0.969 & \textbf{0.971} \\
CINIC-10 / RN18 & 0.857 & 0.858 & \textbf{0.862} \\
Location / MLP4 & 0.906 & 0.908 & \textbf{0.913} \\
Purchase / MLP4 & 0.618 & 0.613 & \textbf{0.623} \\
Texas / MLP4 & 0.868 & 0.867 & \textbf{0.875} \\
\bottomrule
\end{tabular}
\caption{Online AUC, $K=8$. Mean over 32 replicates. Best per row in bold.}
\label{tab:bavaria32_online_auc_K8}
\end{table}

\begin{table}[h]
\centering
\footnotesize
\begin{tabular}{lccc}
\toprule
Testbed & LiRA & BaVarIA-$n$ & BaVarIA-$t$ \\
\midrule
CIFAR-10 / WRN & \textbf{0.092} & 0.087 & 0.070 \\
CIFAR-100 / WRN & \textbf{0.294} & 0.293 & 0.261 \\
CINIC-10 / WRN & \textbf{0.158} & 0.154 & 0.148 \\
Location / MLP3 & 0.344 & \textbf{0.371} & 0.351 \\
Purchase / MLP3 & \textbf{0.025} & 0.021 & 0.021 \\
Texas / MLP3 & 0.204 & \textbf{0.226} & 0.212 \\
CIFAR-10 / RN18 & \textbf{0.026} & 0.023 & 0.022 \\
CIFAR-100 / RN18 & 0.587 & \textbf{0.600} & 0.583 \\
CINIC-10 / RN18 & \textbf{0.231} & 0.231 & 0.217 \\
Location / MLP4 & 0.223 & \textbf{0.250} & 0.224 \\
Purchase / MLP4 & \textbf{0.040} & 0.036 & 0.034 \\
Texas / MLP4 & 0.249 & \textbf{0.271} & 0.251 \\
\bottomrule
\end{tabular}
\caption{Online TPR @ FPR$=0.01$, $K=8$. Mean over 32 replicates. Best per row in bold.}
\label{tab:bavaria32_online_tpr001_K8}
\end{table}

% ------------------------------------------------------------------
\section{Cross-Validation on Independent Data}\label{app:crossval}

To verify that our findings generalize beyond a single training pipeline, we cross-validated all methods on an independent shadow-model collection trained without data augmentation or dropout.
This collection provides 224--1,480 shadow models per dataset (vs.\ 254 in our primary benchmark).

\paragraph{Absolute vulnerability.}
Models trained without augmentation are substantially more vulnerable to MIA on image datasets: CIFAR-10 ResNet AUC increases from 0.613 to 0.874, CIFAR-100 ResNet from 0.978 to 0.993.
This is consistent with the findings of \citet{sablayrolles2019white}, who first documented this effect at scale on ImageNet (attack accuracy dropping from 90\% to 68\% with standard augmentation).
On tabular datasets, the difference is smaller since these tasks already lack augmentation.

\paragraph{Method ranking preservation.}
Despite the large shift in absolute vulnerability, the qualitative ordering is preserved across both data sources: BaVarIA and LiRA (with BASE4 $\equiv$ LiRA at $K \geq 64$) lead, then BASE3, while the pooled BASE1 ($=$ RMIA) trails; within the BaVarIA pair, BaVarIA-$t$ leads on AUC (12 of 12 testbeds at $K=64$), while at low FPR the two variants are effectively tied (BaVarIA-$n$ ahead on 7 of 12 on TPR@0.01, means within $0.001$).
At $K = 64$, BaVarIA-$n$ beats LiRA on TPR@0.01 in 10 of 12 testbeds.
The exceptions are the two CIFAR-100 testbeds, where LiRA retains a marginal edge ($|\Delta\text{TPR}| < 0.003$).
At $K = 8$ the same pattern holds: BaVarIA-$t$ leads LiRA in AUC on all 12 testbeds (mean $+0.005$), while at low FPR neither variant improves on LiRA (BaVarIA-$n$ ahead on 6 of 12, BaVarIA-$t$ on 2).

\paragraph{Convergence at large $K$.}
On the independent data, the convergence of BaVarIA to LiRA is even more apparent because the collection provides up to 1,480 shadow models.
At each testbed's largest available budget ($K = 254$; $K = 128$ for the two CINIC collections), BaVarIA-$n$ and LiRA agree in AUC to within $0.001$ on all 12 testbeds, confirming the theoretical convergence result of Appendix~\ref{app:bavaria}.

These findings support that the BaVarIA-$t$ AUC advantage at small $K$, the BASE hierarchy ordering, and the convergence at large $K$ hold across the two training pipelines and regularization settings tested.

% ------------------------------------------------------------------
\section{ELSA: Discriminative Scoring}\label{app:elsa}

For any exponential-family model, the log-likelihood ratio is a linear function of the sufficient statistics (Section~\ref{sec:framework}):
\begin{equation}\label{eq:elsa_linear}
\LLR(z) = w^\top \phi(z),
\end{equation}
where $\phi$ is a feature map determined by the distributional family.
Rather than estimating $w$ generatively (via MLE of the distributional parameters), we can estimate it discriminatively via logistic regression on shadow-model data.

\paragraph{Feature map.}
We define a universal feature map that subsumes the sufficient statistics of all distributional families considered in this paper:
\begin{equation}\label{eq:elsa_features}
\phi(p) = \begin{bmatrix} 1,\; \log(-\log p),\; \log p,\; \log(1{-}p),\;
\varphi^2,\; p,\; p^2 \end{bmatrix}^\top,
\end{equation}
where $\varphi = \log(p/(1-p))$ is the rescaled logit.
The logit $\varphi = \log p - \log(1{-}p)$ is itself a linear combination of two features already in $\phi(p)$ and cannot enter as a separate column without making the design collinear; the $\varphi$-based variants instead tie the weights on $\log p$ and $\log(1{-}p)$ to a single coefficient $w_\varphi$ with opposite signs.
Different distributional families correspond to different active subsets (or tied combinations) of~$\phi$: the Gaussian model on log-odds uses $\{1, \varphi, \varphi^2\}$, the Beta model uses $\{1, \log p, \log(1{-}p)\}$, and the Gamma model uses $\{1, \log(-\log p), \log p\}$.

\begin{table}
\centering
\caption{ELSA: AUC at $K = 64$, mean over 32 replicates (WideResNet image and MLP4 tabular testbeds). Discriminative ELSA recovers but does not beat the generative Gaussian estimators; best per column in bold.}\label{tab:elsa_auc} 
\small
\begin{tabular}{lcccccc}
\toprule
Method & C10 & C100 & CN10 & Loc & Pur & Tex \\
\midrule
LiRA        & .720 & \textbf{.908} & \textbf{.756} & \textbf{.929} & .649 & \textbf{.895} \\
BASE3       & \textbf{.723} & \textbf{.908} & .753 & .928 & \textbf{.654} & .894 \\
BASE2       & .721 & .906 & .755 & .923 & \textbf{.654} & .886 \\
RMIA        & .691 & .882 & .708 & .843 & .642 & .862 \\
\midrule
ELSA1       & .588 & .805 & .613 & .802 & .563 & .812 \\
ELSA2$\phi$ & .684 & .884 & .701 & .925 & .620 & .882 \\
ELSA3$\beta$ & .679 & .871 & .701 & .925 & .615 & .877 \\
ELSA3$\gamma$ & .682 & .879 & .701 & .925 & .618 & .880 \\
ELSA3$\phi$ & .701 & .897 & .744 & .926 & .629 & .888 \\
ELSA-full   & .699 & .893 & .742 & .925 & .628 & .885 \\
\bottomrule
\end{tabular}
\end{table}

\paragraph{Training.}
For each data point $i$, we fit weights $w_i$ by minimizing the ridge-penalized logistic loss over shadow-model observations: 
\begin{equation}\label{eq:elsa_loss}
\hat{w}_i = \arg\min_w \left[ -\sum_{k=1}^{K} \Big( m_{i,k} \log \sigma\!\big(w^\top \phi_k\big) + (1 - m_{i,k}) \log\!\big(1 - \sigma(w^\top \phi_k)\big) \Big) + \lambda \|w_{-0}\|^2 \right],
\end{equation}
where $\phi_k = \phi(p_{i,k})$, $\sigma$ is the sigmoid function, and $\lambda$ is the ridge penalty applied to all weights except the intercept~$w_0$.
The intercept is excluded from penalization following standard practice in regularized logistic regression: including it would shrink all predictions toward the $0.5$ base rate, preventing the model from learning even the marginal membership probability.

Distribution-specific ELSA variants are obtained by feature masking: setting $w_j = 0$ for features outside the active set forces the model to use only the sufficient statistics of the chosen family.
For example, ELSA2$\phi$ optimizes only $\{w_0, w_\varphi\}$, recovering a discriminatively trained analogue of the equal-variance Gaussian LLR~\eqref{eq:gauss_equal_var}, while ELSA3$\phi$ adds $w_{\varphi^2}$ to capture the full Gaussian structure~\eqref{eq:gauss_llr}.

\paragraph{Results.}
Tables~\ref{tab:elsa_auc} and~\ref{tab:elsa_tlr} report AUC and TPR@0.01 at $K = 64$, mean over 32 replicates, for the ELSA variants.
The feature maps correspond to different distributional initializations:
ELSA1 trains only the intercept, with the $\log p$ weight fixed at unity (a discriminatively recentered exponential/BASE1 score);
ELSA2$\phi$ uses the rescaled logit $\varphi$ (two features, inspired by the Gaussian model);
ELSA3$\beta$ and ELSA3$\gamma$ use three features from the Beta and Gamma sufficient statistics;
ELSA3$\phi$ adds the quadratic term $\varphi^2$;
and ELSA-full uses all seven features.

\paragraph{AUC.}
On AUC (Table~\ref{tab:elsa_auc}), discriminative training does not improve on the generative estimators: LiRA\,($=$\,BASE4) and BASE3 lead on every testbed, and each ELSA variant trails its matched generative twin (ELSA2$\phi$ vs.\ BASE2: $-0.025$ mean AUC; ELSA1 vs.\ RMIA: $-0.074$). Learning the feature map recovers the Gaussian LLR but does not beat the fixed one.

\begin{table}
\centering
\caption{ELSA: TPR@0.01 at $K = 64$, mean over 32 replicates (WideResNet and MLP4 testbeds). The generative BASE hierarchy leads; ELSA2$\phi$ only matches its BASE2 twin; best per column in bold.}\label{tab:elsa_tlr}
\small
\begin{tabular}{lrrrrrr}
\toprule
Method & C10 & C100 & CN10 & Loc & Pur & Tex \\
\midrule
LiRA        & 0.129 & 0.361 & 0.178 & 0.355 & 0.053 & \textbf{0.357} \\
BASE3       & \textbf{0.135} & \textbf{0.363} & 0.178 & \textbf{0.366} & \textbf{0.063} & 0.355 \\
BASE2       & 0.127 & 0.345 & 0.178 & 0.283 & 0.059 & 0.280 \\
RMIA        & 0.112 & 0.326 & 0.129 & 0.189 & 0.049 & 0.281 \\
\midrule
ELSA1       & 0.054 & 0.178 & 0.071 & 0.137 & 0.021 & 0.213 \\
ELSA2$\phi$ & 0.118 & 0.316 & 0.180 & 0.340 & 0.057 & 0.320 \\
ELSA3$\beta$ & 0.107 & 0.274 & 0.177 & 0.319 & 0.053 & 0.296 \\
ELSA3$\gamma$ & 0.114 & 0.301 & 0.179 & 0.329 & 0.055 & 0.309 \\
ELSA3$\phi$ & 0.119 & 0.256 & \textbf{0.183} & 0.317 & 0.054 & 0.320 \\
ELSA-full   & 0.110 & 0.213 & 0.180 & 0.303 & 0.053 & 0.295 \\
\bottomrule
\end{tabular}
\end{table}

\paragraph{TPR at low FPR.}
At TPR@0.01 (Table~\ref{tab:elsa_tlr}) the same pattern holds: the generative BASE3 and LiRA lead on most testbeds, and ELSA2$\phi$ only draws level with its BASE2 twin (mean $+0.010$) without reaching LiRA. Richer discriminative variants give no consistent gain.

\paragraph{Practical assessment.}
ELSA is best read as a diagnostic: it confirms the parametric Gaussian LLR is well-specified rather than improving on it, at the cost of per-point logistic regression and sensitivity to feature selection. For a simple, robust attack we recommend the generative estimators---BaVarIA, or the simpler BASE variants at small budgets---over discriminative feature-map learning. 

% ------------------------------------------------------------------
\section{Offline Setting Analysis}\label{app:offline}

In the online (audit) setting used throughout this paper, each target point appears in roughly half of the shadow-model training sets, providing per-point IN and OUT observations.
In the offline (practical) setting, shadow models are trained on data from the same distribution as the target model but the target point does not appear in any shadow training set~\citep{carlini2022membership}.
All shadow observations for a given target point are therefore OUT-class samples; no direct IN-class observations are available.

\paragraph{Existing offline approaches.}
LiRA~\citep{carlini2022membership} estimates $\mu_{\text{out}}$ per-point from all (OUT-only) shadows but must approximate $\mu_{\text{in}}$ via a global mean-shift heuristic: $\hat\mu_{\text{in}} = \hat\mu_{\text{out}} + \Delta$, where $\Delta$ is a population-level constant estimated from reference data.
With forced equal variances $\sigma_{\text{in}}^2 = \sigma_{\text{out}}^2 = \sigma^2$, the Gaussian LLR~\eqref{eq:gauss_llr} reduces to
\begin{equation}\label{eq:lira_offline}
\LLR_{\text{offline}}(\varphi) = \frac{\Delta}{\sigma^2}\left(\varphi - \hat\mu_{\text{out}} - \frac{\Delta}{2}\right),
\end{equation}
a linear function of the observation---an instance of the equal-variance Gaussian LLR~\eqref{eq:gauss_equal_var} with the shifted mean.
This represents a substantial degradation from the full four-parameter online LiRA, particularly at low FPR where per-point heterogeneity carries the strongest signal.

In practice, \citet{carlini2022membership} implement the offline attack via the Gaussian log-CDF: $\text{score}_i = \log\Phi\!\bigl((\varphi_{i,0} - \hat\mu_{i,\text{out}})/\hat\sigma\bigr)$, where $\Phi$ is the standard normal CDF.
When $\hat\sigma$ is a single global constant (the offline regime), this is a monotone transformation of the linear score~\eqref{eq:lira_offline} and produces the same ROC curve.
When $\hat\sigma$ varies per-point, the two formulations can differ in ranking. 

RMIA~\citep{zarifzadeh2024lowcost} uses a population-reference mechanism that does not explicitly split shadow observations into IN and OUT classes, making its score formula applicable without modification in the offline setting.
However, the offline score distribution shifts because all shadow models are OUT-only: the likelihood ratios computed against reference points are systematically biased when the target point never appeared in training.
Empirically, RMIA shows a small but consistent degradation in our experiments ($\Delta_{\text{AUC}} \approx 0.008$; Tables~\ref{tab:offline_auc}--\ref{tab:offline_tpr}), confirming that the offline setting still carries a cost even for reference-based methods. 

BASE~\citep{lassila2025base} introduces a scaling factor $\alpha \in [0,1]$ on the log-sum-exp term for the offline case:
\begin{equation}\label{eq:base_offline}
\text{BASE}_i^{\text{off}} = -\ell_{i,0} - \alpha \cdot \log\!\left(\tfrac{1}{K}\textstyle\sum_{k=1}^K e^{-\ell_{i,k}}\right),
\end{equation}
where the sum runs over OUT-only shadow models and $\alpha$ compensates for the systematically higher losses of models that never trained on the target point.
The BASE--RMIA equivalence (Proposition~\ref{prop:equiv}a) holds only in the online setting: offline, BASE applies the $\alpha$ scaling of~\eqref{eq:base_offline} while RMIA's score is unmodified, so the two attacks no longer coincide.

\paragraph{BASE hierarchy in the offline setting.}
Table~\ref{tab:offline_hierarchy} summarizes how each BASE variant adapts to the offline case.
The key pattern is that methods requiring IN-class parameters lose access to per-point IN observations and must fall back on population-level estimates.

\begin{table}
\centering
\caption{BASE hierarchy: online vs.\ offline parameter estimation.
In the offline setting, methods that estimate separate IN-class parameters collapse toward pooled estimation.}\label{tab:offline_hierarchy}
\small
\begin{tabular}{lll}
\toprule
Method & Online parameters & Offline fallback \\
\midrule
BASE1 & $\hat{z}_{i,\text{pool}}$ from all shadows & Same; optional $\alpha$ scaling~\eqref{eq:base_offline} \\
BASE2 & $\hat\mu_{i,0}, \hat\mu_{i,1}$; global $\sigma^2$ & $\hat\mu_{i,1} = \hat\mu_{i,0} + \Delta$; $\to$ const-var LLR \\
BASE3 & $\hat\mu_{i,0}, \hat\mu_{i,1}$; $\hat\sigma_i^2$ pooled & $\hat\mu_{i,1} = \hat\mu_{i,0} + \Delta$; $\to$ eq-var LLR \\
BASE4 & $\hat\mu_{i,m}, \hat\sigma_{i,m}^2$ per class & $\sigma_{\text{in}}^2 = \sigma_{\text{out}}^2$, $\mu_{\text{in}}$ via shift; $\to$ eq-var LLR \\
\bottomrule
\end{tabular}
\end{table}

BASE1 is inherently offline-friendly because it never splits shadow observations into IN and OUT classes: the pooled centering uses all available shadows regardless of membership status.
BASE2, together with BASE3 and BASE4, requires separate IN-class parameters and must estimate the missing $\mu_{\text{in}}$ via a global mean-shift heuristic (BASE2 already uses a constant variance; BASE3 estimates a single pooled variance).
Under these constraints they reduce to a constant- or equal-variance Gaussian LLR~\eqref{eq:gauss_equal_var} with the shifted mean, differing only in the $\Delta$-dependent centering point.

\paragraph{BaVarIA in the offline setting.}
BaVarIA's conjugate NIG prior provides a principled mechanism for the offline case.
With zero IN-class observations for the target point ($n_1 = 0$), the NIG posterior updates~\eqref{eq:nig_update1} reduce to the prior: $\mu'_{i,1} = \mu_{\varnothing,1}$, $\kappa'_{i,1} = \kappa_\varnothing$, $\alpha'_{i,1} = \alpha_\varnothing$, $\beta'_{i,1} = \beta_{\varnothing,1}$.
The IN-class parameters are thus set entirely by the empirical Bayes prior.

The prior estimation procedure is unchanged between online and offline settings.
Even in the offline case, the adversary controls the shadow model training and therefore knows which \emph{reference points} (the shadow models' own training data, disjoint from the target's) were IN or OUT for each shadow model.
The class-specific priors $(\mu_{\varnothing,m}, \beta_{\varnothing,m})$ are estimated from these reference-point observations (Appendix~\ref{app:bavaria}), which include both IN and OUT samples.
The offline limitation is purely per-point: for the specific target point being audited, the IN-class posterior collapses to the (well-estimated) IN-class prior, while the OUT-class posterior is fully updated from $K$ shadow observations.

This parallels what LiRA's hard-switch heuristic achieves (global variance when $K$ is small), but BaVarIA obtains it as a smooth limiting case of the same Bayesian machinery used in the online setting.
As per-point IN-class observations become available (online setting with increasing $K$), the posterior moves from population-level to per-point estimates.
The offline-to-online transition requires no separate implementation or additional hyperparameters.

\paragraph{Experimental results.}
Figures~\ref{fig:offline_auc}--\ref{fig:offline_tpr_rn} compare online and offline
performance for LiRA, RMIA, and the two BaVarIA variants across all testbeds (solid = online, dashed = offline).
Tables~\ref{tab:offline_auc} and~\ref{tab:offline_tpr} report the mean offline
performance and the degradation $\Delta$ = online $-$ offline.
All methods show positive degradation: online consistently outperforms offline, as expected.
BaVarIA-$n$ achieves the best offline AUC at $K \geq 64$ and the best offline TPR at FPR${=}0.01$ across all $K$, while BaVarIA-$t$ leads for AUC at $K{=}8$.
RMIA exhibits the smallest degradation ($\Delta_{\text{AUC}} \approx 0.008$), followed by BaVarIA-$n$ ($\Delta_{\text{AUC}} \approx 0.013$--$0.034$); LiRA shows the largest gap ($\Delta_{\text{AUC}} \approx 0.033$--$0.045$).
Overall, the offline setting remains viable across all approaches, with
absolute performance differences between methods being more practically relevant
than the online-to-offline gap.

No single method dominates: the best offline method varies across metrics and testbeds even at a fixed $K$.

% Offline AUC comparison table (compact 7-column)
\begin{table}
\centering
\caption{Offline AUC (mean over 12 testbeds, 32 replicates).
$\Delta$ = online $-$ offline. Bold: best offline value per $K$.}\label{tab:offline_auc}
\small
\begin{tabular}{l cc cc cc}
\toprule
 & \multicolumn{2}{c}{$K=8$} & \multicolumn{2}{c}{$K=64$} & \multicolumn{2}{c}{$K=254$} \\
\cmidrule(lr){2-3} \cmidrule(lr){4-5} \cmidrule(lr){6-7}
Method & Off & $\Delta$ & Off & $\Delta$ & Off & $\Delta$ \\
\midrule
LiRA & 0.755 & +0.033 & 0.772 & +0.042 & 0.775 & +0.045 \\
RMIA & 0.763 & +0.009 & 0.769 & +0.008 & 0.770 & +0.008 \\
BaVarIA-$n$ & 0.773 & +0.013 & \textbf{0.785} & +0.029 & \textbf{0.786} & +0.034 \\
BaVarIA-$t$ & \textbf{0.778} & +0.015 & 0.784 & +0.031 & 0.784 & +0.036 \\
\bottomrule
\end{tabular}
\end{table}

% Offline TPR comparison table (compact 7-column)
\begin{table}
\centering
\caption{Offline TPR at FPR${=}0.01$ (mean over 12 testbeds, 32 replicates).
$\Delta$ = online $-$ offline. Bold: best offline value.}\label{tab:offline_tpr}
\small
\begin{tabular}{l cc cc cc}
\toprule
 & \multicolumn{2}{c}{$K=8$} & \multicolumn{2}{c}{$K=64$} & \multicolumn{2}{c}{$K=254$} \\
\cmidrule(lr){2-3} \cmidrule(lr){4-5} \cmidrule(lr){6-7}
Method & Off & $\Delta$ & Off & $\Delta$ & Off & $\Delta$ \\
\midrule
LiRA & 0.174 & +0.034 & 0.216 & +0.050 & 0.233 & +0.051 \\
RMIA & 0.144 & +0.013 & 0.180 & +0.013 & 0.185 & +0.013 \\
BaVarIA-$n$ & \textbf{0.200} & +0.014 & \textbf{0.246} & +0.027 & \textbf{0.254} & +0.031 \\
BaVarIA-$t$ & 0.196 & +0.005 & 0.242 & +0.029 & 0.249 & +0.035 \\
\bottomrule
\end{tabular}
\end{table}

% ------------------------------------------------------------------
\subsection{Offline RAPID Baseline}\label{app:rapid}

RAPID \citep{he2024rapid} is an offline difficulty-calibration attack that feeds two features---the raw target loss and a calibrated score (target loss minus a reference-model mean)---to a small learned head. We reproduce it here entirely from our cached shadow outputs, with no shadow-model retraining: the calibrated feature is the offline BASE score (the same OUT-only reference used by our offline BASE, LiRA and BaVarIA-$n$), and an MLP head is trained on shadow-derived member/non-member labels leave-one-model-out---each audited target model is scored by a head trained only on the \emph{other} models, so the target's own membership never enters training---then frozen and applied to the target. This keeps RAPID a principled offline comparison alongside offline BASE, LiRA and BaVarIA-$n$.

\begin{table*}[t]
\centering
\footnotesize
\begin{tabular}{l cccc}
\toprule
Testbed & BASE & LiRA & RAPID & BaVarIA-$n$ \\
\midrule
CIFAR-10 / WRN & 0.082 & 0.087 & \textbf{0.115} & 0.074 \\
CIFAR-100 / WRN & 0.218 & 0.308 & \textbf{0.337} & 0.322 \\
CINIC-10 / WRN & 0.090 & 0.123 & 0.158 & \textbf{0.181} \\
Location / MLP3 & 0.148 & 0.327 & 0.393 & \textbf{0.412} \\
Purchase / MLP3 & 0.034 & 0.023 & \textbf{0.037} & 0.026 \\
Texas / MLP3 & 0.177 & 0.250 & 0.238 & \textbf{0.268} \\
CIFAR-10 / RN18 & 0.046 & 0.038 & \textbf{0.048} & 0.032 \\
CIFAR-100 / RN18 & 0.251 & 0.579 & 0.431 & \textbf{0.663} \\
CINIC-10 / RN18 & 0.123 & 0.223 & 0.191 & \textbf{0.277} \\
Location / MLP4 & 0.132 & 0.299 & 0.212 & \textbf{0.333} \\
Purchase / MLP4 & 0.040 & 0.038 & \textbf{0.050} & 0.042 \\
Texas / MLP4 & 0.192 & 0.294 & 0.272 & \textbf{0.321} \\
\bottomrule
\end{tabular}
\caption{Offline TPR @ FPR$=0.01$ at $K=64$, per testbed (mean over 32 replicates); best per row in bold. RAPID is an offline difficulty-calibration attack: an MLP head on two features---the raw loss and the offline BASE calibrated score (the same OUT-only reference used by BASE/LiRA/BaVarIA-$n$)---trained on shadow-derived member/non-member labels, frozen, and applied to the target.}
\label{tab:rapid_comparison}
\end{table*}

On offline TPR @ FPR $=0.01$ at $K=64$ (Table~\ref{tab:rapid_comparison}), RAPID is a competitive baseline: it improves on LiRA and BASE on most testbeds and wins outright on five of twelve (both CIFAR-10, both Purchase, and CIFAR-100/WRN). BaVarIA-$n$ nonetheless remains best overall ($7/12$ testbeds), with the largest margins on the harder testbeds (CINIC, Location, Texas, CIFAR-100/RN18) where per-class variance modelling matters most. At FPR $=0.001$ the ordering is the same except on Location/MLP3, leaving BaVarIA-$n$ and RAPID ahead on six testbeds each; on AUC (Table~\ref{tab:rapid_comparison_auc}) RAPID is strongest (best on 8 of 12), consistent with its learned head optimizing global ranking rather than the low-FPR tail. Figures~\ref{fig:rapid_scaling_wrn} and~\ref{fig:rapid_scaling_rn} trace the offline TPR@0.01 scaling against the shadow budget $K$ across all testbeds.

\begin{table*}[t]
\centering
\footnotesize
\begin{tabular}{l cccc}
\toprule
Testbed & BASE & LiRA & RAPID & BaVarIA-$n$ \\
\midrule
CIFAR-10 / WRN & 0.659 & 0.677 & \textbf{0.703} & 0.666 \\
CIFAR-100 / WRN & 0.825 & 0.856 & \textbf{0.893} & 0.847 \\
CINIC-10 / WRN & 0.670 & 0.691 & 0.745 & \textbf{0.749} \\
Location / MLP3 & 0.816 & 0.902 & \textbf{0.941} & 0.927 \\
Purchase / MLP3 & 0.614 & 0.607 & \textbf{0.622} & 0.596 \\
Texas / MLP3 & 0.822 & 0.816 & \textbf{0.856} & 0.832 \\
CIFAR-10 / RN18 & 0.628 & 0.616 & \textbf{0.630} & 0.591 \\
CIFAR-100 / RN18 & 0.892 & 0.937 & 0.943 & \textbf{0.973} \\
CINIC-10 / RN18 & 0.771 & 0.807 & 0.831 & \textbf{0.863} \\
Location / MLP4 & 0.809 & 0.887 & 0.872 & \textbf{0.902} \\
Purchase / MLP4 & 0.628 & 0.634 & \textbf{0.649} & 0.629 \\
Texas / MLP4 & 0.820 & 0.832 & \textbf{0.868} & 0.850 \\
\bottomrule
\end{tabular}
\caption{Offline AUC at $K=64$ (RAPID MLP head); best per row in bold.}
\label{tab:rapid_comparison_auc}
\end{table*}

\begin{table}[t]
\centering
\footnotesize
\begin{tabular}{llccc}
\toprule
Calibrated score & Attack & AUC & TPR@.01 & TPR@.001 \\
\midrule
BASE calibration (RMIA/BASE) & score alone & 0.746 & 0.128 & 0.041 \\
 & $+$ raw loss (RAPID) & 0.796 & 0.207 & 0.091 \\
\addlinespace
LiRA (variance-aware) & score alone & 0.772 & 0.216 & 0.093 \\
 & $+$ raw loss (RAPID) & 0.807 & 0.207 & 0.075 \\
\addlinespace
BaVarIA-$n$ (variance-aware) & score alone & 0.785 & 0.246 & 0.117 \\
 & $+$ raw loss (RAPID) & 0.797 & 0.247 & 0.118 \\
\addlinespace
\bottomrule
\end{tabular}
\caption{Mean over the 12 testbeds (offline, $K=64$): a calibrated score alone vs.\ RAPID (that score $+$ raw loss, MLP head). Adding the raw-loss feature lifts the plain BASE calibration substantially but does almost nothing once the calibrated component is already variance-aware (BaVarIA-$n$)---indicating that it gives ``diminishing returns'' once variance is modelled.}
\label{tab:rapid_varianceaware}
\end{table}

Finally, Table~\ref{tab:rapid_varianceaware} isolates RAPID's raw-loss correction. Adding it to a plain mean-only BASE(1) calibration lifts mean TPR@0.01 markedly ($0.128 \to 0.207$); adding it to BaVarIA-$n$ changes almost nothing ($+0.001$ at both tail FPRs), and for LiRA it trades the tail for the bulk (AUC $0.772 \to 0.807$ but TPR@0.001 $0.093 \to 0.075$)---the diminishing returns noted in Section~\ref{sec:related}. This may suggest that RAPID's tail gain is largely a variance correction that the upper BASE hierarchy and BaVarIA already supply by construction.

% ------------------------------------------------------------------
\section{Membership Inference under DP-SGD}\label{app:dpsgd}

This appendix accompanies Section~\ref{sec:dpsgd}. We retrain the entire pipeline---target and shadow models alike---with DP-SGD \citep{abadi2016deep}: per-sample gradient clipping at $C=10$ followed by Gaussian noise of multiplier $\sigma$, using the Opacus implementation \citep{yousefpour2021opacus}. We evaluate four testbeds (CIFAR-10/WRN, CIFAR-100/WRN, Location/MLP3, Purchase/MLP3) at $K=64$, online, over $\sigma \in \{0, 0.01, 0.1, 1.0\}$, mean over 32 replicates.

For per-sample-gradient compatibility the pipeline replaces BatchNorm with GroupNorm and uses larger batches with retuned learning rates; this same calibration is shared across all models in the grid, so within a testbed the only variable across rows is $\sigma$. The image testbeds therefore lose accuracy relative to the main paper ($-12.7$ pp on CIFAR-10, $-27.5$ pp on CIFAR-100, driven by the BatchNorm$\to$GroupNorm switch), while the plain-MLP tabular testbeds are barely affected. Absolute risk numbers here are consequently not comparable to the main-paper testbeds---we read them \emph{relative} across noise levels.

Table~\ref{tab:privacy} reports the privacy budget and utility per noise level; Tables~\ref{tab:auc} and~\ref{tab:tpr01} give AUC and TPR@0.01 across testbeds and noise levels; Figures~\ref{fig:def_baseline}--\ref{fig:atk_bavaria_t} show the corresponding ROC curves. We set $\delta \approx 1/n_{\mathrm{train}}$, and report $\varepsilon$ from the Opacus accountant (rounded down to the nearest order of magnitude when ${>}100$): at $\sigma=0.01$ and $\sigma=0.1$ the budget is vacuous ($\varepsilon > 100\mathrm{K}$ and $\varepsilon$ in the thousands, respectively), while at $\sigma=1.0$ it becomes meaningful ($\varepsilon \approx 6$--$9$) at a substantial utility cost.
In quantitative terms (Section~\ref{sec:dpsgd}), BASE/RMIA progressively overtakes the three Gaussian methods: in AUC it beats them on one testbed at $\sigma=0$, two at $\sigma=0.01$, three at $\sigma=0.1$ and all four at $\sigma=1.0$ (Table~\ref{tab:auc}); on TPR@0.01 the same trend is present but not monotone (Table~\ref{tab:tpr01}).

The intermediate regime $\sigma=0.01$ is notable: it costs little accuracy (CIFAR-10 76\%$\to$72\%, Purchase 90\%$\to$87\%; Table~\ref{tab:privacy}) yet substantially reduces attack TPR (best-attack TPR@0.01 on CIFAR-10 drops $0.156 \to 0.049$; Table~\ref{tab:tpr01})---empirical leakage falls well before the formal guarantee becomes meaningful.

\begin{table}
\centering
\caption{Privacy budget and model utility per testbed and noise level (DP-SGD $C=10$ and $\delta \approx 1/n_{\mathrm{train}}$); $\varepsilon$ from the Opacus accountant, rounded down to the nearest order of magnitude when ${>}100$.}\label{tab:privacy}
\small
\begin{tabular}{ll rrr}
\toprule
Testbed & Defence & Test Acc (\%) & $\varepsilon$ & $\delta$ \\
\midrule
\multirow{4}{*}{CIFAR-10 / WRN}
  & Baseline      & 76.1 & ---       & ---  \\
  & $\sigma=0.01$ & 71.6 & ${>}1$M   & $3.3\!\times\!10^{-5}$ \\
  & $\sigma=0.1$  & 71.2 & ${>}10$K  & $3.3\!\times\!10^{-5}$ \\
  & $\sigma=1.0$  & 54.1 & 6.7       & $3.3\!\times\!10^{-5}$ \\
\midrule
\multirow{4}{*}{CIFAR-100 / WRN}
  & Baseline      & 40.0 & ---       & ---  \\
  & $\sigma=0.01$ & 27.7 & ${>}1$M   & $3.3\!\times\!10^{-5}$ \\
  & $\sigma=0.1$  & 26.4 & ${>}10$K  & $3.3\!\times\!10^{-5}$ \\
  & $\sigma=1.0$  &  9.0 & 6.1       & $3.3\!\times\!10^{-5}$ \\
\midrule
\multirow{4}{*}{Location / MLP3}
  & Baseline      & 66.4 & ---       & ---  \\
  & $\sigma=0.01$ & 64.6 & ${>}100$K & $4.0\!\times\!10^{-4}$ \\
  & $\sigma=0.1$  & 55.8 & ${>}1$K   & $4.0\!\times\!10^{-4}$ \\
  & $\sigma=1.0$  &  7.1 & 8.9       & $4.0\!\times\!10^{-4}$ \\
\midrule
\multirow{4}{*}{Purchase / MLP3}
  & Baseline      & 89.6 & ---       & ---  \\
  & $\sigma=0.01$ & 87.3 & ${>}1$M   & $1.0\!\times\!10^{-5}$ \\
  & $\sigma=0.1$  & 86.4 & ${>}1$K   & $1.0\!\times\!10^{-5}$ \\
  & $\sigma=1.0$  & 77.3 & 5.7       & $1.0\!\times\!10^{-5}$ \\
\bottomrule
\end{tabular}
\end{table}

\begin{table}
\centering
\caption{AUC across testbeds and defense levels. Mean over 32 replicates, $K=64$, online. Best method per row in bold. BaVarIA-$n$ / BaVarIA-$t$ track or beat LiRA at every noise level (16/16 cells in AUC); BASE progressively overtakes the Gaussian methods as $\sigma$ grows, winning 4/4 testbeds at $\sigma=1.0$. The crossover mirrors the small-$K$ regime in the absence of defense---simpler distributional families dominate when per-sample signal is scarce, whether scarcity comes from few shadow models or from added DP-SGD noise.}\label{tab:auc}
\small
\begin{tabular}{ll r cccc}
\toprule
Testbed & Defense & Acc\,(\%) & LiRA & BaVarIA-$n$ & BaVarIA-$t$ & BASE \\
\midrule
\multirow{4}{*}{\shortstack[l]{CIFAR-10\\WRN}}
  & Baseline      & 76.1 & .773 & .775 & \textbf{.776} & .762 \\
  & $\sigma=0.01$ & 71.6 & .674 & \textbf{.678} & .678 & .653 \\
  & $\sigma=0.1$  & 71.2 & .641 & \textbf{.645} & .645 & .630 \\
  & $\sigma=1.0$  & 54.1 & .509 & .510 & .510 & \textbf{.526} \\
\midrule
\multirow{4}{*}{\shortstack[l]{CIFAR-100\\WRN}}
  & Baseline      & 40.0 & .741 & .743 & \textbf{.744} & .742 \\
  & $\sigma=0.01$ & 27.7 & .600 & .603 & .604 & \textbf{.620} \\
  & $\sigma=0.1$  & 26.4 & .584 & .588 & .589 & \textbf{.607} \\
  & $\sigma=1.0$  &  9.0 & .505 & .505 & .505 & \textbf{.522} \\
\midrule
\multirow{4}{*}{\shortstack[l]{Location\\MLP3}}
  & Baseline      & 66.4 & .823 & .824 & \textbf{.825} & .819 \\
  & $\sigma=0.01$ & 64.6 & .741 & .743 & \textbf{.743} & .726 \\
  & $\sigma=0.1$  & 55.8 & .622 & .626 & .627 & \textbf{.642} \\
  & $\sigma=1.0$  &  7.1 & .511 & .511 & .512 & \textbf{.535} \\
\midrule
\multirow{4}{*}{\shortstack[l]{Purchase\\MLP3}}
  & Baseline      & 89.6 & .532 & .533 & .534 & \textbf{.553} \\
  & $\sigma=0.01$ & 87.3 & .530 & .532 & .533 & \textbf{.550} \\
  & $\sigma=0.1$  & 86.4 & .522 & .523 & .524 & \textbf{.541} \\
  & $\sigma=1.0$  & 77.3 & .506 & .507 & .507 & \textbf{.522} \\
\bottomrule
\end{tabular}
\end{table}

\begin{table}
\centering
\caption{TPR at FPR$=0.01$ across testbeds and defense levels. Mean over 32 replicates, $K=64$, online. Best per row in bold. The Gaussian-family BaVarIA-$n$ matches or beats LiRA at every noise level; BASE leads at heavy noise. At $\sigma=1.0$ all attacks approach random guessing (TPR$\le 0.015$ vs.\ random at $0.01$).}\label{tab:tpr01}
\small
\begin{tabular}{ll r cccc}
\toprule
Testbed & Defense & Acc\,(\%) & LiRA & BaVarIA-$n$ & BaVarIA-$t$ & BASE \\
\midrule
\multirow{4}{*}{\shortstack[l]{CIFAR-10\\WRN}}
  & Baseline      & 76.1 & .150 & \textbf{.156} & .156 & .148 \\
  & $\sigma=0.01$ & 71.6 & .045 & \textbf{.049} & .049 & .048 \\
  & $\sigma=0.1$  & 71.2 & .032 & .037 & .036 & \textbf{.038} \\
  & $\sigma=1.0$  & 54.1 & .010 & .011 & .011 & \textbf{.013} \\
\midrule
\multirow{4}{*}{\shortstack[l]{CIFAR-100\\WRN}}
  & Baseline      & 40.0 & .082 & .087 & .086 & \textbf{.093} \\
  & $\sigma=0.01$ & 27.7 & .022 & .023 & .023 & \textbf{.028} \\
  & $\sigma=0.1$  & 26.4 & .019 & .021 & .021 & \textbf{.024} \\
  & $\sigma=1.0$  &  9.0 & .009 & .009 & .010 & \textbf{.012} \\
\midrule
\multirow{4}{*}{\shortstack[l]{Location\\MLP3}}
  & Baseline      & 66.4 & .139 & .146 & .144 & \textbf{.153} \\
  & $\sigma=0.01$ & 64.6 & .081 & \textbf{.084} & .081 & .075 \\
  & $\sigma=0.1$  & 55.8 & .031 & .033 & .032 & \textbf{.036} \\
  & $\sigma=1.0$  &  7.1 & .012 & .012 & .012 & \textbf{.015} \\
\midrule
\multirow{4}{*}{\shortstack[l]{Purchase\\MLP3}}
  & Baseline      & 89.6 & .013 & .015 & .015 & \textbf{.021} \\
  & $\sigma=0.01$ & 87.3 & .013 & .015 & .014 & \textbf{.017} \\
  & $\sigma=0.1$  & 86.4 & .012 & .013 & .013 & \textbf{.016} \\
  & $\sigma=1.0$  & 77.3 & .010 & .011 & .011 & \textbf{.012} \\
\bottomrule
\end{tabular}
\end{table}

% ------------------------------------------------------------------
\section{Scaling and ROC Curves (All Testbeds)}\label{app:curves}

This appendix collects the per-testbed scaling and ROC curves referenced throughout: BaVarIA scaling over all 12 testbeds (Figures~\ref{fig:scaling_bv_auc_wrn}--\ref{fig:scaling_rn_tpr}), the BASE-hierarchy scaling that validates the concordance scores of Table~\ref{tab:summary32_concordance} (Figures~\ref{fig:online_base_auc}--\ref{fig:online_base_tpr_rn}), the online-vs-offline scaling (Figures~\ref{fig:offline_auc}--\ref{fig:offline_tpr_rn}), and representative log--log ROC curves (Figures~\ref{fig:roc_base_k8}--\ref{fig:roc_bav_offline_k64}).
\subsection{BaVarIA scaling (all 12 testbeds)}
Online AUC and TPR @ FPR$=0.01$ vs.\ $K$ for LiRA, BaVarIA-$n$, and BaVarIA-$t$ (mean $\pm 1$ SE over 32 replicates), across all 12 testbeds. BaVarIA tracks LiRA closely, with the largest gains around the $K=64$ variance-rule cutover; BaVarIA-$t$ leads on AUC and BaVarIA-$n$ at low FPR. At the smallest budgets BaVarIA-$t$ can trail LiRA at low FPR, and BaVarIA-$n$ trails slightly on AUC.
\subsection{BASE hierarchy scaling (all 12 testbeds)}
Online AUC and TPR @ FPR$=0.01$ vs.\ $K$ for the BASE hierarchy (BASE1--BASE4). The simpler families lead at small $K$ and the richer families at large $K$; these curves validate the ordering-reversal concordance scores of Table~\ref{tab:summary32_concordance}.

\subsection{Online vs.\ offline scaling}
Online (solid) vs.\ offline (dashed) AUC and TPR @ FPR$=0.01$ as a function of $K$. Offline consistently trails online, but BaVarIA retains its advantage over LiRA in the offline setting (full analysis in Appendix~\ref{app:offline}).

\subsection{ROC curves (WRN/MLP3 testbeds)}
Representative log--log ROC curves. The BASE-hierarchy panels (Figures~\ref{fig:roc_base_k8},~\ref{fig:roc_base_k254}) show the small-$K$ to large-$K$ ordering reversal quantified in Table~\ref{tab:summary32_concordance}; the BaVarIA panels (Figures~\ref{fig:roc_bav_online_k64},~\ref{fig:roc_bav_offline_k64}) compare LiRA, BASE, BaVarIA-$n$, and BaVarIA-$t$ at $K=64$, online and offline. BaVarIA matches or exceeds LiRA on most testbeds, with the widest margin offline, where the NIG prior supplies the IN-class likelihood that LiRA's offline score lacks; LiRA retains an edge on the CIFAR-10 testbeds.

\clearpage

\begin{figure}
\centering
\includegraphics[width=\linewidth]{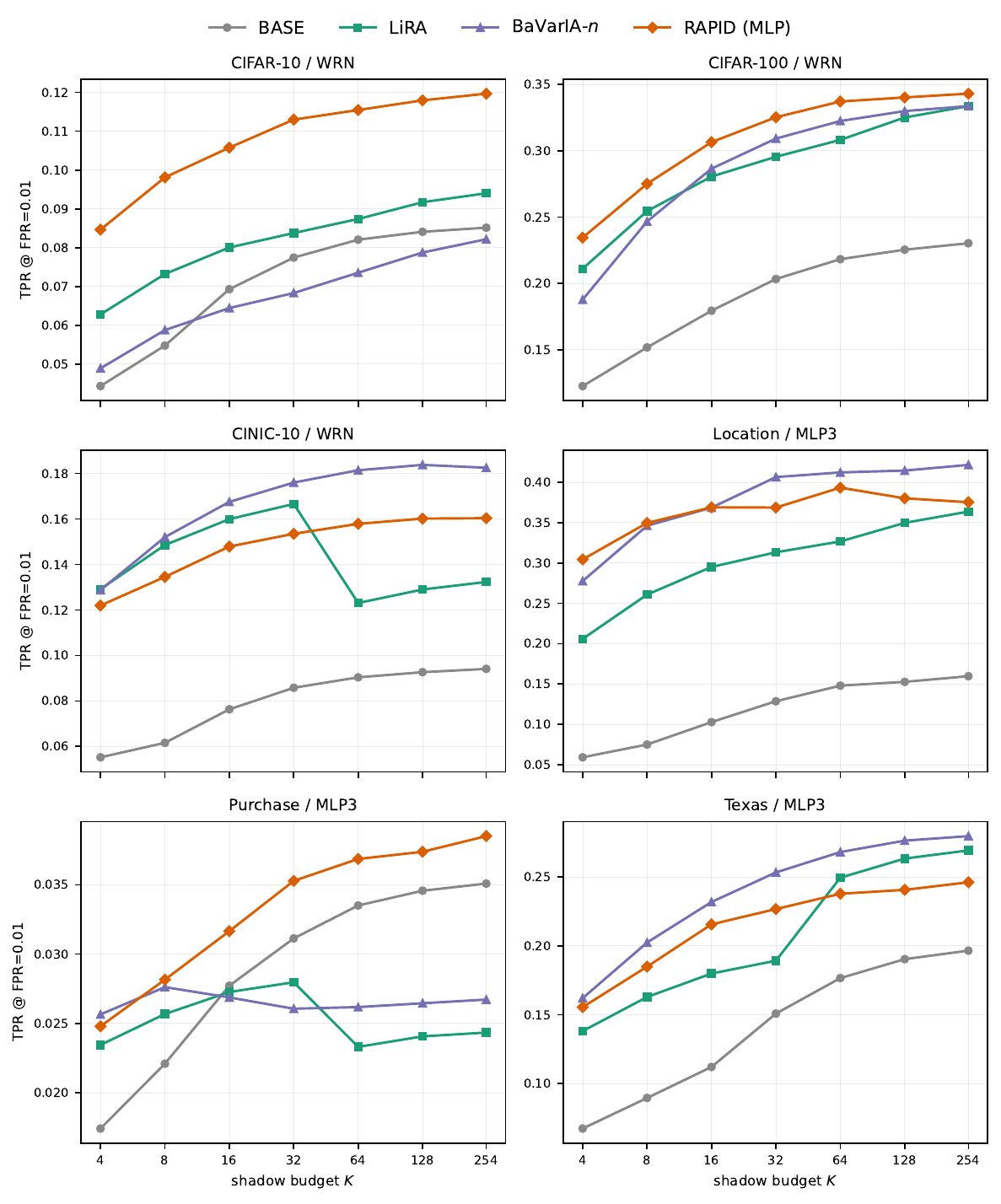}
\caption{Offline TPR @ FPR$=0.01$ vs.\ shadow budget $K$ (WRN and MLP3 testbeds), mean over 32 replicates: BASE, LiRA, BaVarIA-$n$, and RAPID (MLP head). RAPID tracks the calibrated attacks and is competitive at small $K$, while BaVarIA-$n$ pulls ahead as the budget grows.}
\label{fig:rapid_scaling_wrn}
\end{figure}

\begin{figure}
\centering
\includegraphics[width=\linewidth]{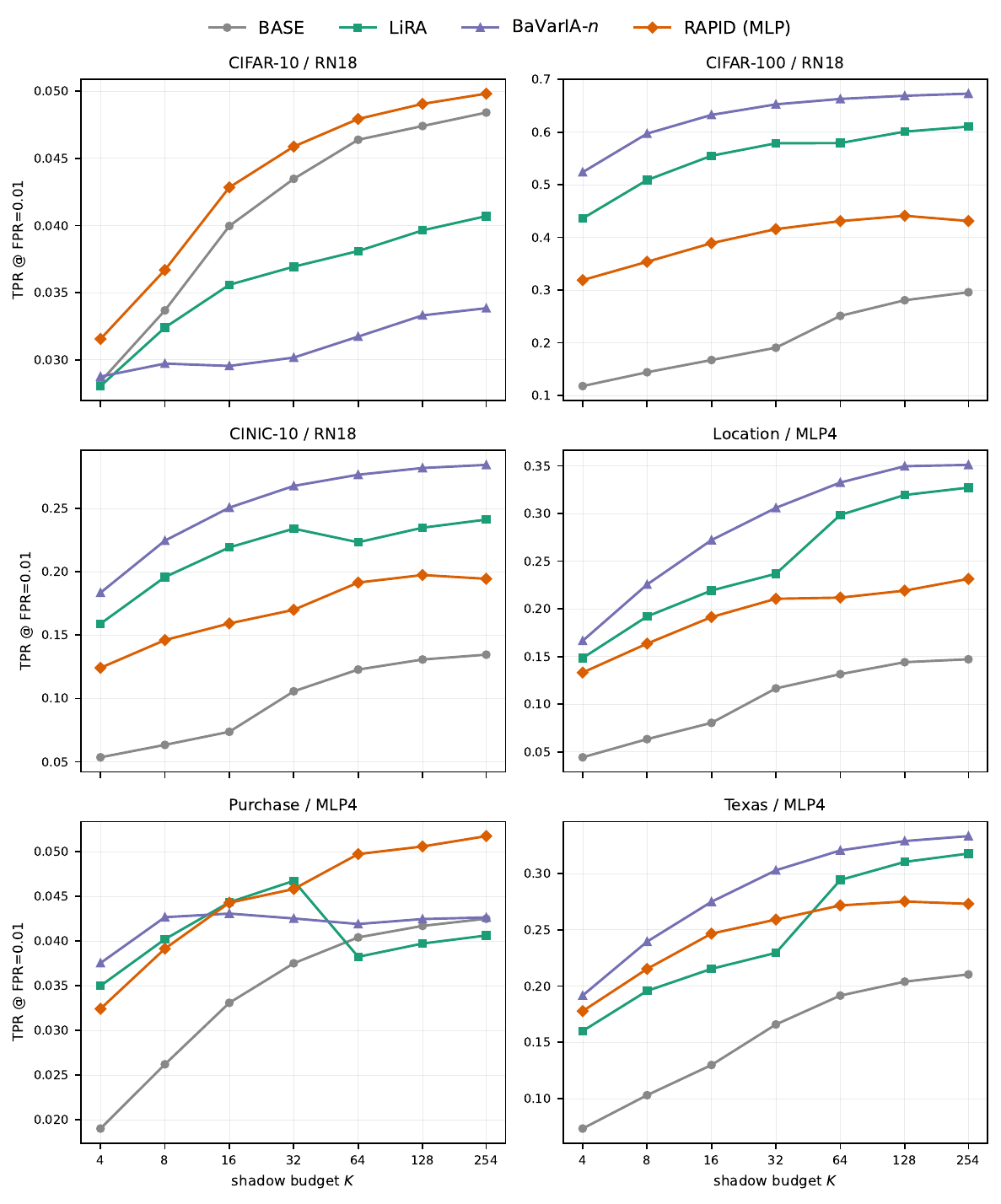}
\caption{Offline TPR @ FPR$=0.01$ vs.\ shadow budget $K$ (RN18 and MLP4 testbeds); same methods and protocol as Figure~\ref{fig:rapid_scaling_wrn}.}
\label{fig:rapid_scaling_rn}
\end{figure}

\begin{figure}
\centering
\includegraphics[width=\linewidth]{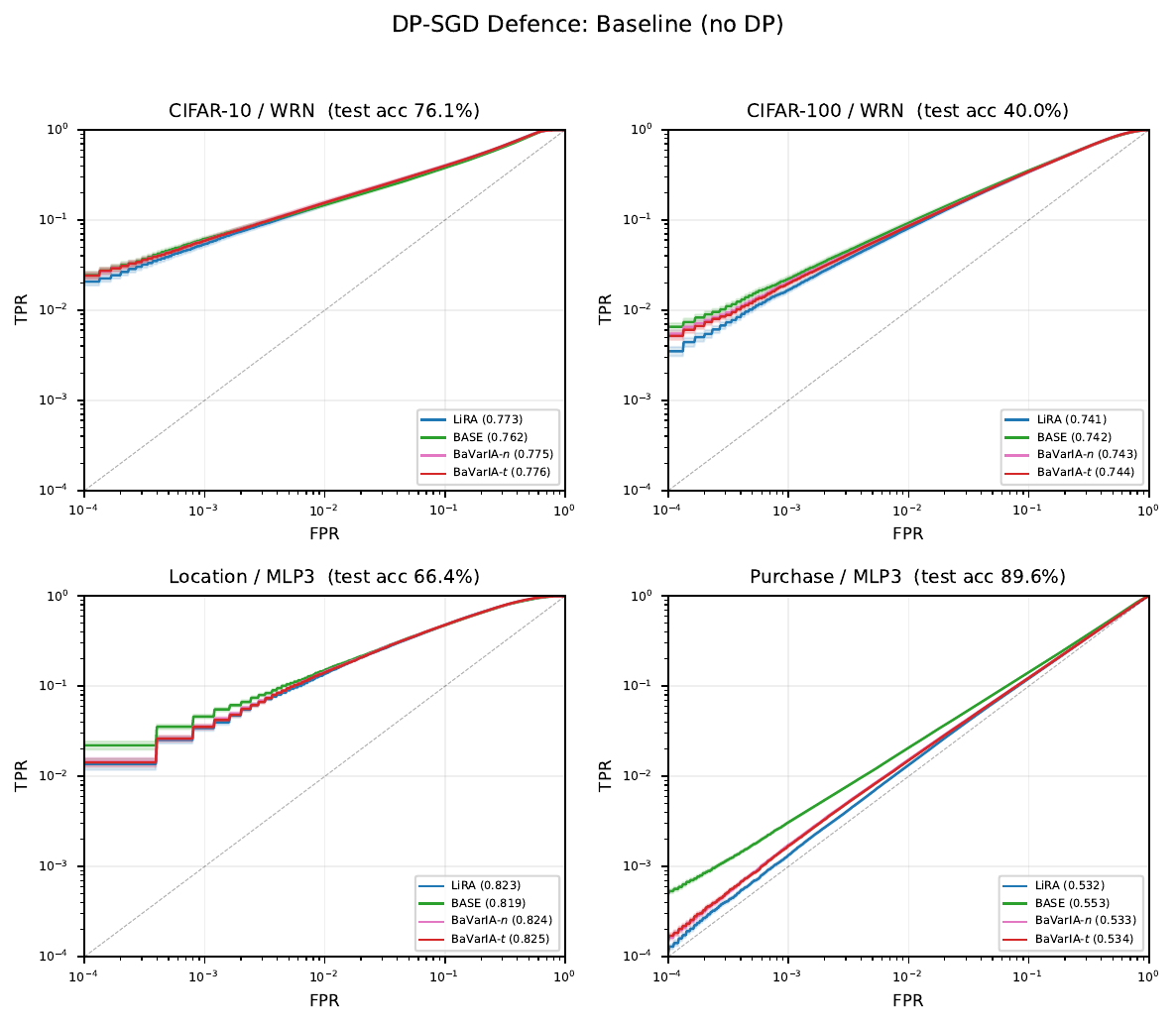}
\caption{Defense-level view, baseline (clipping;  $\sigma=0$): all four attacks (LiRA, BASE, BaVarIA-$n$, BaVarIA-$t$) overlaid on each of four testbeds. $K=64$, online, mean ROC over 32 replicates, log--log axes; legend AUC included. Panel titles include mean test accuracy across shadow models.}\label{fig:def_baseline}
\end{figure}

\begin{figure}
\centering
\includegraphics[width=\linewidth]{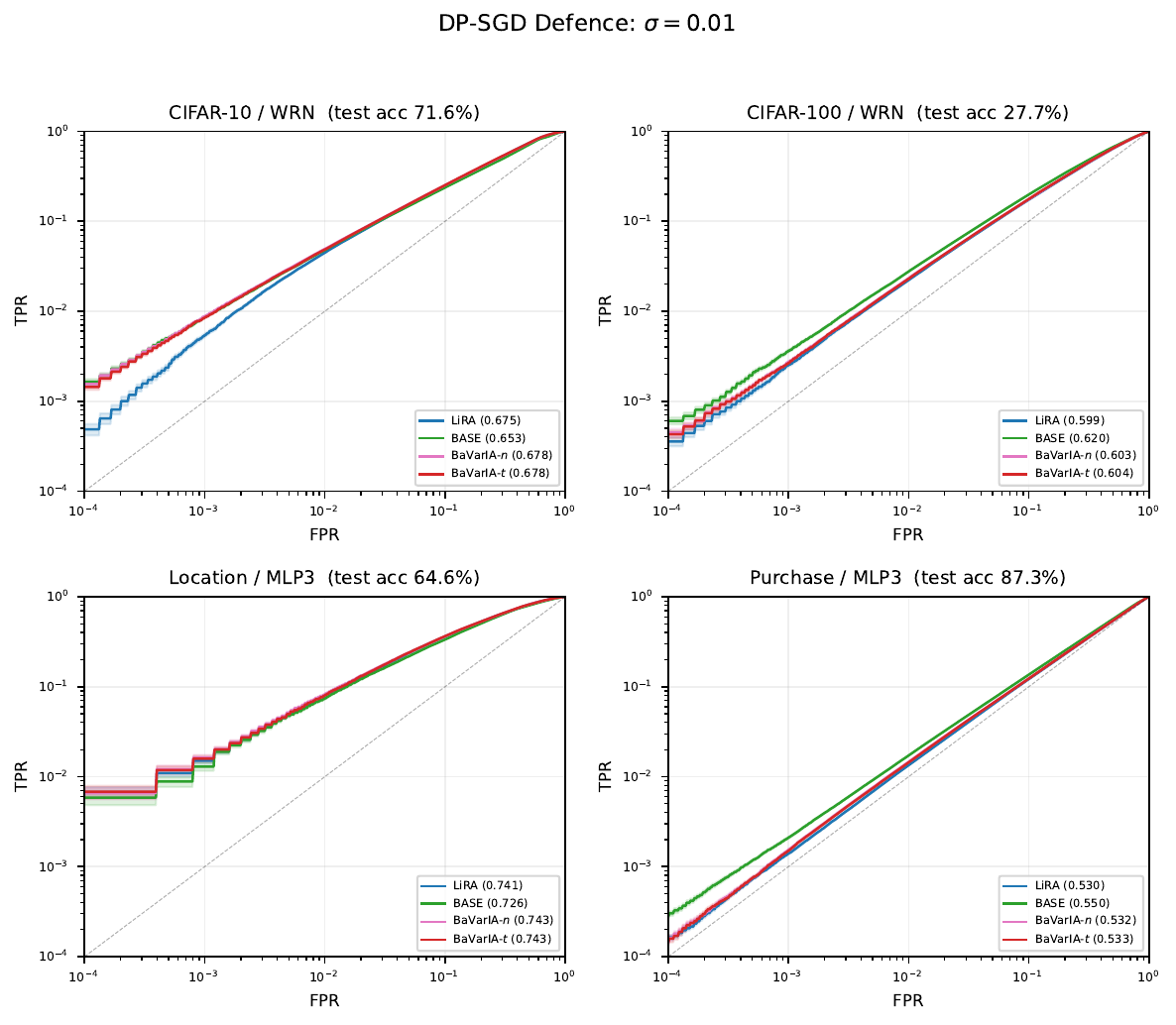}
\caption{Defense-level view, DP-SGD with $\sigma = 0.01$, $C = 10$. LiRA, BASE, BaVarIA-$n$, BaVarIA-$t$ overlaid; $K=64$, mean over 32 replicates. Privacy budget at this noise level is vacuous ($\varepsilon > 100\mathrm{K}$).}\label{fig:def_sgm001}
\end{figure}

\begin{figure}
\centering
\includegraphics[width=\linewidth]{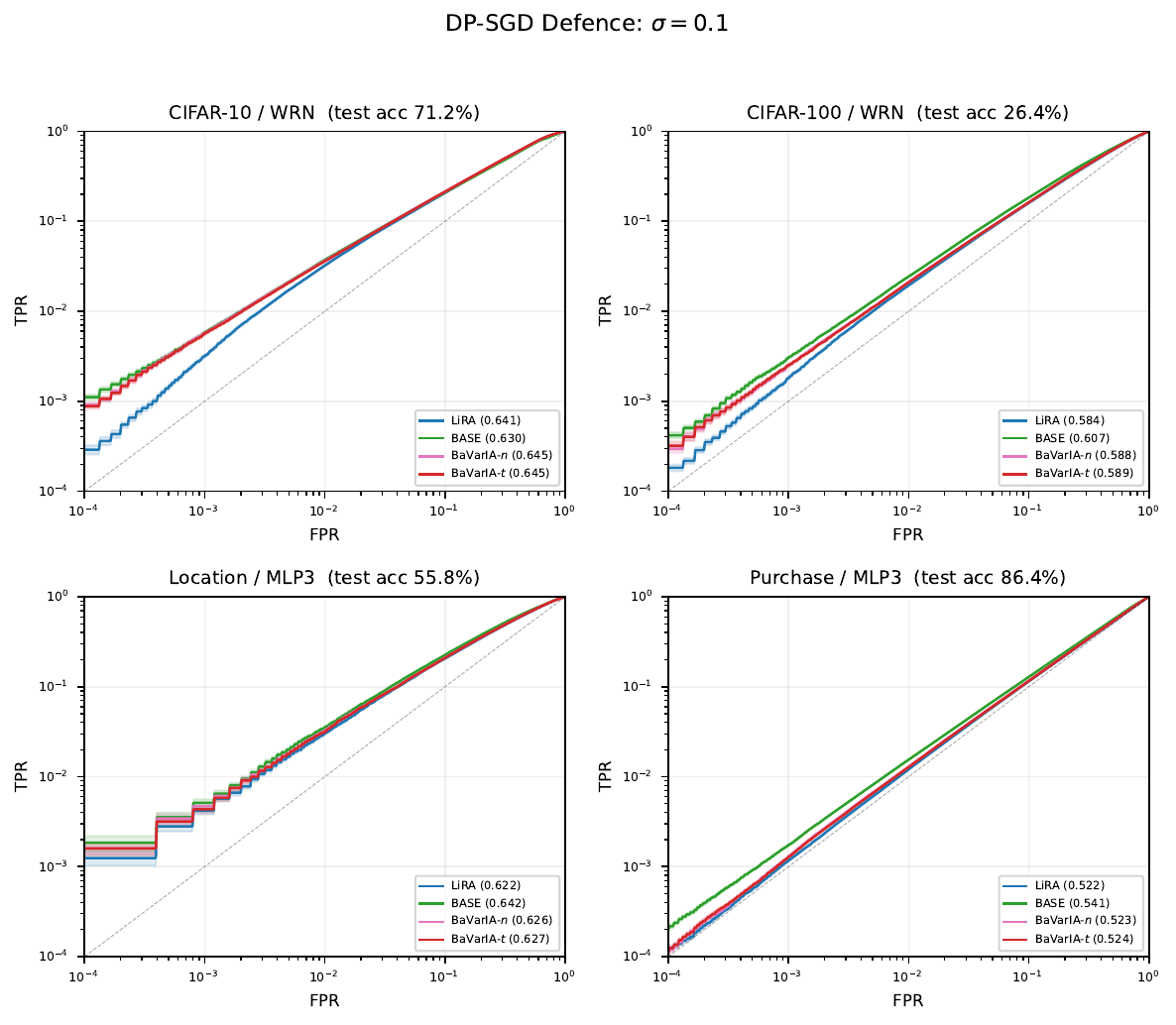}
\caption{Defense-level view, DP-SGD with $\sigma = 0.1$, $C = 10$. LiRA, BASE, BaVarIA-$n$, BaVarIA-$t$ overlaid; $K=64$, mean over 32 replicates.}\label{fig:def_sgm01}
\end{figure}

\begin{figure}
\centering
\includegraphics[width=\linewidth]{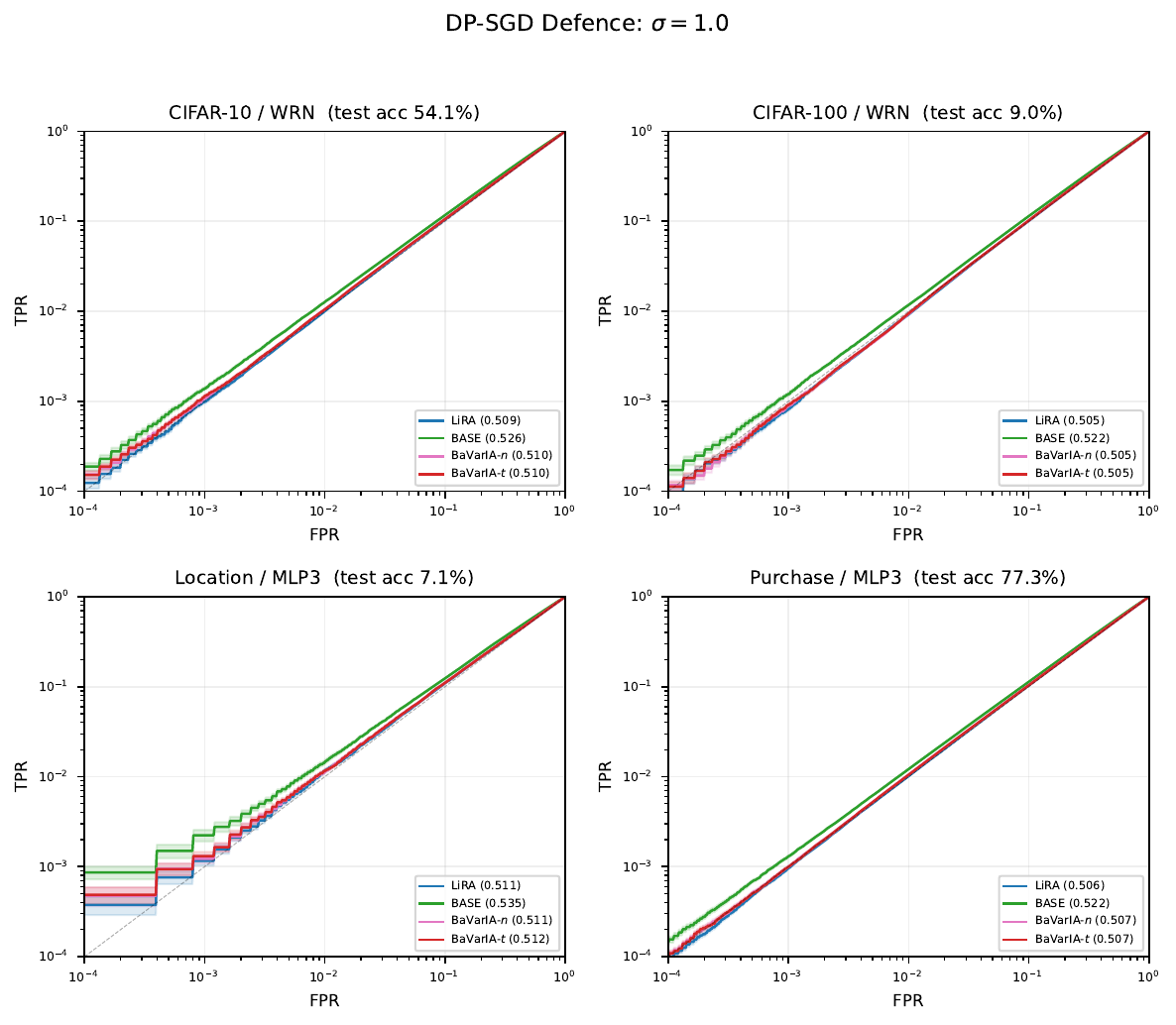}
\caption{Defense-level view, DP-SGD with $\sigma = 1.0$, $C = 10$ (meaningful privacy, $\varepsilon \approx 6$--$9$). LiRA, BASE, BaVarIA-$n$, BaVarIA-$t$ overlaid; $K=64$, mean over 32 replicates. All attacks approach random guessing at this noise level.}\label{fig:def_sgm1}
\end{figure}

\begin{figure}
\centering
\includegraphics[width=\linewidth]{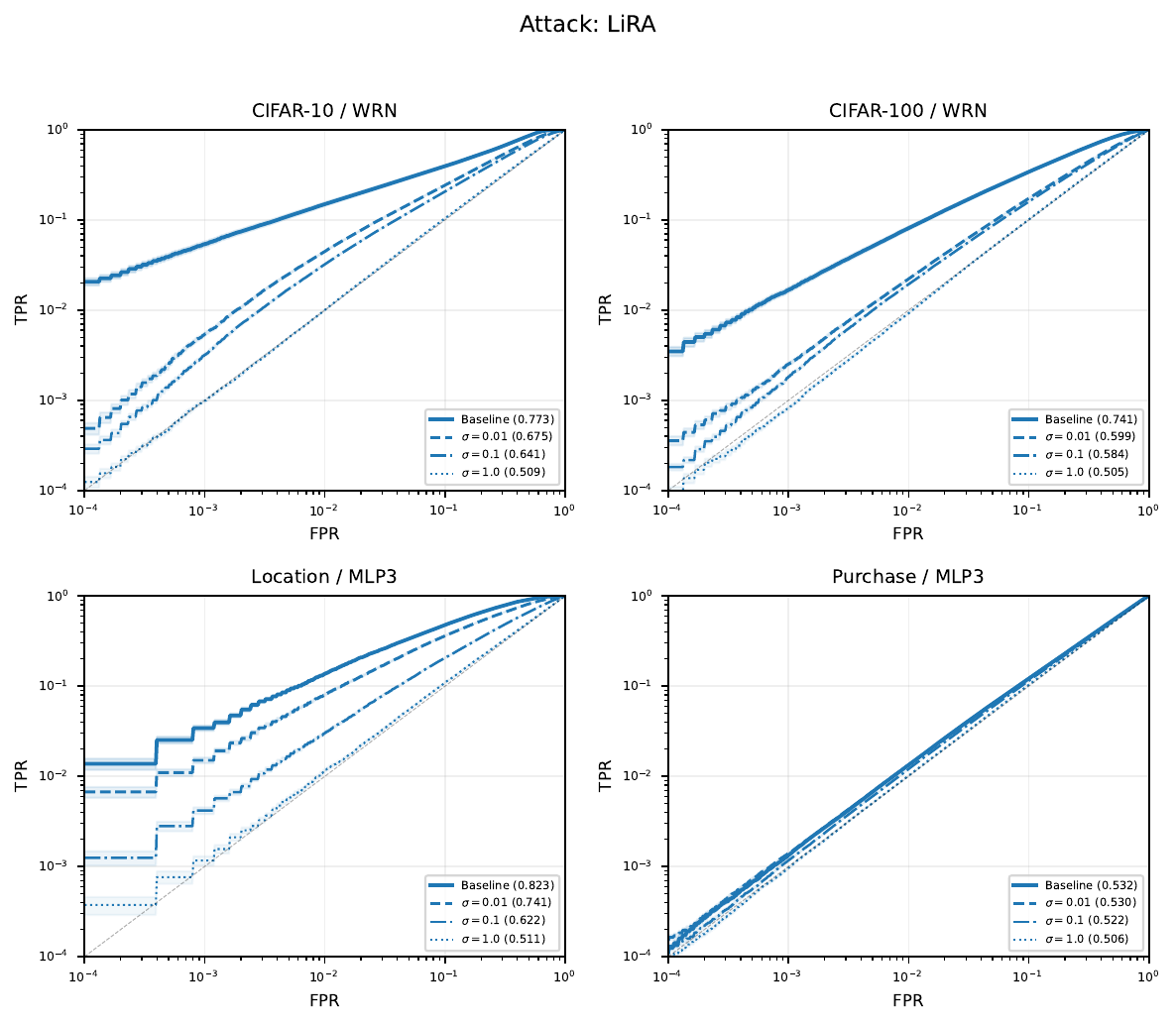}
\caption{Attack-method view, LiRA across defense levels $\sigma \in \{0,\,0.01,\,0.1,\,1.0\}$ on each testbed. Linestyle distinguishes defense level; $K=64$, online, mean over 32 replicates, log--log axes.}\label{fig:atk_lira}
\end{figure}

\begin{figure}
\centering
\includegraphics[width=\linewidth]{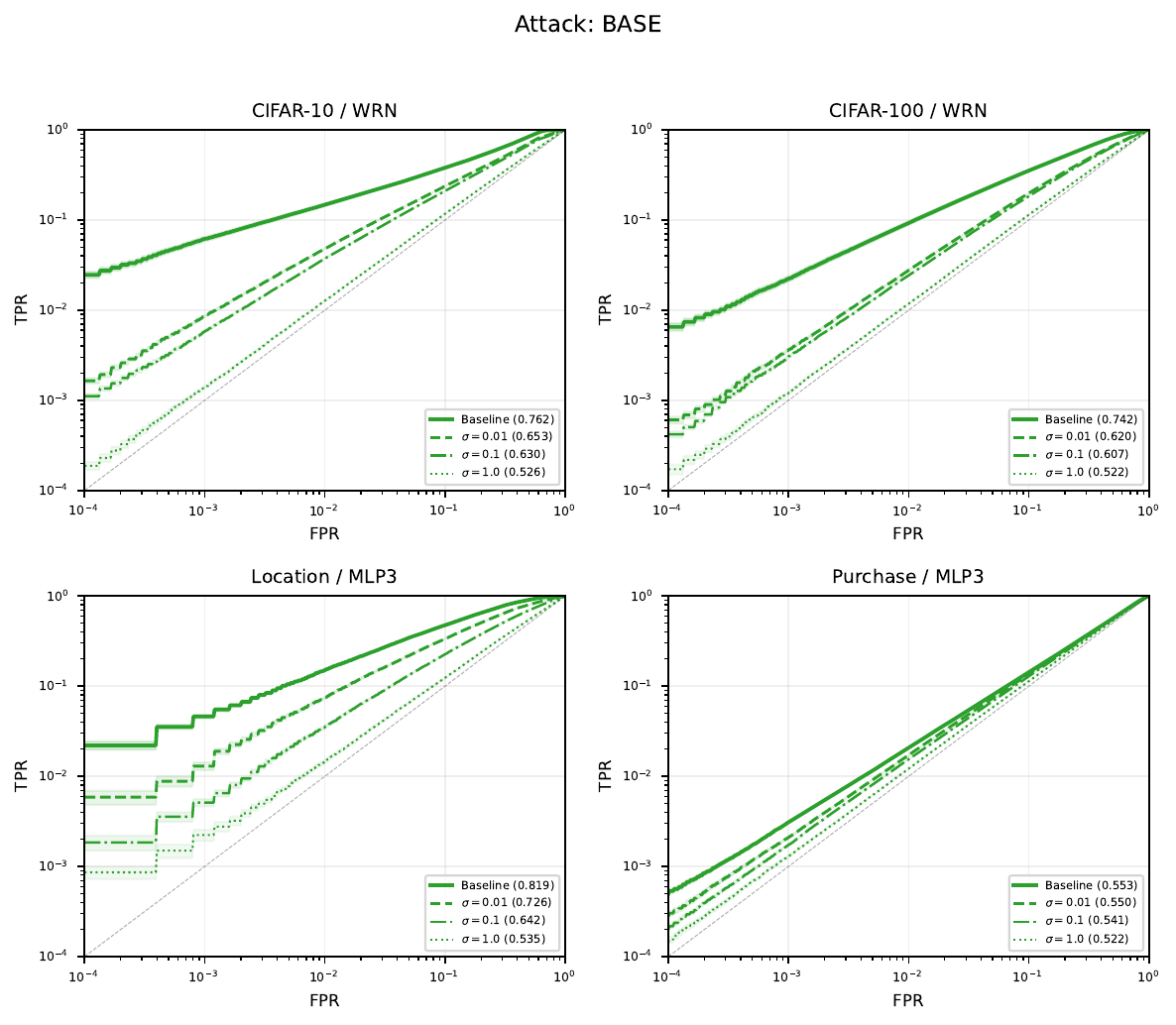}
\caption{Attack-method view, BASE across defense levels $\sigma \in \{0,\,0.01,\,0.1,\,1.0\}$. Linestyle distinguishes defense level; $K=64$, online, mean over 32 replicates.}\label{fig:atk_base}
\end{figure}

\begin{figure}
\centering
\includegraphics[width=\linewidth]{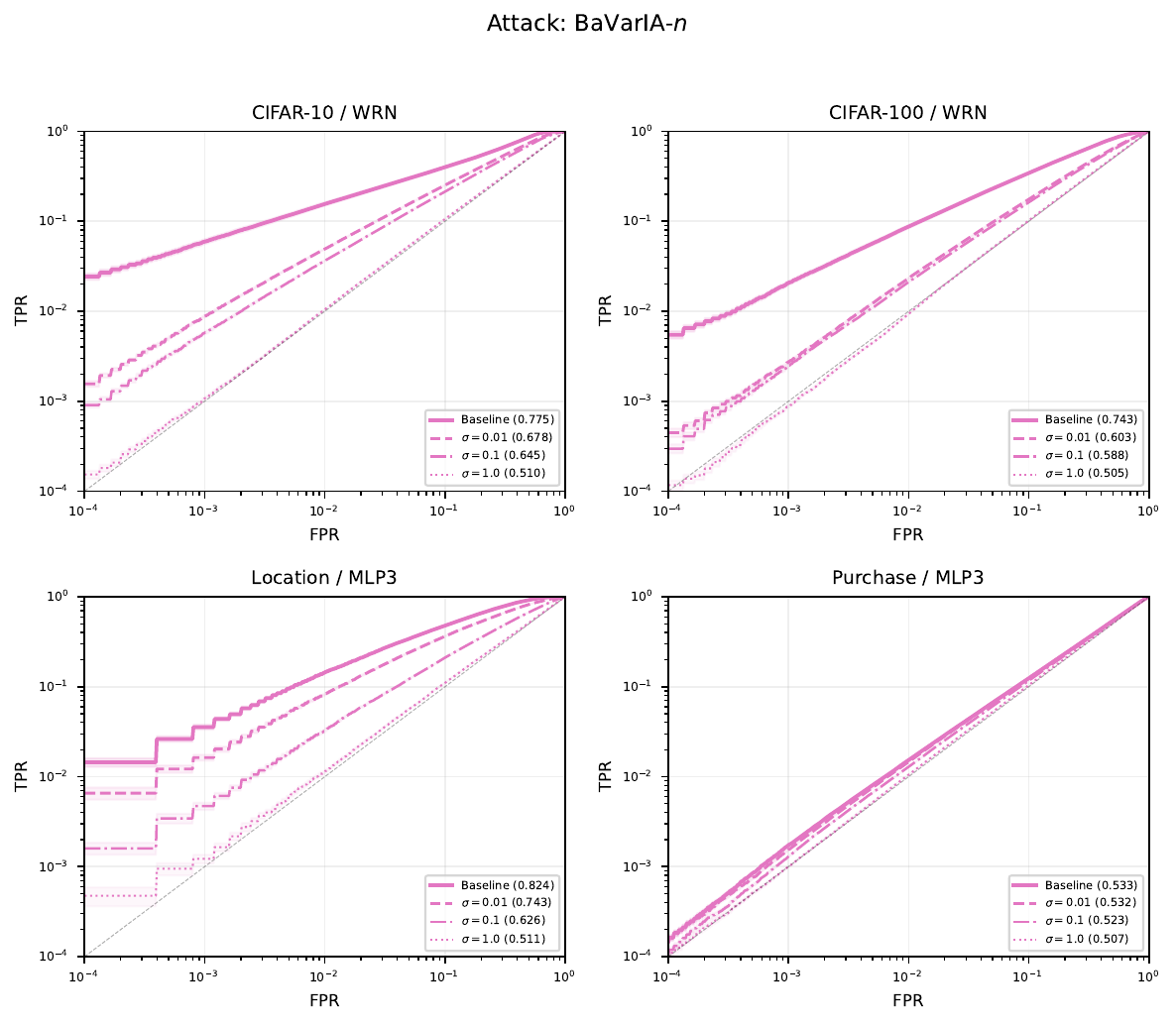}
\caption{Attack-method view, BaVarIA-$n$ across defense levels $\sigma \in \{0,\,0.01,\,0.1,\,1.0\}$. Linestyle distinguishes defense level; $K=64$, online, mean over 32 replicates.}\label{fig:atk_bavarian}
\end{figure}

\begin{figure}
\centering
\includegraphics[width=\linewidth]{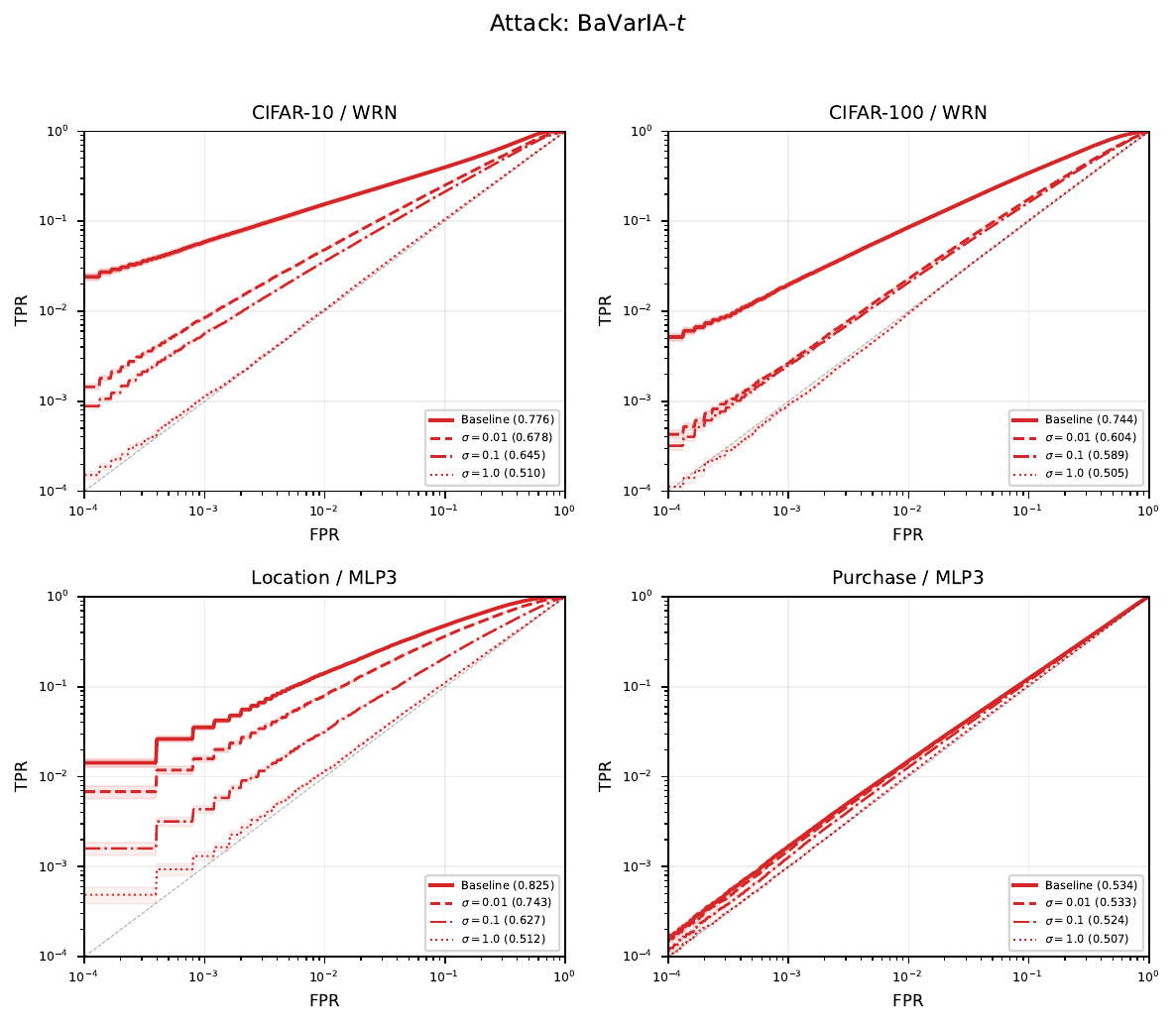}
\caption{Attack-method view, BaVarIA-$t$ across defense levels $\sigma \in \{0,\,0.01,\,0.1,\,1.0\}$. Linestyle distinguishes defense level; $K=64$, online, mean over 32 replicates.}\label{fig:atk_bavaria_t}
\end{figure}

\begin{figure}
\centering
\includegraphics[width=0.95\linewidth]{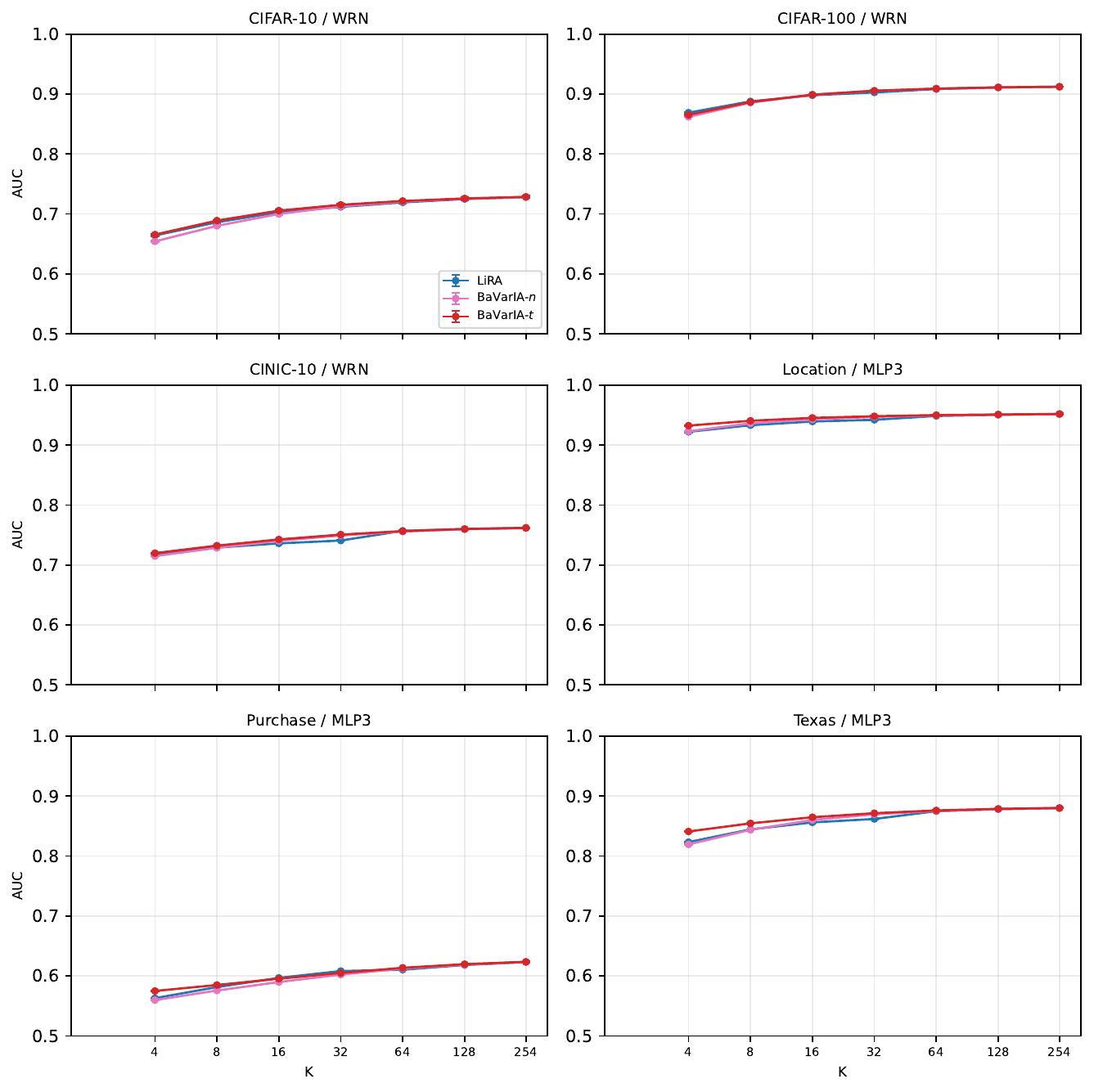}
\caption{Online AUC vs.\ $K$ (WRN/MLP3 testbeds).}\label{fig:scaling_bv_auc_wrn}
\end{figure}
\begin{figure}
\centering
\includegraphics[width=0.95\linewidth]{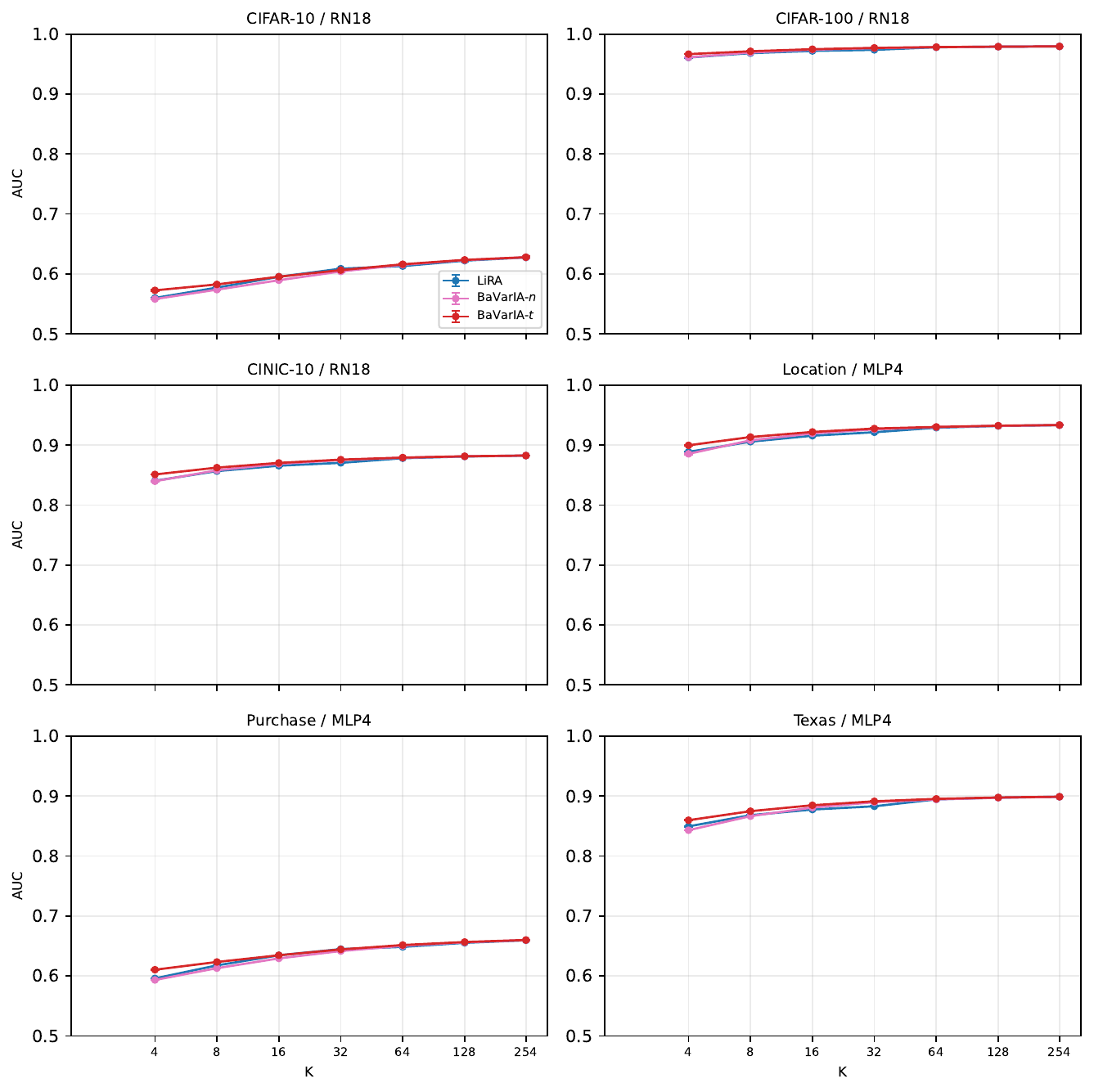}
\caption{Online AUC vs.\ $K$ (RN18/MLP4 testbeds).}\label{fig:scaling_rn_auc}
\end{figure}
\begin{figure}
\centering
\includegraphics[width=0.95\linewidth]{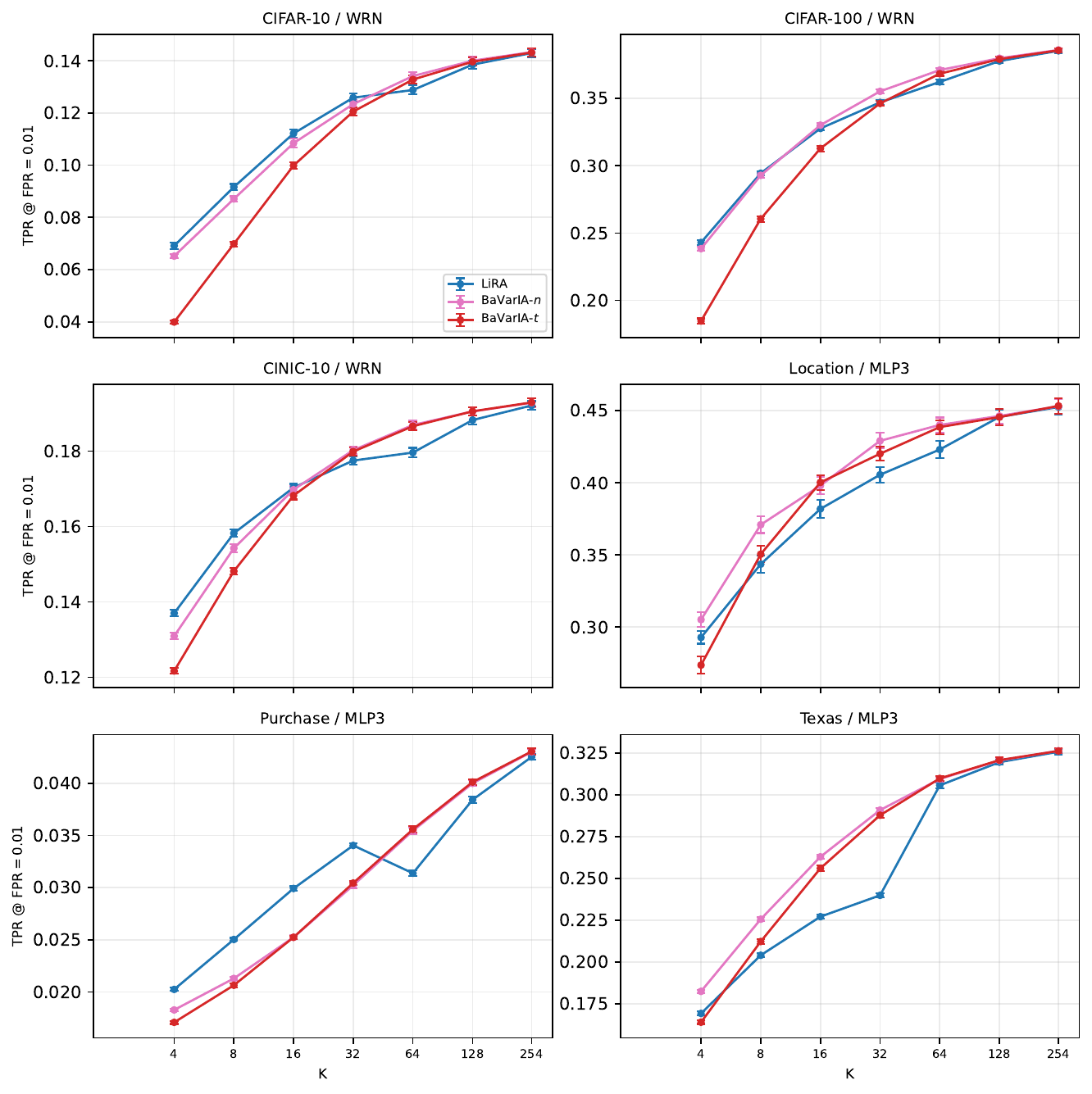}
\caption{Online TPR @ FPR$=0.01$ vs.\ $K$ (WRN/MLP3 testbeds).}\label{fig:scaling_bv_tpr_wrn}
\end{figure}
\begin{figure}
\centering
\includegraphics[width=0.95\linewidth]{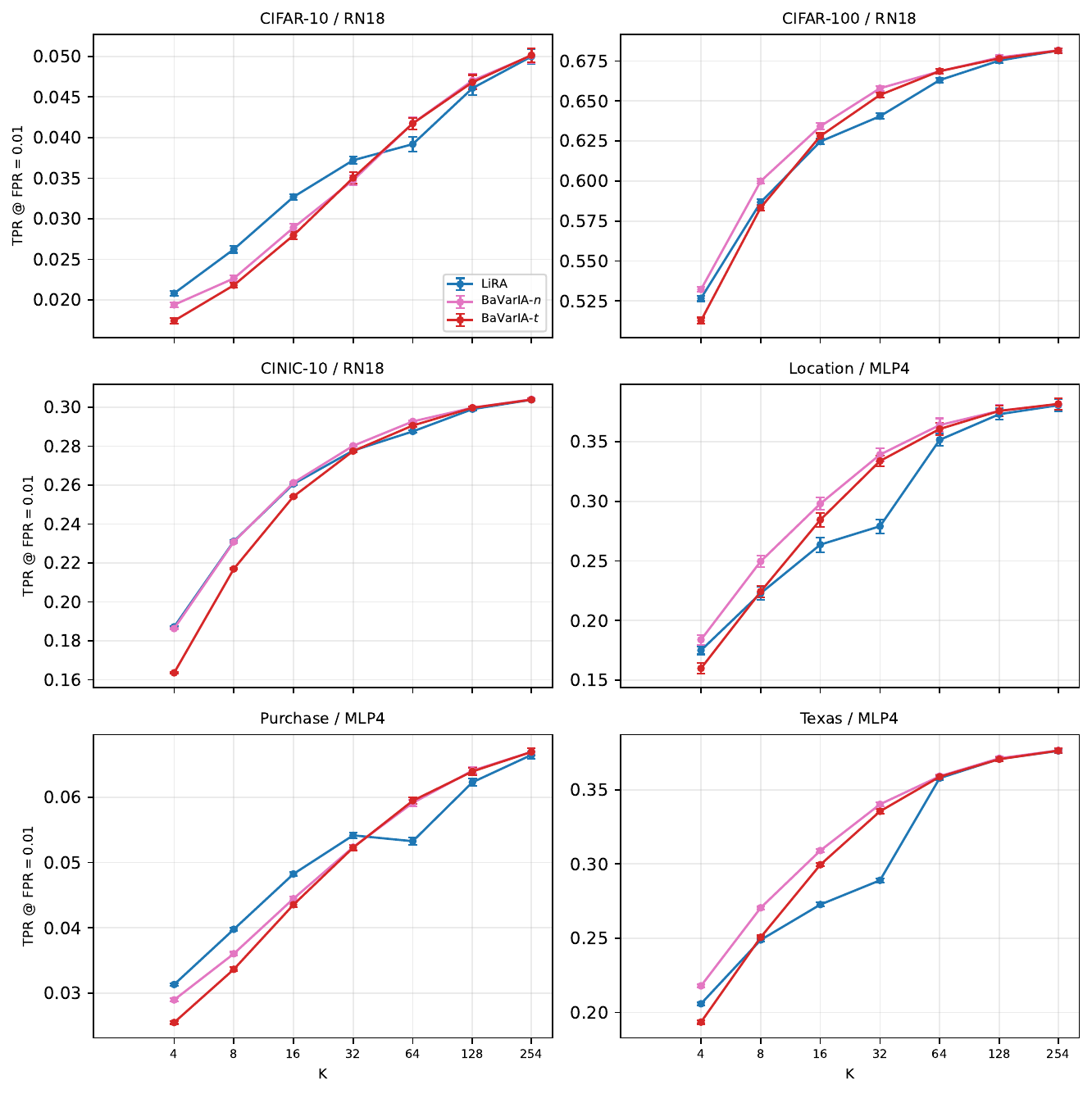}
\caption{Online TPR @ FPR$=0.01$ vs.\ $K$ (RN18/MLP4 testbeds).}\label{fig:scaling_rn_tpr}
\end{figure}

\begin{figure}
\centering
\includegraphics[width=0.95\linewidth]{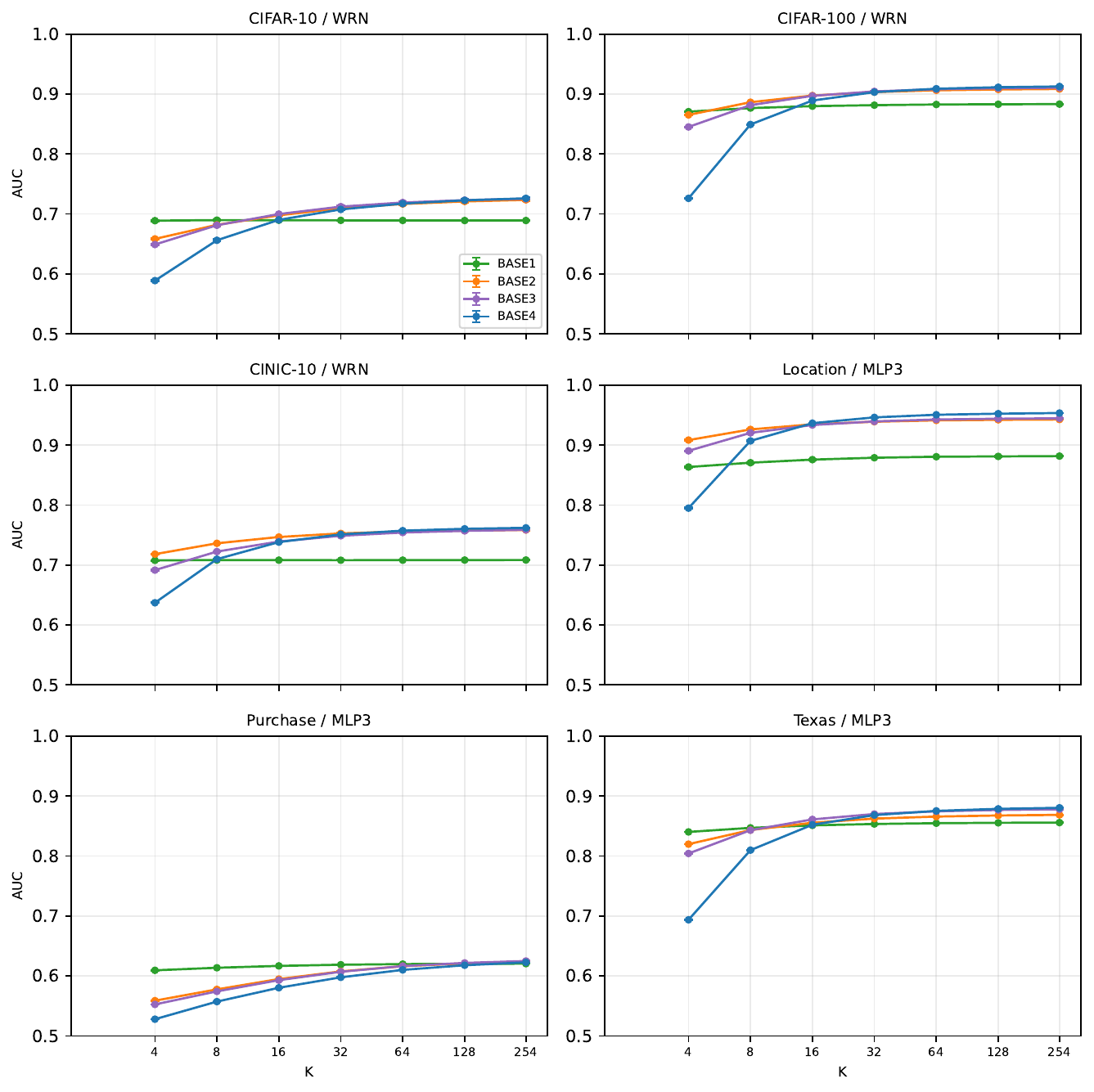}
\caption{Online AUC vs.\ $K$ for the BASE hierarchy (WRN/MLP3 testbeds). BASE3/BASE4 overtake the simpler families at moderate $K$; BASE2 tracks BASE3.}\label{fig:online_base_auc}
\end{figure}
\begin{figure}
\centering
\includegraphics[width=0.95\linewidth]{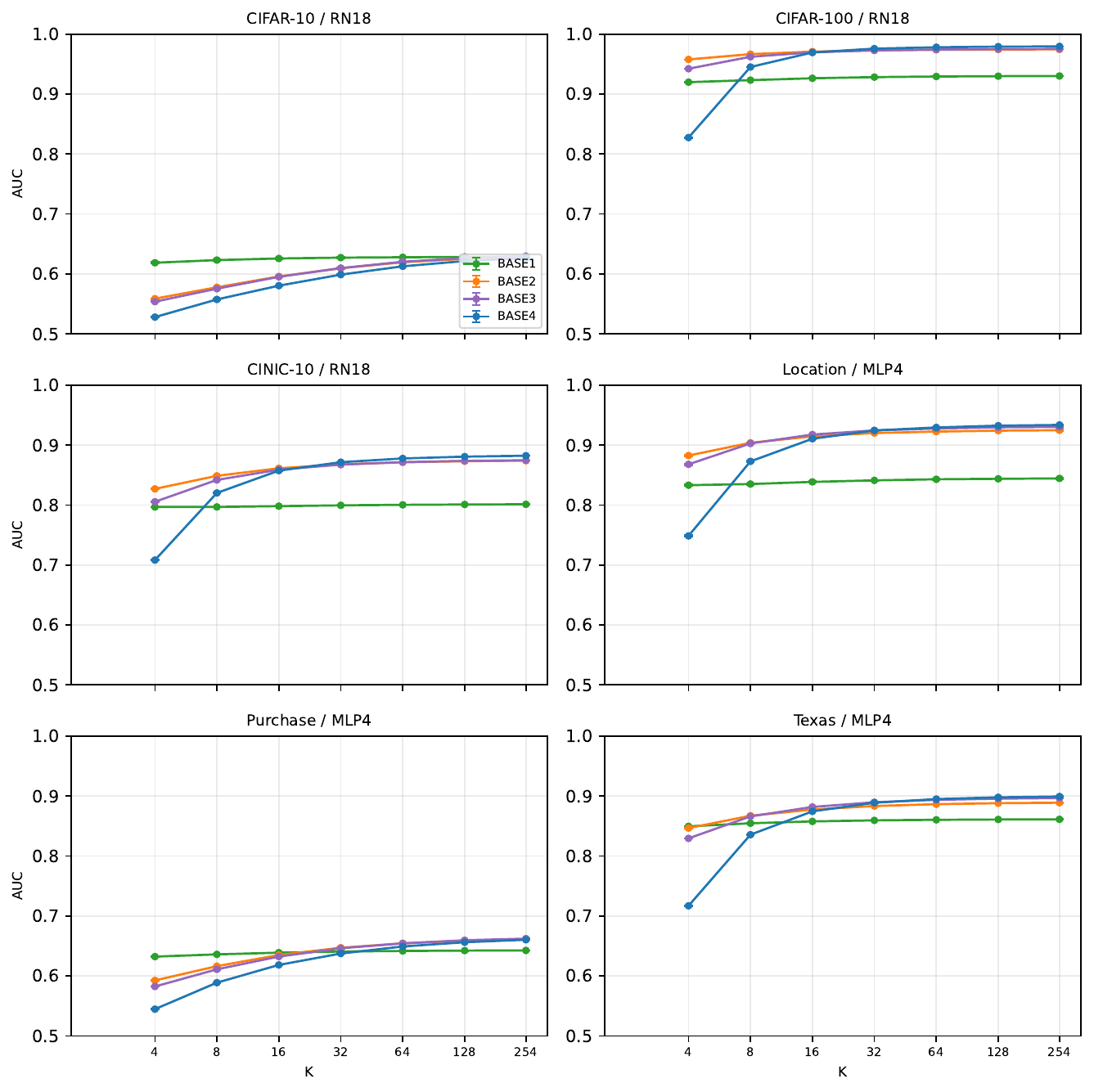}
\caption{Online AUC vs.\ $K$ for the BASE hierarchy (RN18/MLP4 testbeds).}\label{fig:online_base_auc_rn}
\end{figure}
\begin{figure}
\centering
\includegraphics[width=0.95\linewidth]{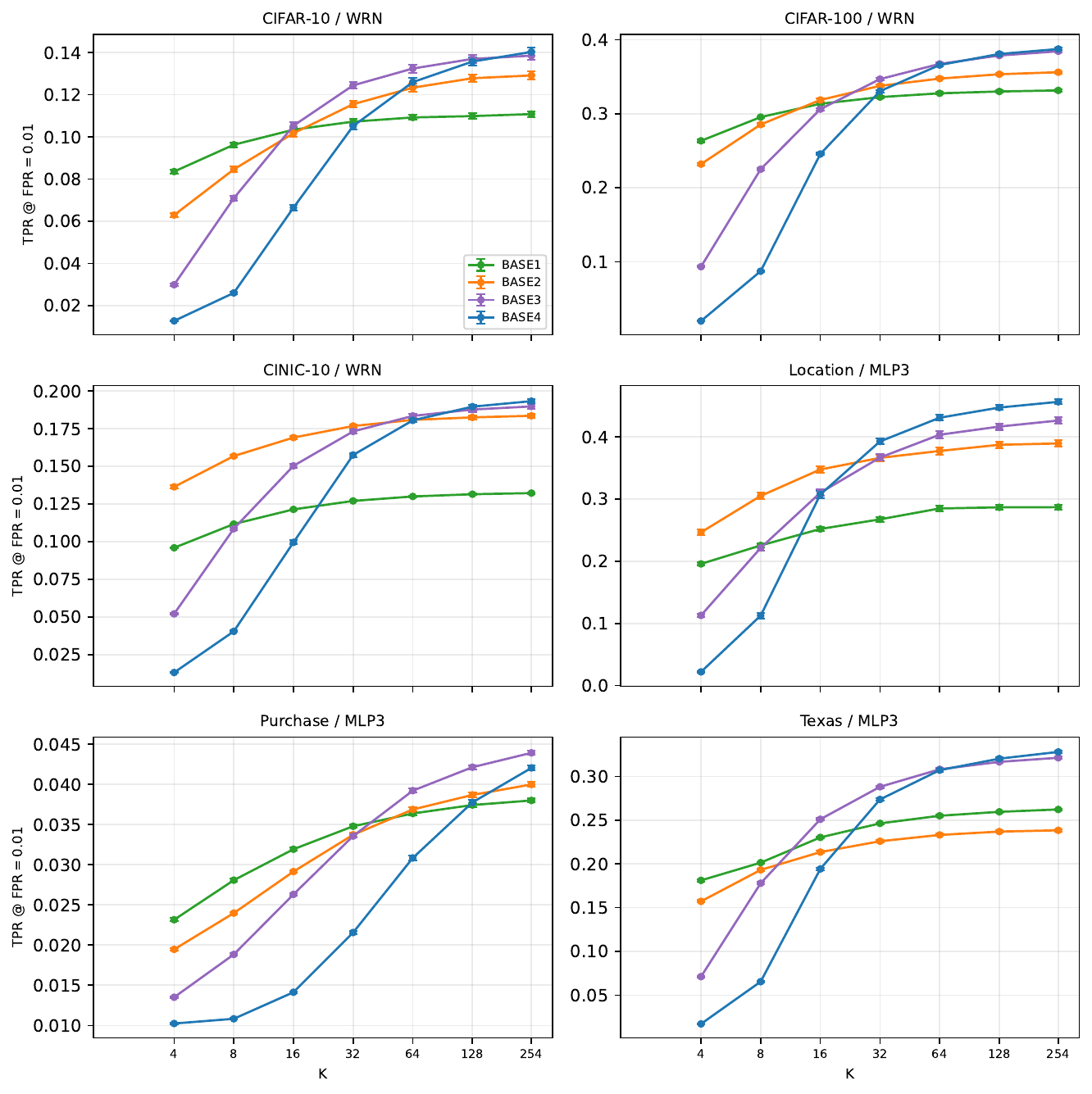}
\caption{Online TPR at FPR${=}0.01$ vs.\ $K$ for the BASE hierarchy (WRN/MLP3 testbeds).}\label{fig:online_base_tpr}
\end{figure}
\begin{figure}
\centering
\includegraphics[width=0.95\linewidth]{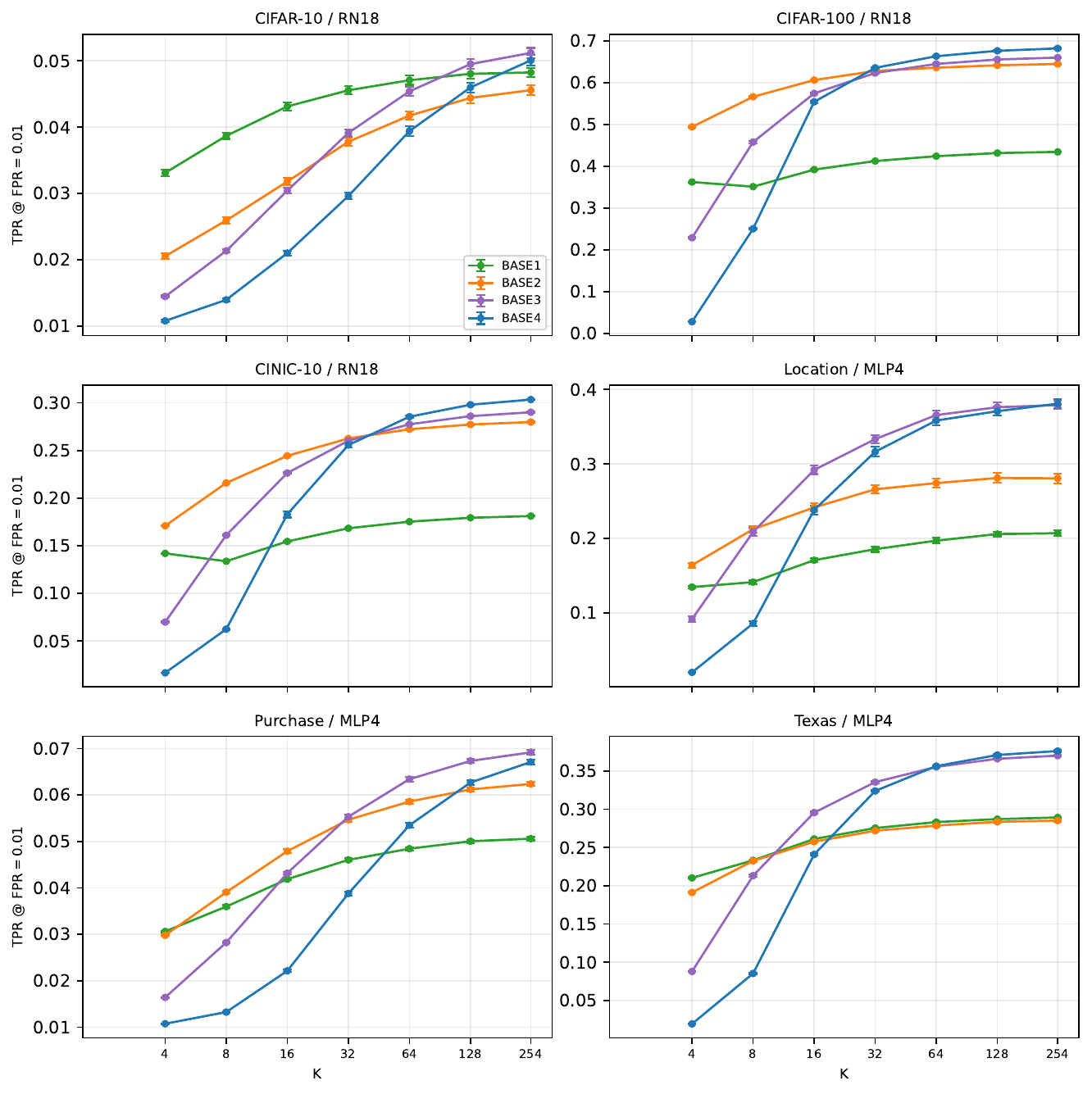}
\caption{Online TPR at FPR${=}0.01$ vs.\ $K$ for the BASE hierarchy (RN18/MLP4 testbeds).}\label{fig:online_base_tpr_rn}
\end{figure}

\begin{figure}
\centering
\includegraphics[width=0.95\linewidth]{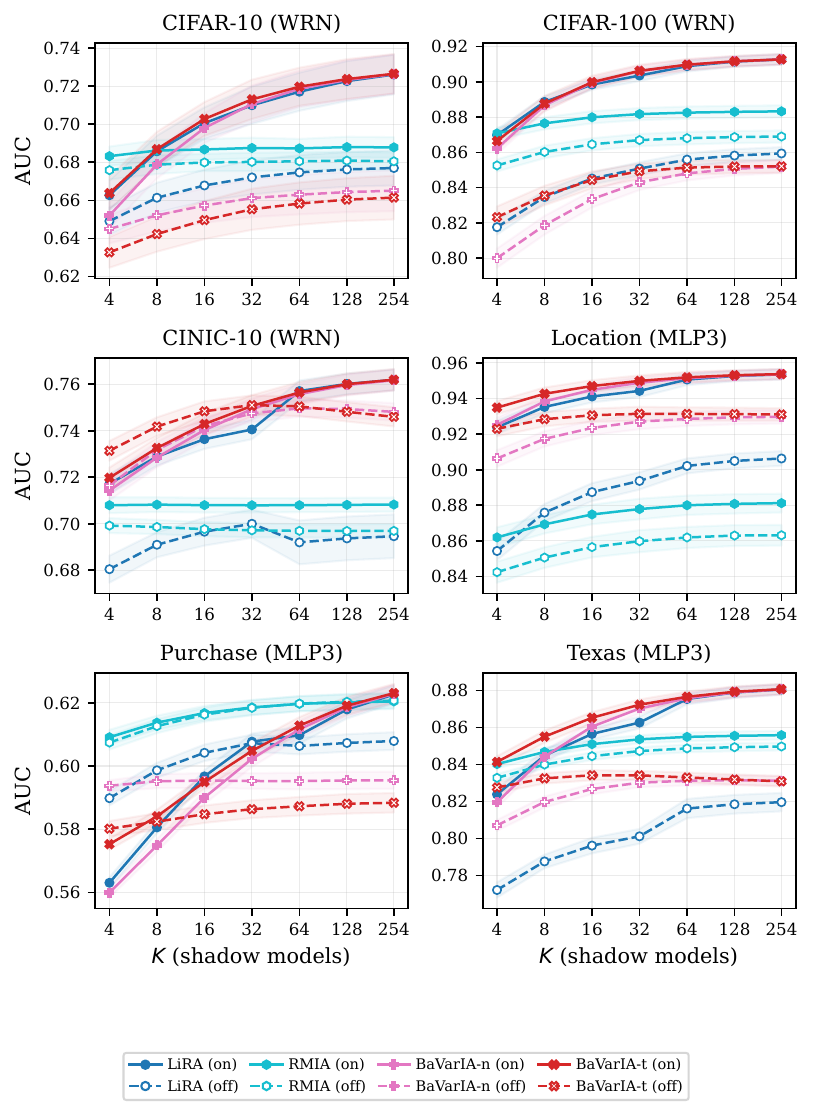}
\caption{Online vs.\ offline AUC as a function of $K$ (WRN/MLP3 testbeds). Solid = online, dashed = offline.}\label{fig:offline_auc}
\end{figure}
\begin{figure}
\centering
\includegraphics[width=0.95\linewidth]{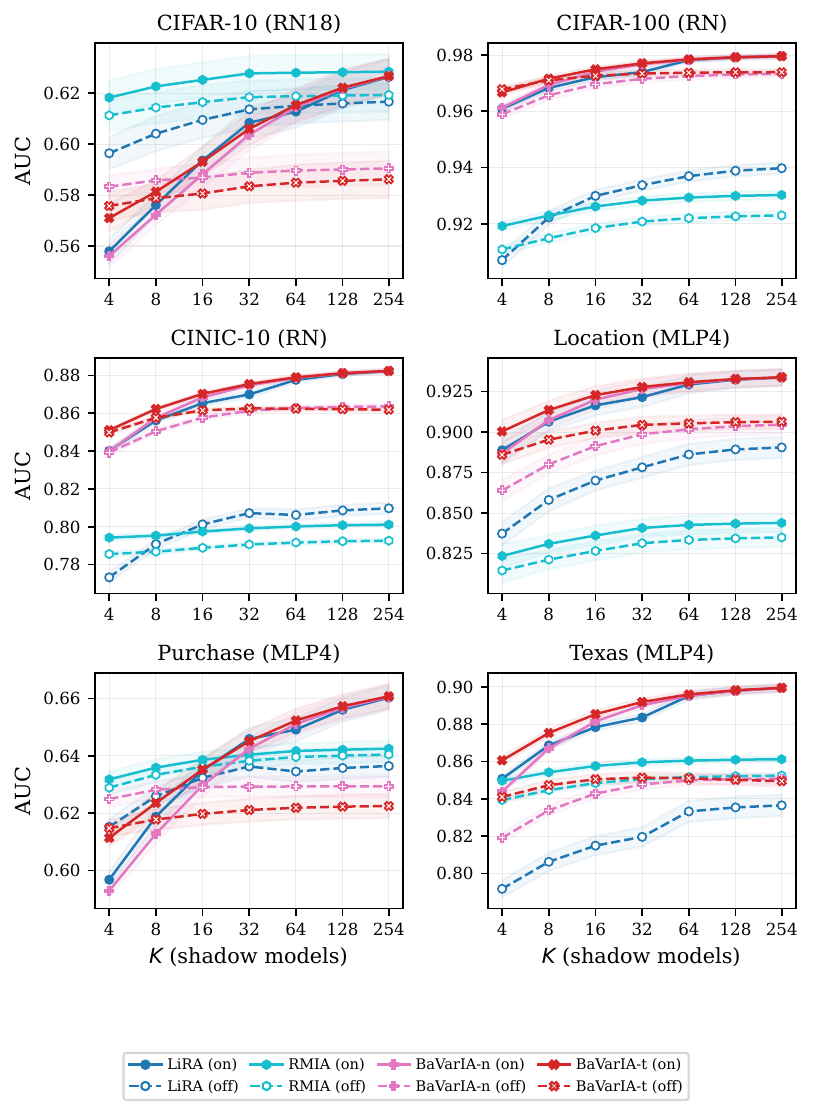}
\caption{Online vs.\ offline AUC as a function of $K$ (RN18/MLP4 testbeds).}\label{fig:offline_auc_rn}
\end{figure}
\begin{figure}
\centering
\includegraphics[width=0.95\linewidth]{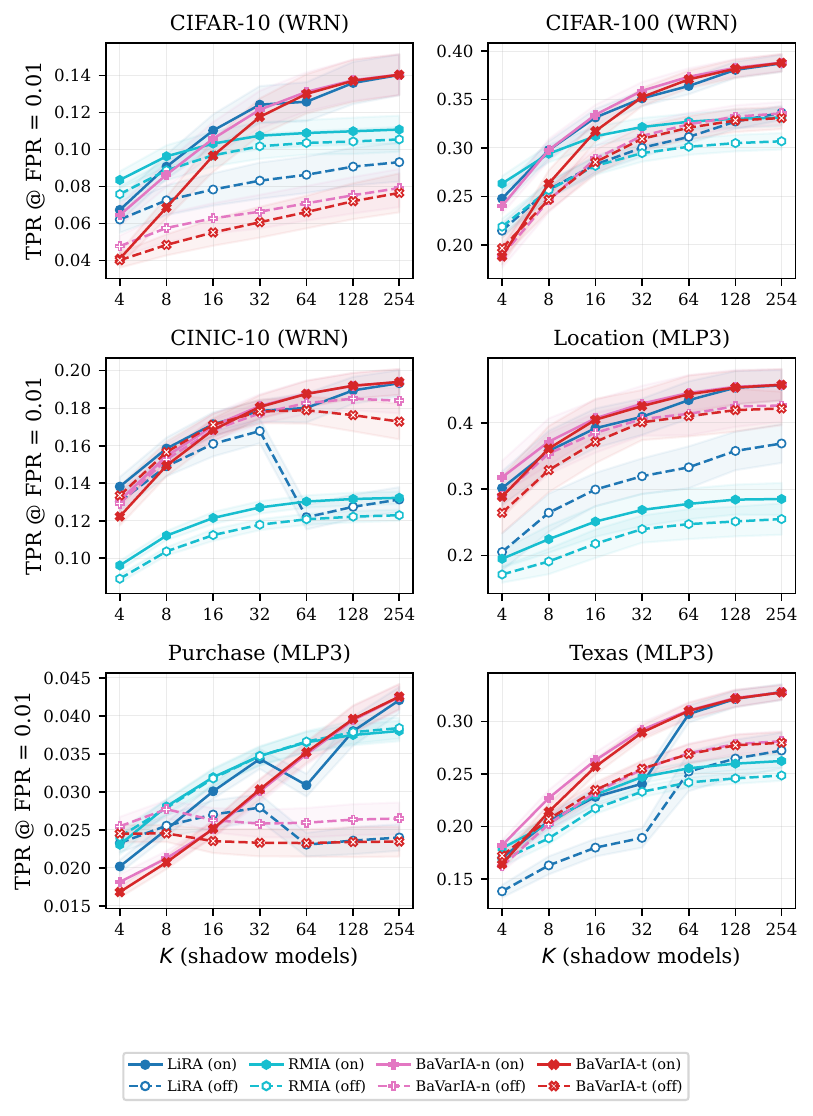}
\caption{Online vs.\ offline TPR at FPR${=}0.01$ as a function of $K$ (WRN/MLP3 testbeds). Solid = online, dashed = offline.}\label{fig:offline_tpr}
\end{figure}
\begin{figure}
\centering
\includegraphics[width=0.95\linewidth]{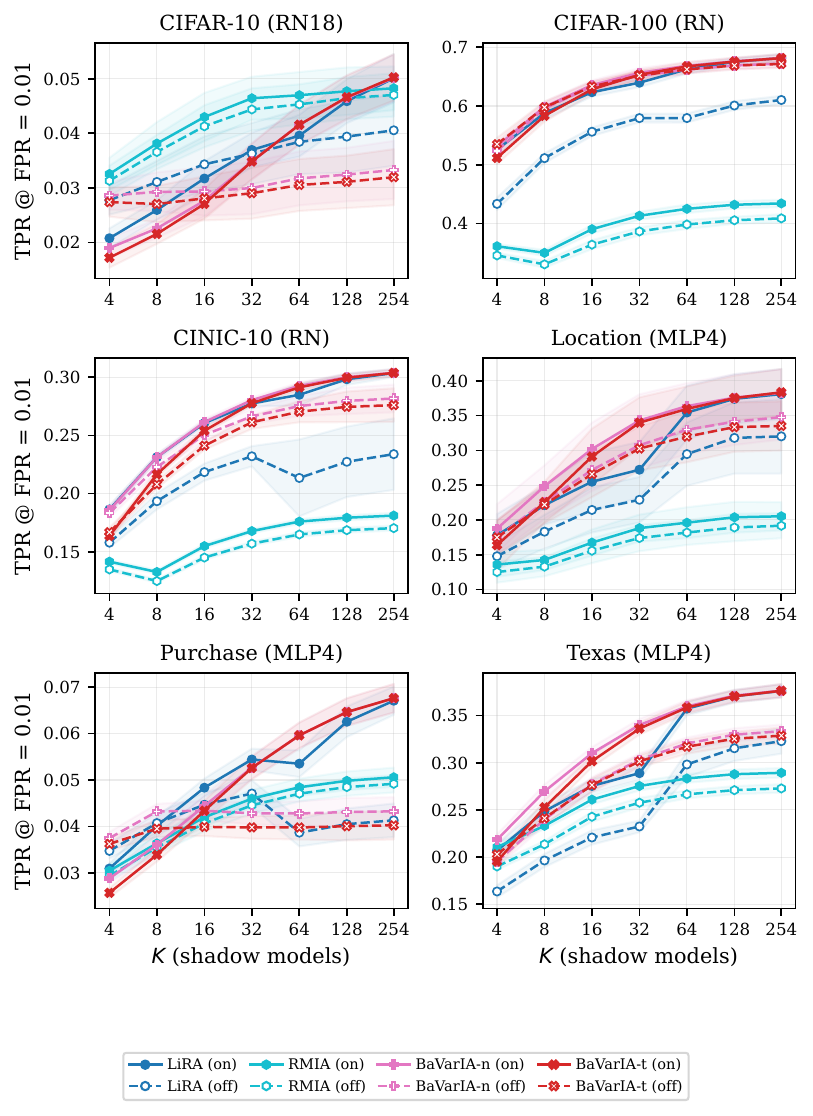}
\caption{Online vs.\ offline TPR at FPR${=}0.01$ as a function of $K$ (RN18/MLP4 testbeds).}\label{fig:offline_tpr_rn}
\end{figure}

\begin{figure}
\centering
\includegraphics[width=0.95\linewidth]{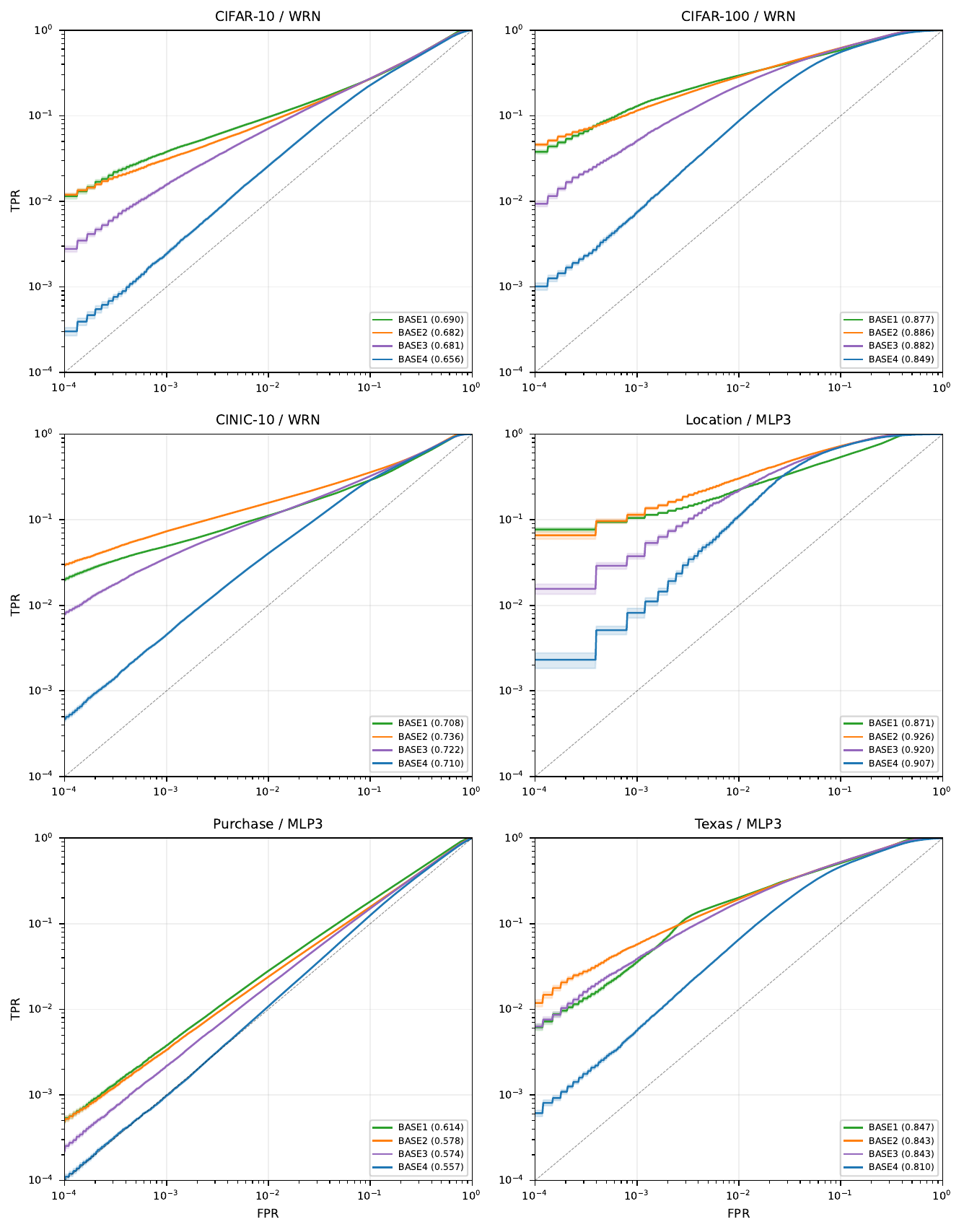}
\caption{BASE hierarchy ROC at $K=8$ (WRN/MLP3 testbeds).}\label{fig:roc_base_k8}
\end{figure}
\begin{figure}
\centering
\includegraphics[width=0.95\linewidth]{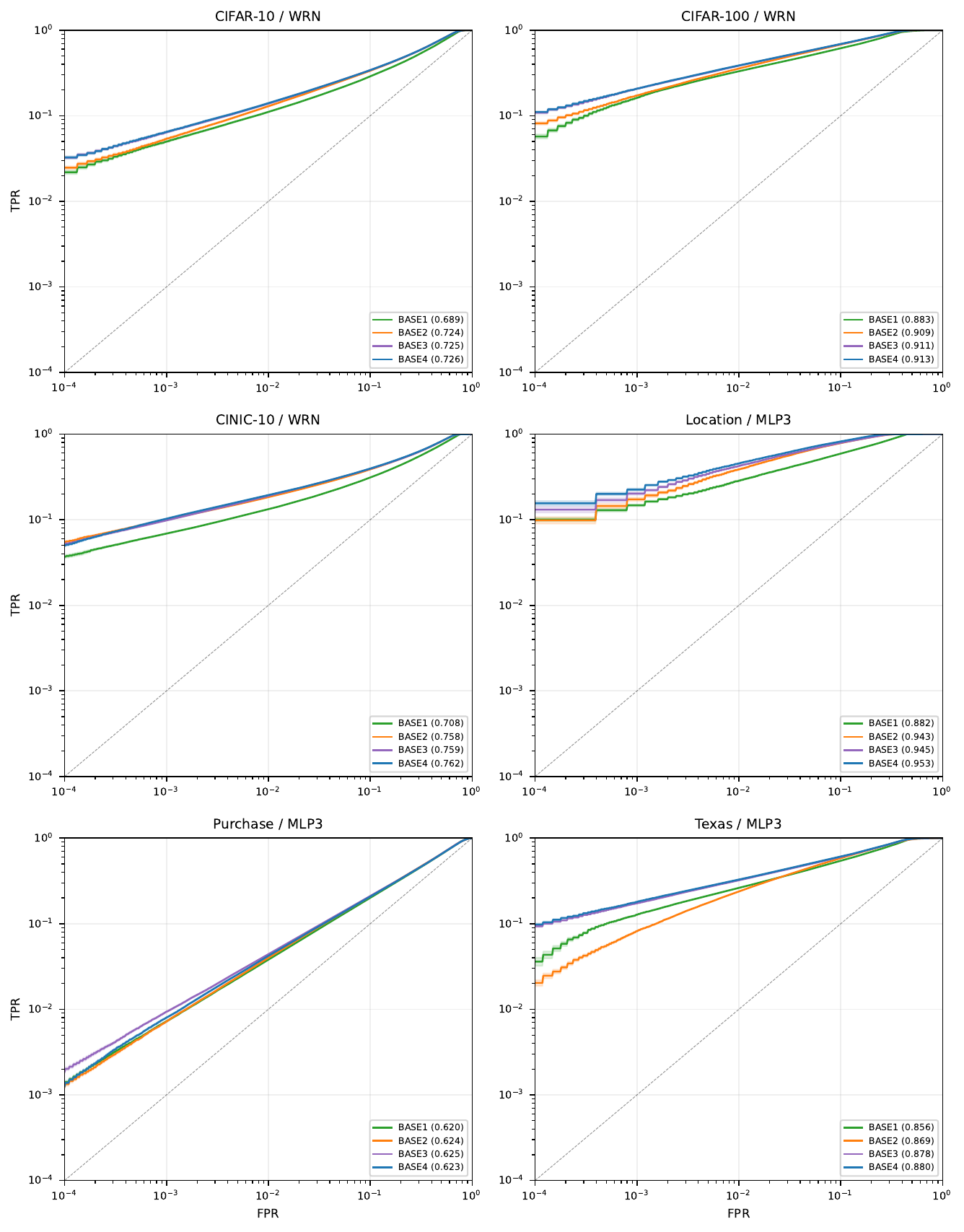}
\caption{BASE hierarchy ROC at $K=254$ (WRN/MLP3 testbeds), ordering reversed relative to $K=8$.}\label{fig:roc_base_k254}
\end{figure}
\begin{figure}
\centering
\includegraphics[width=0.95\linewidth]{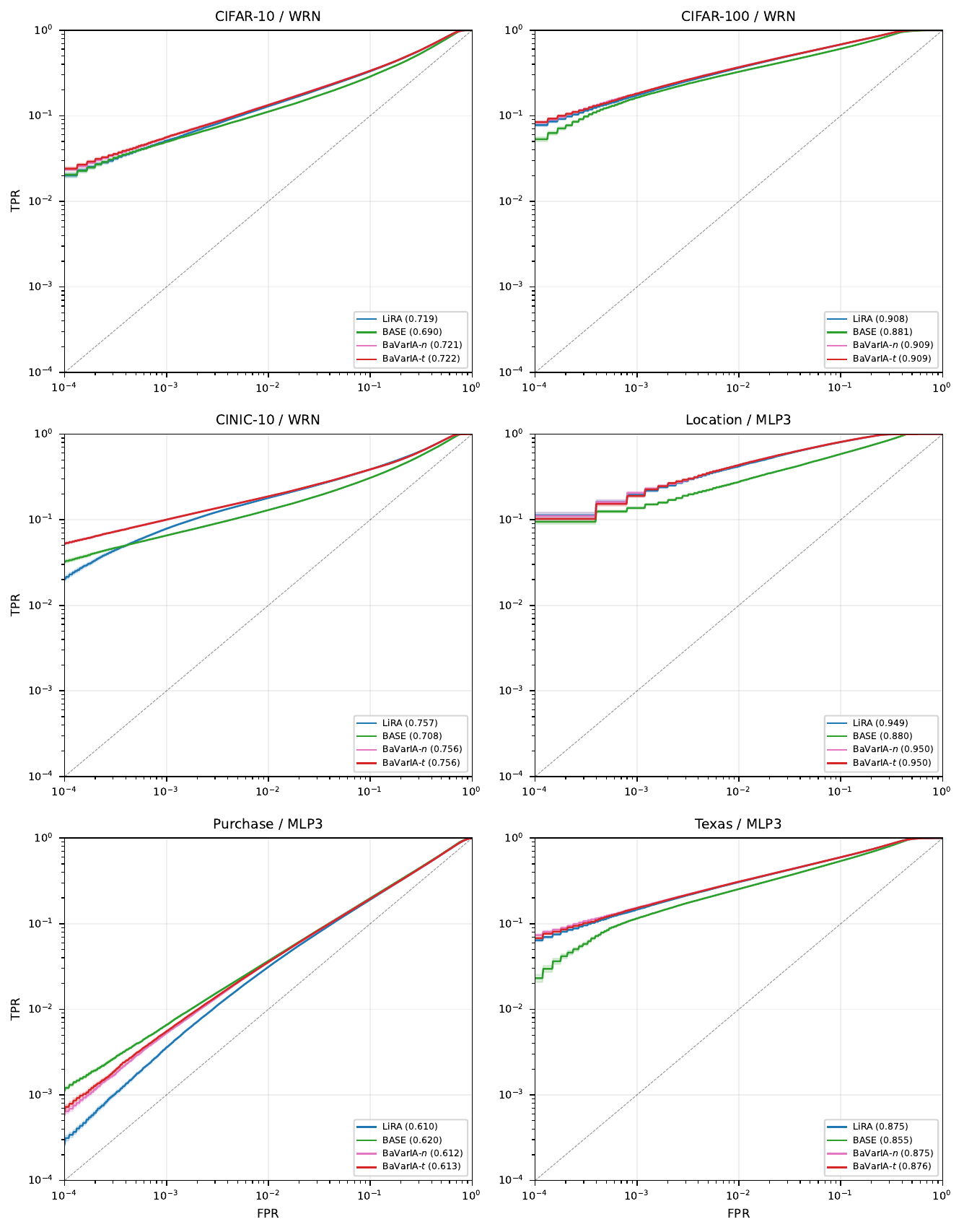}
\caption{BaVarIA online ROC at $K=64$ (WRN/MLP3 testbeds).}\label{fig:roc_bav_online_k64}
\end{figure}
\begin{figure}
\centering
\includegraphics[width=0.95\linewidth]{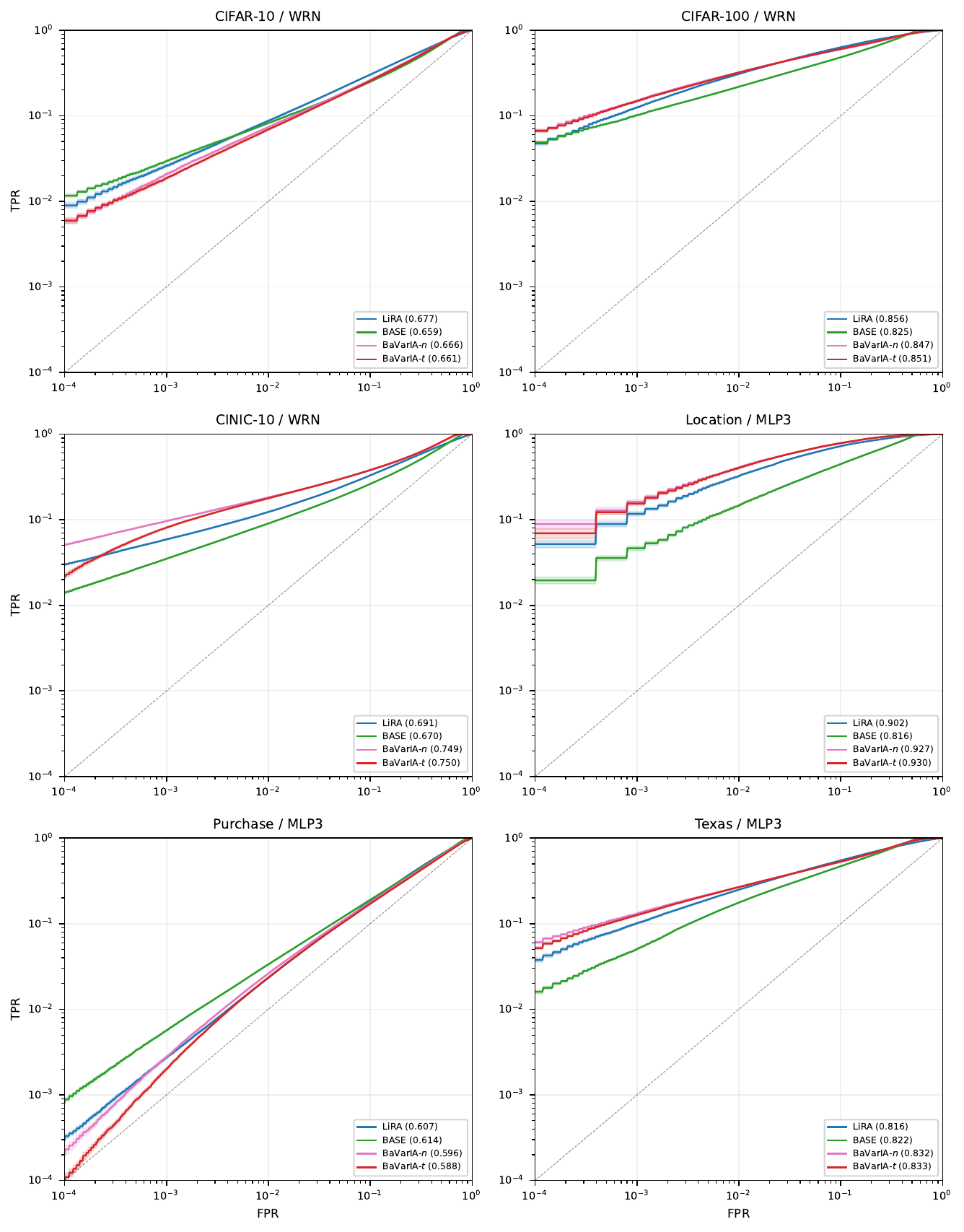}
\caption{BaVarIA offline ROC at $K=64$ (WRN/MLP3 testbeds).}\label{fig:roc_bav_offline_k64}
\end{figure}

\end{document}